\documentclass[sigconf]{acmart}

\usepackage{amsmath,amsfonts,bm}

\def\eqref#1{equation~\ref{#1}}

\def\floor#1{\lfloor #1 \rfloor}
\def\1{\bm{1}}

\def\vh{{\bm{h}}}

\def\vx{{\bm{x}}}
\def\vy{{\bm{y}}}

\def\mA{{\bm{A}}}

\def\mD{{\bm{D}}}

\def\mI{{\bm{I}}}

\def\mL{{\bm{L}}}

\def\mP{{\bm{P}}}

\def\mW{{\bm{W}}}
\def\mX{{\bm{X}}}

\DeclareMathAlphabet{\mathsfit}{\encodingdefault}{\sfdefault}{m}{sl}
\SetMathAlphabet{\mathsfit}{bold}{\encodingdefault}{\sfdefault}{bx}{n}

\def\gA{{\mathcal{A}}}

\def\gD{{\mathcal{D}}}
\def\gE{{\mathcal{E}}}
\def\gF{{\mathcal{F}}}
\def\gG{{\mathcal{G}}}
\def\gH{{\mathcal{H}}}

\def\gN{{\mathcal{N}}}
\def\gO{{\mathcal{O}}}
\def\gP{{\mathcal{P}}}

\def\gS{{\mathcal{S}}}

\def\gU{{\mathcal{U}}}
\def\gV{{\mathcal{V}}}

\def\sY{{\mathbb{Y}}}

\newcommand{\R}{\mathbb{R}}

\usepackage{float}
\usepackage{graphicx}
\usepackage{comment}
\usepackage{setspace}

\usepackage{placeins}

\usepackage{graphicx}
\usepackage{amsthm}
\usepackage{multirow}

\usepackage{xcolor}
\hypersetup{
    colorlinks,
    linkcolor={red!50!black},
    citecolor={blue!50!black},
    urlcolor={blue!80!black}
}

\usepackage{tikz}
\usetikzlibrary{trees}
\usepackage{forest}
\usepackage{wrapfig}
\usepackage{caption}
\usepackage{colortbl}

\usepackage{graphicx}
\usepackage{subfig} 
\usepackage{booktabs} 

\usepackage{hyperref}

\newtheorem{assumption}{Assumption}[section]
\usepackage{algorithmic}
\usepackage{algorithm}
\usepackage{multirow}
\usepackage[linesnumbered,ruled,vlined,algo2e]{algorithm2e}
\usepackage[inline]{enumitem}

\usepackage{cleveref}

\newcommand*\justify{%
  \fontdimen2\font=0.4em
  \fontdimen3\font=0.2em
  \fontdimen4\font=0.1em
  \fontdimen7\font=0.1em
  \hyphenchar\font=`\-
}

\newcommand{\payAttention}[3]{\textcolor{#1}{{{#2}: {#3}}}}
\newcommand{\xSays}[2]{\payAttention{blue}{#1 Says}{#2}}

\newcommand{\Sid}[1]{\xSays{Siddhartha}{#1}}

\newcommand{\smf}[1]{\xSays{Ferdous}{#1}}
\newcommand{\ags}{\texttt{AGS-GNN}}
\newcommand{\agsns}{\texttt{AGS-NS}}
\newcommand{\agsgs}{\texttt{AGS-GS}}
\newcommand{\agsgsrw}{\texttt{AGS-GS-RW}}
\newcommand{\agsgsdis}{\texttt{AGS-GS-Disjoint}}

\usepackage[compact]{titlesec}

\AtBeginDocument{%
  \providecommand\BibTeX{{%
    \normalfont B\kern-0.5em{\scshape i\kern-0.25em b}\kern-0.8em\TeX}}}

\setcopyright{none}
\renewcommand\footnotetextcopyrightpermission[1]{} 
\acmConference[KDD '24]{}{August}{Barcelona}
\acmISBN{978-1-4503-XXXX-X/18/06}

\begin{document}

\title{AGS-GNN: Attribute-guided Sampling for Graph Neural Networks}


\author{Siddhartha Shankar Das}
\email{das90@purdue.edu}
\affiliation{%
  \institution{Purdue University}
  \streetaddress{P.O. Box 1212}
  \city{West Lafayette}
  \state{IN}
  \country{USA}
  \postcode{47906}
}

\author{S M Ferdous}
\email{sm.ferdous@pnnl.gov}
\affiliation{%
  \institution{Pacific Northwest National Lab.}
  \streetaddress{P.O. Box 999}
  \city{Richland}
  \state{WA}
  \country{USA}
  \postcode{99352}
}

\author{Mahantesh M Halappanavar}
\email{hala@pnnl.gov}
\affiliation{%
  \institution{Pacific Northwest National Lab.}
  \streetaddress{P.O. Box 999}
  \city{Richland}
  \state{WA}
  \country{USA}
  \postcode{99352}
}

\author{Edoardo Serra}
\email{edoardoserra@boisestate.edu}
\affiliation{%
  \institution{Boise State University}
  \streetaddress{P.O. Box 999}
  \city{Boise}
  \state{ID}
  \country{USA}
  \postcode{83725}
}

\author{Alex Pothen}
\email{apothen@purdue.edu}
\affiliation{%
  \institution{Purdue University}
  \streetaddress{P.O. Box 1212}
  \city{West Lafayette}
  \state{IN}
  \country{USA}
  \postcode{47906}
}



\renewcommand{\shortauthors}{Trovato and Tobin, et al.}


\begin{abstract}
We propose \ags{}, a novel attribute-guided sampling algorithm for Graph Neural Networks (GNNs) that exploits node features and connectivity structure of a graph while simultaneously adapting for both homophily and heterophily in graphs. 
(In homophilic graphs vertices of the same class are more likely to be connected, and vertices of different classes tend to be linked in heterophilic graphs.) 
While GNNs have been successfully applied to homophilic graphs, their application to heterophilic graphs remains challenging. 
The best-performing GNNs for heterophilic graphs do not fit the sampling paradigm, suffer high computational costs, and are not inductive.
We employ samplers based on feature-similarity and feature-diversity 
to select subsets of neighbors for a node, and adaptively capture information from homophilic and heterophilic neighborhoods using dual channels.
Currently, \ags{} is the only algorithm that we know of that explicitly controls homophily in the sampled subgraph through similar and diverse neighborhood samples. 
For diverse neighborhood sampling, we employ submodularity, which was not used in this context prior to our work. The sampling distribution is pre-computed and highly parallel, achieving the desired scalability.
Using an extensive dataset consisting of 35 small ($\leq100K$ nodes) and large ($>100K$ nodes) homophilic and heterophilic graphs, we demonstrate the superiority of \ags{}  compare to the current approaches in the literature. \ags{} achieves comparable test accuracy to the best-performing heterophilic GNNs, even outperforming methods using the entire graph for node classification. \ags{} also converges faster compared to methods 
that sample neighborhoods randomly, and can be incorporated into existing GNN models that employ node or graph sampling.
\end{abstract}




\maketitle

\section{Introduction}
\label{sec:introduction}

Traditional Graph Neural Networks (GNNs)  rely on the \emph{homophilic} property of the learning problem, which assumes that a significant portion of a node's neighbors share the class label of the node. However, this assumption has been  challenged in recent years since graphs in several practical applications 
do not satisfy homophily~\citep{zheng2022graph,platonov2022characterizing}. 
Consequently,  GNNs  designed with the assumption of homophily fail to classify  heterophilic graphs accurately due to noisy or improper neighborhood aggregation~\citep{luan2022revisiting}.  
A simple Multi-layer Perceptron (MLP) based model that ignores the graph structure  can outperform existing homophilic GNNs on 
heterophilic graphs~\citep{ye2021sparse, lim2021large, zhu2020beyond,  liu2021non,luan2020complete, chien2020adaptive}. As a result, a number of special-purpose GNNs have been developed to classify heterophilic 
graphs~\citep{zheng2022graph,luan2022revisiting, zhu2020beyond,  liu2021non,luan2020complete, chien2020adaptive, zhu2021graph, yan2022two,li2022finding,xu2023node}.

Although GNNs have been adapted for heterophilic graphs in earlier work,
their applicability is limited since they do not scale to large graphs and are transductive. Unlike  homophilic GNNs~\citep{hamilton2017inductive, zeng2019graphsaint,chen2018stochastic,chen2018fastgcn,huang2018adaptive,chiang2019cluster}, where subgraph sampling strategies have been developed for scaling, currently there are no effective sampling approaches for heterophilic graphs~\citep{lim2021large}. 
Recent authors have  
enabled scaling by first transforming node features and adjacency matrix into lower dimensional representations,  and then applying mini-batching on the combined 
representations~\citep{lim2021large,liao2023ld}. 
Thus inference on heterophilic graphs via graph sampling remains 
challenging. 

The dichotomy of classifying graphs into  heterophilic or homophilic 
graphs, as used in current literature,  is blurred in practice by the presence of locally homophilic and locally heterophilic nodes.  
As shown in Fig.~\ref{fig:local_node_homophily} (\S\ref{sec:prelim}), 
both homophilic and heterophilic graphs could have nodes with high local homophily or  heterophily. However, there is no systematic and scalable approach for neighborhood sampling that distinguishes each node w.r.t. its local homophily property.
We  propose a new sampling strategy that incorporates both the adjacency structure of the graph as well as the feature information of the nodes that is capable of making this distinction. 
For a homophilic node, we build our sampling strategy based on a widely used \emph{smoothing} assumption~\citep{van2020survey}  that labels and features generally correlate positively. Heterophilic nodes of the same labels, however, are expected to have the same dissimilar neighborhood~\cite {luan2022revisiting}, and our sampling strategy exploits this.
Thus, we generate two sets of local neighborhood samples for each node: one based on feature similarity that potentially improves local homophily, while the other encourages diversity and increases heterophily. These two samples are adaptively learned using an MLP to select the appropriate one based on the downstream task. Our attribute-based sampling strategy can be seamlessly plugged into GNNs designed for classifying heterophilic graphs and  performs well in practice. The strength of our approach, however, is that even when paired with  GNNs designed for homophilic graphs, we  obtain better accuracies for heterophilic graphs, thus rendering an overall scalable approach for the latter graphs. The key contributions and findings of this work are summarized as follows:
\begin{enumerate}[wide, labelwidth=!, labelindent=2pt,itemsep=1pt,topsep=1pt] 
    \item We propose a novel scalable and inductive unsupervised and supervised feature-guided sampling framework, \ags{}, to learn node representations for both homophilic and heterophilic graphs.

    \item \ags{} incorporates sampling based on similarity and diversity (modeled by a submodular function). We are not aware of earlier work that uses submodularity in this context. \ags{} employs dual channels with MLPs to learn from both similar and diverse neighbors of a node.
    


    \item We experimented with $35$ benchmark datasets for node classification and compared them against GNNs designed for heterophilic and homophilic graphs. 
    For both types of graphs, \ags{} achieved improved accuracies relative to earlier methods (Fig.~\ref{fig:perfplot}) 
   (\S\ref{subsec:ablation_study}). Further, \ags{} also requires fewer iterations  (up to $50\%$ less) (\S\ref{subsec:runtimeandconvergence}) to 
   converge relative to random sampling.
\end{enumerate}

\section{Preliminaries}
\label{sec:prelim}

Consider a weighted graph $\gG(\gV, \gE)$ with set of vertices $\gV$,  and
set of edges $\gE$.  
We denote the number of vertices and edges by 
 $|\gV| \equiv n$ and $|\gE| \equiv m$. 
The adjacency matrix of the graph will be denoted by $\mA\in \R^{n\times n}$,  
and the $f$-dimensional feature matrix of nodes by 
$\mX\in \R^{n\times f}$. 
We denote $y_u \in \sY = \{1,2,\cdots,c\}$ to be the label of a node $u$ that belongs to one of the $c$ classes, 
and the vector $\vy\in \sY^{n}$ to denote the labels of all nodes. 
Additionally, the graph may have associated edge features of dimension $f_e$. 
The degree of node $u$ is denoted by $d_u$, the average degree by $d$, 
and 
the set $\gN(u)$ denotes the set of neighboring vertices of a node $u$. 
For a GNN, $\ell$ denotes the number of layers,  $H$  the number of neurons in the hidden layer, and $\mW^i$  the learnable weight of the $i$-th layer.

\subsection{Homophily Measures}
\label{subsec:homophiliydefinition}
The \emph{homophily} of a graph characterizes how likely vertices with the same labels are neighbors of each other. Many measures of homophily fit this definition, and Appendix~\ref{subsec:appendixhomophily_measure} discusses  seven of them, with the values computed on our benchmark dataset. However, for conciseness, we will focus here on {\em node homophily} ($\gH_{n}$) (intuitive), and {\em adjusted homophily} ($\gH_{a}$) (handles class imbalance).
The {\em local node homophily} of node $u$ is $\gH_n(u)  = \frac{| \{v\in \gN(u) : y_v = y_u\}|}{|\gN(u)|}$, and its mean value  {\em node homophily}~\citep{pei2020geom}is defined as follows:  
\begin{align}
\gH_{n} & = \frac{1}{|\gV|} \sum_{u\in \gV} \gH_n(u).
\end{align}
The {\em edge homophily}~\citep{zhu2020beyond} of a graph is 
$
    \gH_{e} = \frac{|\{(u,v)\in \gE : y_u = y_v\}|}{|\gE|}. 
$
Let $D_k = \sum_{v:y_v=k}d_v$ denote the sum of degrees of the nodes belonging to class $k$. Then the 
{\em adjusted homophily}~\citep{platonov2022characterizing} is defined as
\begin{equation}
    \gH_{a} = \frac{\gH_{e}-\sum_{k=1}^{c} D_k^2/(2|\gE|^2)}{1-\sum_{k=1}^c D_k^2/2|\gE|^2}. 
\end{equation}



The values of the node homophilies and edge homophily range from $0$ to $1$, and the adjusted homophily  ranges from $-\frac{1}{3}$ to $+1$ (Proposition 1 in~\citep{platonov2022characterizing}). 
In this paper, we will classify graphs with adjusted homophily less than $0.50$ as heterophilic.
Fig.~\ref{fig:local_node_homophily}
shows the distribution of the {\em local node homophily} of a  homophilic and a heterophilic graph. We see that both the graphs have a mix of locally heterophilic and locally homophilic nodes. 

\begin{figure}[!htbp]
\centering
\subfloat[Reddit (homophilic)]{%
\includegraphics[width=0.45\linewidth]{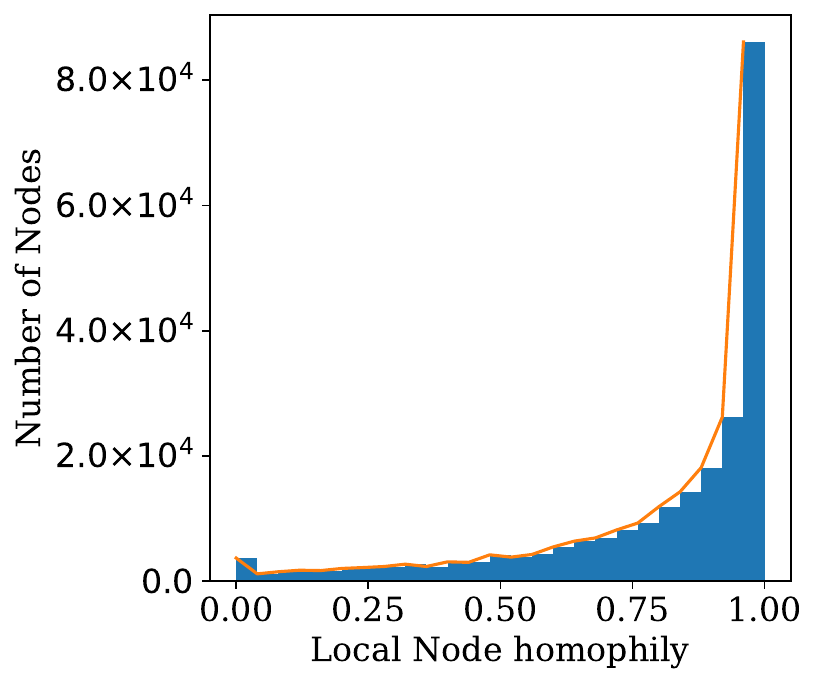}%
}
\hfill
\subfloat[Penn94 (heterophilic)]{%
\includegraphics[width=0.45\linewidth]{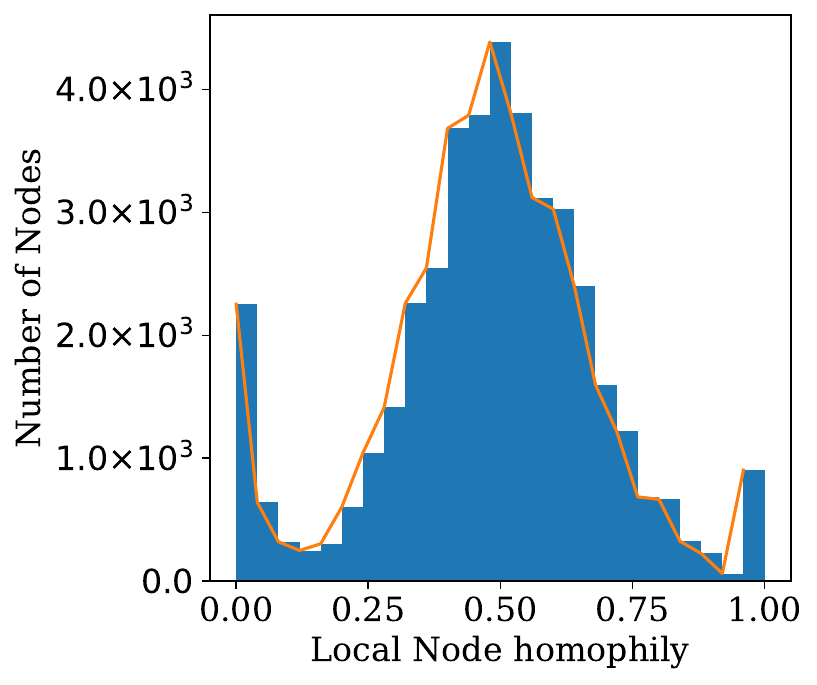}%
}

\caption{The distribution of local node homophily in a homophilic and a heterophilic graph. 
Figs.~\ref{fig:homophily_distributions_1} and \ref{fig:homophily_distributions_2} in Appendix~\ref{subsec:distribution} show this for all datasets.
}

\label{fig:local_node_homophily}
\end{figure}

\subsection{Effect of Homophily on Classification }
\label{subsec:synthetic}
To highlight the effect of homophily, we conduct an experiment on node classification using synthetic graphs with different levels of node homophily. These synthetic graphs are generated from the existing graphs, ignoring the structure information of the graph but retaining the node features and class labels. 
Following~\citep {luan2022revisiting}, to generate an undirected graph with an average degree of $d$ and node homophily $\gH_n$, for each node $u$  we
randomly assign $\gH_n\cdot d/2$ edges from the same class as $u$ and $(1-\gH_n) \cdot d/2$ edges from other classes. 
We left the class distribution unbalanced, as it is in the original graph, making it more challenging for GNNs since the neighborhood of a heterophilic node could potentially have more nodes from the majority class.\looseness=-1 

Fig.~\ref{fig:synthetic_hetrophily} shows the performance of GSAGE~\citep{hamilton2017inductive} (a homophilic GNN) and  ACM-GCN~\citep{luan2022revisiting} (a heterophilic GNN) on the synthetic graphs generated from Squirrel and Chameleon~\citep{rozemberczki2021multi} datasets. We use two versions of ACM-GCN, one with three channels (low-pass, high-pass, and identity) and the other with an additional channel with graph structure information (ACM-GCN-struc).
The original Squirrel and Chameleon datasets are heterophilic, and the ACM-GCN-struc is the best-performing. For synthetic graphs, on both Squirrel and Chameleon, (Fig.~\ref{fig:synthetic_hetrophily}), we see that surprisingly the worst $F_1$ score is not achieved on the graphs whose homophily value is zero but for values near $0.25$.
When the homophily score is high, GNNs perform well since it aggregates relevant information, but as we observe ACM-GCN also does well at a very low homophily.
This is because some locally heterophilic nodes become easier to classify after neighborhood aggregation on features (note that features are not considered in the definition of homophily based solely on labels). 
Another intuitive reason is that when two nodes are adjacent to the same dissimilar (wrt class labels) neighbors, the high pass filters (e.g., graph Laplacian) used in ACM-GCN treat these nodes as similar and can classify them correctly (Appendix~\ref{sec:gnnfilter}). 

\begin{figure}[t]
\centering
\subfloat[Squirrel Synthetic]{%
\includegraphics[width=0.4\linewidth]{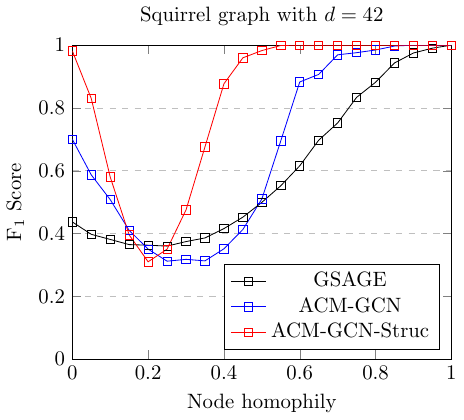}%
}
\hfill
\subfloat[Chameleon Synthetic]{%
\includegraphics[width=0.4\linewidth]{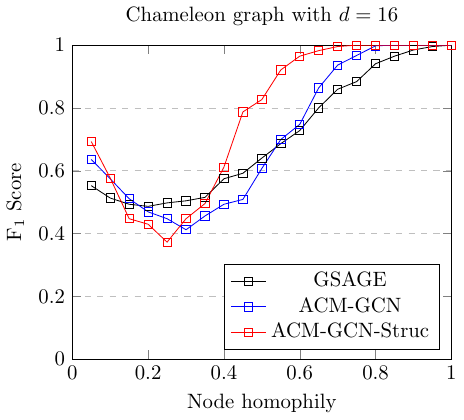}%
}
\caption{$F_1$ Score comparison of GSAGE and ACM-GCN on synthetic graphs generated from Squirrel (a) and Chameleon (b) datasets with varying node homophily. }
\label{fig:synthetic_hetrophily}
\end{figure}



\subsection{Similarity, Diversity, and Homophily}
\label{subsec:prelimhomophily}

Nodes with similar features tend to have similar labels~\citep{van2020survey}. We carried out an experiment to validate this statement and show in Appendix~\ref{subsec:featurevslabel} that the labels of nodes often correlate positively with their features.
Therefore, if instead of sampling the neighbors of a node $u$ uniformly at random, we sample neighbors that are similar to $u$ in feature space,
we are likely to increase the homophily of the sampled subgraph. 

However, this strategy alone is not enough for heterophilic graphs as they include both locally homophilic and heterophilic nodes (Fig.~\ref{fig:local_node_homophily}).
Two heterophilic nodes with the same label are expected to have similar class label distributions in their neighbors. In other words, the diversity in their class labels makes them similar.
It has been shown~\citep{luan2022revisiting,zheng2022graph, xu2023node} that high-pass filters~\citep{ekambaram2014graph} (e.g., variants of graph Laplacians: $\mL=\mD-\mA$, $\mL_{sym}=\mD^{-1/2}\mL\mD^{-1/2}$, or $\mL_{rw}=\mD^{-1}\mL$) capture the difference between nodes,  and low-pass filters (e.g., scaled adjacency matrices, $\mA_{sym}=
\mD^{-1/2}\mA \mD^{-1/2}$, or $\mA_{rw}=\mD^{-1} \mA$) retain the commonality of node features (Appendix~\ref{sec:gnnfilter}). 
Therefore, if we sample diverse neighbors (in feature space) and use a GNN with a high-pass filter, we expect a higher chance of mapping two heterophilic nodes of same class to the same space (after feature transformation) since they will have the same dissimilar neighborhood.

A  mathematical approach to ensure diversity is through 
\emph{submodular function} maximization~\citep{krause2014submodular,schreiber2020apricot}. 
A submodular function is a set function that favors a new node that is most distant in feature space to add to a partially sampled neighborhood. It accomplishes this by 
defining a suitable marginal gain of the new element with respect to the 
current neighborhood, and maximizing this value. 
However,  employing only diverse neighborhoods  can also cause issues since two nodes with different labels may have similar neighborhoods after sampling. In this scenario,  sampling based on similarity is more appropriate.
For the {\em spectral} domain, \ags{} 
considers two channels: one with a sampled subgraph ensuring diversity (used with a high-pass filter) and the other with a subgraph sampled based on similarity (used with a low-pass filter). Similar to ACM-GCN, we can also use an identity channel. 
However, {\em spectral} GNNs are difficult to scale as they often do not support mini-batching, and  
are  transductive. Hence we consider  {\em spatial} GNNs for heterophilic graphs,   
which employ  graph  topology to develop aggregation strategies. However, for heterophilic graphs, both similar and dissimilar neighbors need to be considered, and  in \ags{}, we achieve this through attribute-guided biased sampling of similar and diverse samples. \looseness = -1


\subsubsection{Node Homophily with Similar and Diverse neighborhood:}


Consider an ego node $t=\{\vx_t,y_t\}$ with feature $\vx_t$,  label $y_t$, and local node homophily $\gH_n(t)$. Let the feature and label tuples of the neighbors  of $t$ be 
$\gN(t)=\{(\vx_1,y_1),(\vx_2,y_2),\cdots, (\vx_{d_t},y_{d_t})\}$.
From the definition of homophily, the probability of randomly selecting a neighbor $i\in \gN(t)$ with the same label as the ego node is $P_\gU(y_i=y_t)=\gH_n(t)$. Here, $\gU$ refers to the distribution of selecting a neighbor uniformly at random.
Let $s(\vx_i,\vx_t)$  be a positive similarity score in features between a neighbor, $i$, and the ego node, $t$. \looseness=-1 

\begin{assumption}
    For a node $t$,  the average similarity of neighbors with the same label as $t$  is greater than  or equal to the average similarity of all neighbors. 
\end{assumption}

\begin{lemma}
If the probability of selecting a neighboring node is proportional to its similarity to the ego node $t$, the local node homophily of sampled neighborhood $\gH_n'(t)\ge \gH_n(t)$. If the  sampling probability distribution is $\gS$,  then $P_\gS(y_i=y_t)\ge P_\gU(y_i=y_t)$.
\label{lm:similar}
\end{lemma}

To retrieve a diverse set of neighbors we can employ a submodular function. 
An example  could be the {\em facility location} function based on maximizing pairwise similarities between the points in the data set and their nearest chosen point, 
$
    f(S, A) = \sum_{y\in A} \max_{x\in S} \phi(x, y),   
$
where $A$ is the ground set, $S \subseteq A$ is a subset, and $\phi$ is the similarity measure.
In our context $S$ is the current set of selected nodes initialized with ego node $t$, and ground set $A = \gN(t)\cup \{t\}$. 
The \emph{marginal gain} is $f_i = f(S\cup \{i\},A)-f(S,A)$ for each neighbor $i\in \gN(t)\setminus S$.  Successively,  neighbors are added to the sampled neighborhood by choosing them to have maximum marginal gain with respect to the current sample of the neighborhood.

\begin{assumption}
    The average marginal gain of the neighbors  of a node $t$ with the same label as $t$  is less than  or equal to the average marginal gain of all neighbors.
\end{assumption}

\begin{lemma}
If the probability of selecting a neighboring node is proportional to its marginal gain wrt the ego node $t$, then the  local node homophily of sampled neighborhood $\gH_n'(t)\le \gH_n(t)$. If the sampling probability distribution is $\gD$,  then $P_\gD(y_i=y_t)\le P_\gU(y_i=y_t)$.
\label{lm:diverse}
\end{lemma}


Proofs of Lemmas~\ref{lm:similar} and~\ref{lm:diverse} are included in the Appendix~\ref{subsec:prob_similar} and Appendix~\ref{subsec:prob_diverse} respectively. We also report experimental verification of these properties in Appendix~\ref{subsec:emphlemma}. 
From the Lemma~\ref{lm:similar}, sampling neighbors based on feature-similarity improves homophily, which potentially helps to map homophilic nodes of the same labels into the same space.
We assume features and labels to be positively correlated; thus, if we ensure feature diversity among neighbors, we can also expect label diversity in the samples, reducing local homophily 
(as shown in Lemma~\ref{lm:diverse}) and increasing the chances of mapping two heterophilic nodes into the same space.
We devise feature-similarity and feature-diversity-based sampling based on these results in the next section.

\section{Proposed Method: \ags{}}
\label{sec:method}



Fig.~\ref{fig:agssampling_paradigm} shows an overview of the \ags{} framework. 
\ags{} has a \emph{pre-computation} step that ranks the neighbors of nodes and computes edge weights to form a distribution to sample from in the \emph{sampling} phases during training. 

       

\begin{figure*}[!htbp]
 \centering

 \subfloat[Rank of neighbors for similar and diverse sampling.]{%
\includegraphics[width=0.5\linewidth]{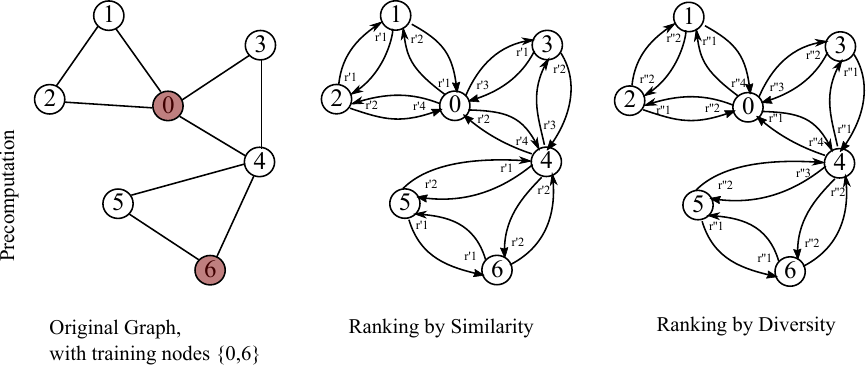}%
}
\hfill
\subfloat[Node Sampling from training vertices with hop size $2$]{%
\includegraphics[width=0.45\linewidth]{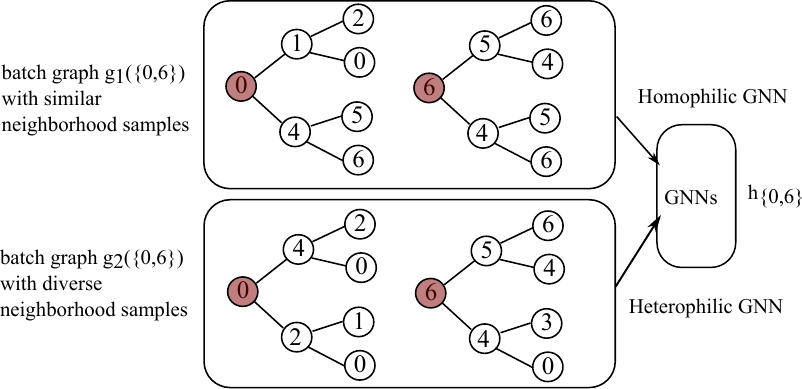}%
}

 \caption{\ags{} framework with Node Sampling. a) Pre-computation step to rank the neighbor of vertices. b) Demonstrates how weighted node sampling is performed based on the selection probabilities of ranked neighbors.}
 \label{fig:agssampling_paradigm}
\end{figure*}
\subsection{Pre-computing Probability Distribution}
\label{subsec:precomputation}
An ideal graph sampling process should have the following key requirements: 
\begin{enumerate*}[label=(\roman*)]
\item Nodes with strong mutual influences (in terms of structure and attributes) should be sampled together in a subgraph~\citep{zeng2019graphsaint}.
\item Should be able to distinguish between similar and dissimilar neighbors (especially for heterophilic graphs)~\citep{zheng2022graph}.
\item Every edge should have a non-zero probability of being sampled 
in order to generalize and explore the full feature and label space.  
\end{enumerate*}

In this section, we will devise sampling strategies satisfying these requirements. We assume that we have access to a similarity measure between any two nodes in the graph. This similarity function typically depends on the problem and the dataset. For an example, in text based datasets, we may use cosine similarity of the feature vectors (e.g., \texttt{TF-IDF}) 
generated from the texts of the corresponding items. We may also learn the similarities when an appropriate similarity measure is not apparent. Once we have the similarity function, for each vertex $u$, we construct a probability distribution over its neighbors as follows: 
\begin{enumerate*}[label=(\roman*)] 
\item \emph{rank} the neighboring nodes ($\gN(u)$) using the similarity scores to $u$.
\item assign weights to the adjacent edges of $u$ based on this ranking. A few choices are 
shown in Fig.~\ref{fig:densityfunction} in Appendix.
\item normalize the weights to construct the probability mass function (PMF) $\gP(u)$ of the distribution over $\gN(u)$.
\end{enumerate*}

We construct the probability distribution  using the rank rather than the actual similarity values, since the distribution using the similarity values could be skewed and extremely biased towards a few top items.
Here we consider two choices of rankings of neighborhoods that are suited for homophilic and heterophilic graphs.

\subsubsection{Ranking based on similarity:}
\label{subsubsec:precompnodesample}
For this case, our goal is to sample subgraphs favoring similar edges to be present more frequently in the subgraph. To achieve that we propose to construct a probability distribution over the edges based on the similarity scores. 
We sort the similarity values of all the neighboring vertices of a vertex from high to low and assign ranking based on the order to the adjacent edges. Thus, although the similarity function is symmetric, we may get two different ranks for each edge. 
Note that the computation required to generate the rankings are local to each vertex and thus highly parallel. For cosine similarity the time complexity to compute the ranking of the adjacent edges of a vertex $u$ is $\gO(fd_u \log{d_u})$.

Once we have the rankings, we can use these to construct different probability distributions. Some choices of Probability Mass Functions (PMF)  can be {\em linear} or {\em exponential} decay with non-zero selection probability to the later elements in the ranking order. Another options is the {\em step} function where the top $k_1\%$  neighbors of a vertex are given a uniform weight ($\lambda_1$), the next $k_2\%$ as $\lambda_2$, and the rest $\lambda_3$, where $\lambda_1 > \lambda_2 > \lambda_3$. The benefit of using such function is that we can partially sort the top $(k_1+k_2)\%$ of neighbors avoiding the full ordering. Fig.~\ref{fig:densityfunction} in Appendix~\ref{sec:apadditonaldetails} shows pictorial representations of some of the PMFs mentioned.

\subsubsection{Ranking based on diversity:}
\label{subsec:rank-div}
As discussed in \cref{subsec:prelimhomophily}, for a heterophilic graph, in general it is desirable to construct subgraphs based on diversity in the class labels. To accomplish that, we propose a sampling strategy based on optimizing (a nonlinear) submodular function. A submodular function is a set function that has the diminishing returns property. Formally, given a ground set $V$ and  two subsets $A$ and $B$ with $ A \subseteq B \subseteq V$, a function $f$ is submodular if and only if, for every $e \in V \setminus B$, $f(A \cup e) - f(A) \geq f(B \cup e) - f(B)$. 
The quantity  on either side of the inequality is called the \emph{marginal gain} of an element $e$ with respect to the two sets $A$ or $B$. 
For maximizing  cardinality constrained submodular functions, where the solution set ($S$) is required to be at most $k$ in size, a natural \emph{Greedy} algorithm that starts with empty solution and augments it with the element with highest marginal gain is  $(1-1/e)$-approximate~\cite{nemhauser1978analysis}. 

To see how a submodular function may behave differently than the (linear) similarity based function (\cref{subsubsec:precompnodesample}), consider a paper citation graph~\citep{sen2008collective} where the nodes are the scientific documents and the feature on each node is the  binary count vector of the associate text. 
For a vertex $u$, our goal is to find a set of $k$ neighboring vertices of $u$, ($S \subseteq \gN(u)$ where $|S| = k$), such that given the initial set $\{u\}$, we maximize the number of \emph{unique} word counts. This objective can be modeled as a submodular function known as maximum $k$-coverage problem, and the Greedy approach would select successive nodes with maximum marginal gains wrt to $S$.
Intuitively, the Greedy algorithm prioritizes neighbors that have more distinct words than what has been covered through the selected nodes. Therefore, if different word sets correspond to different class labels, the final set $S$ will likely represent a diverse set. This contrasts sharply with the ranking based on similarity, where we would encourage neighbors with similar words. 
{\em Facility Location, Feature-based functions},  and {\em Maximum Coverage} are some submodular functions applicable in the graph context. \looseness=-1

Given a submodular function, for a vertex $u$, we execute the Greedy algorithm on the neighbors of $u$ to compute their marginal gains, assuming $u$ is in the initial solution. We use these marginal gains to rank neighbors of $u$ and then use the ranks to construct a probability distribution as described in \cref{subsubsec:precompnodesample}. 
%
To compute the solution faster, we employed a variant of the Greedy algorithm which is called Lazy Greedy~\cite{minoux1978accelerated} that can reduce the number of marginal gain computations. Algorithm~\ref{alg:subweights}  and \ref{alg:lazygreddy} in Appendix~\ref{subsec:submodular_funcs} shows the pseudocode of the Lazy Greedy version of the ranking procedure using the {\em Facility Location} function. Since we need to compute pairwise similarity for this function,  the complexity of computing the ranks of the neighbors of a vertex $u$  is $\gO(fd_u^2)$. 

\subsubsection{Learnable Similarity Function:}
\label{subsubsec:learnable}
When we lack domain knowledge to compute an appropriate similarity metric, we can use the training subgraph to train a regression model to learn the edge weights of the graphs~\citep{chien2020adaptive}. We form training batches using equal numbers of  edges and randomly chosen  non-edges of the graph. We set the target label as $1$ for nodes with the same labels and $0$ otherwise, and  then train a model using the batches to get an approximation of edge weights. 
We use these weights as similarity values to compute ranks based on similarity or diversity (as detailed in \cref{subsubsec:precompnodesample} and \cref{subsec:rank-div}),  and to compute a probability distribution over the directed edges. 
Given the number of training edges and an equal number of non-edges, we can expect the computation complexity to be $\mathcal{O}(fm)$. The neural network model is shallow, learnable parameters are relatively small, and computation is performed in batches, making it very fast.
Appendix~\ref{subsec:edge_weight} shows the architecture of a Siamese~\citep{chicco2021siamese} edge-weight computation model and the training and inference process (Algorithm~\ref{alg:learnweights}) in detail.

\subsection{Sampling}
\label{subsec:node_sampling}


Algorithm~\ref{alg:agsnodesampling} describes  \ags{}, using the node sampling process for training denoted as \agsns{}. The method {\em NodeSample} is employed for sampling an $\ell$-hop neighborhood based on the probability distribution derived from learned or computed weights. 




\begin{algorithm}[htbp]
\small
\caption{\agsns{} ($\gG, \mX, \vy$) [Node Sampling]} 
\SetKwInOut{Input}{Input}
\Input{Graph $\gG(\gV, \gE)$, Feature matrix $\mX \in \R^{n\times f}$, Label, $\vy\in \sY^n$}

\begin{algorithmic}[1]
    \STATE  $\gS =  \mathrm{LearnSimilarity (\gG,\mX,\vy)}$
    \STATE  $\gP =  \mathrm {'step'}$ \tcc{probability mass function}
    
    \STATE $\mathrm{R_1 = RankBySimilarity(\gG, \mX, \gS, \gP)}$\tcc{Alg.~\ref{alg:nnweights}}

    \STATE $\mathrm{R_2 = RankByDiversity(\gG, \mX, \gS, \gP)}$\tcc{Alg.~\ref{alg:subweights}}
    
    \FOR{$\mathrm {epoch}$ \textbf{in} $\mathrm{num\_epochs}$}
        \STATE $\mathrm{nodes, y_{nodes} = BatchOfTrainingNodes (\gG,\vy)}$ \tcc{random nodes from training vertices}
        \STATE $\mathrm{g_1 = NodeSample(R_1, nodes)}$ \tcc{similar samples}
        \STATE $\mathrm{g_2 = NodeSample(R_2, nodes)}$ \tcc{diverse samples}        
        \STATE $\mathrm{output = GNN_\mW(g_1, g_2)}$
        \STATE  $l = \mathrm{loss (output, y_{nodes}}$)
        \STATE Backward propagation from $l$ to update $\mW$ 
    \ENDFOR
\end{algorithmic}
\label{alg:agsnodesampling}
\end{algorithm}

AGS can be plugged in with other GNNs that use sampling strategy. Here, we demonstrate this with GSAGE and call it AGS-GSAGE.
The original formulation of GSAGE for $i$-th layer is, 
\begin{equation}
    \vh_{v}^{(i)} = \sigma(\mW^{(i)}, \vh_v^{(i-1)}|\mathrm{AGG}(\{\vh_{u}^{(i-1)}:u\in \gN_{rn}(v)\}).
\end{equation}
Here $\vh$ denotes feature vectors, 
$\gN_{rn}(v)$ ($\subseteq \gN(v)$) is the subset of neighbors of $v$ (of size $k$) sampled uniformly at random. AGG is any permutation invariant function (\texttt{mean, sum, max, LSTM}, etc.). 

In AGS-GSAGE, we use the probability distribution derived in \cref{subsec:precomputation} to sample with or without replacement from the neighborhood of a node. 
Depending on the nature of the graph, single or dual-channel GNNs might be necessary. 
Some homophilic graphs may have only a few locally heterophilic nodes, where samples from a distribution generated by similarity only may be sufficient. However, we expect the dual-channel AGS to perform better for heterophilic graphs since they typically have both locally homophilic and heterophilic nodes (see Fig.~\ref{fig:local_node_homophily}, and Fig.~\ref{fig:homophily_distributions_1}, 
Fig.~\ref{fig:homophily_distributions_2} from Appendix~\ref{subsec:distribution}). 
Fig.~\ref{fig:nodesampling_computation_graph} shows some possible computation graphs.
The dual-channel AGS mechanism can be incorporated easily into GSAGE. We generate two similar and diverse neighborhood samples for the target node, compute the transformed feature representations using both  samples, and use MLPs to combine these representations. 

\begin{align}
    \vh_{v_{sim}}^{(i)} = \sigma(\mW^{(i)}_{sim}, \vh_v^{(i-1)}|\mathrm{AGG}(\{\vh_{u}^{(i-1)}:u\in \gN_{sim}(v)\})\\
    \vh_{v_{div}}^{(i)} = \sigma(\mW^{(i)}_{div}, \vh_v^{(i-1)}|\mathrm{AGG}(\{\vh_{u}^{(i-1)}:u\in \gN_{div}(v)\}).
\end{align}

Here $\gN_{sim}(v)$ and $\gN_{div}(v)$ are the subset of neighbors of $v$ sampled from the distribution based on similarity and diversity, respectively. We can combine these representations by concatenating, $\vh_{v}^{(i)} = \mathrm{MLP}(\vh_{v_{sim}}^{(i)}|\vh_{v_{div}}^{(i)})$
or using skip connections,
\begin{align}
\vh_{v}^{(i)} = \mathrm{MLP}(\sigma(\mW(\vh_{v_{sim}}^{(i)}|\vh_{v_{div}}^{(i)})+\vh_{v_{sim}}^{(i)}+\vh_{v_{div}}^{(i)})).    
\end{align}




\begin{figure}[t]
 \centering
 \includegraphics[width=1.0\linewidth, angle = 0 ]{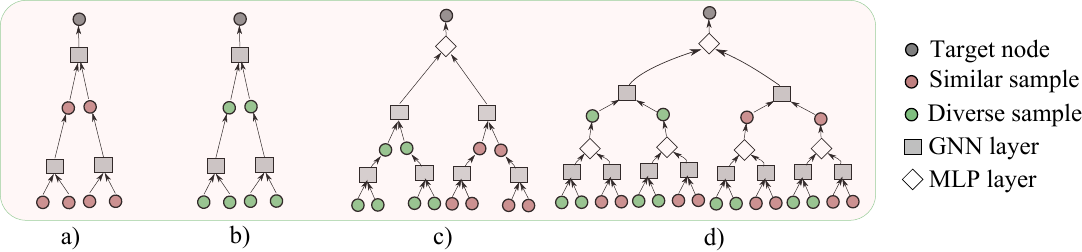}
 \caption{Computation graph with sample size $k=2$ and hop-size $2$. a), b) samples from similarity and diversity ranking for a single channel, c) dual channel with combined representation at the target node, and d) similar and diverse weighted samples at each sampled node.}
 \label{fig:nodesampling_computation_graph}
\end{figure}

Combining representations at the root of the computation graph as shown in Fig.~\ref{fig:nodesampling_computation_graph}c can be better than combining two different samples at each node of the tree, as in Fig.~\ref{fig:nodesampling_computation_graph}d. This is to avoid overfitting and make it computationally efficient. 
For transductive learning, as we have already precomputed the probability distributions, the inference process works similarly  to the training phase. 
In inductive setting, 
we have to compute the probability distribution over the neighbors of a node as described in \cref{subsec:precomputation}. 

\subsubsection{Other sampling strategies and models:}
AGS can also be integrated into ACM-GCN~\citep{luan2022revisiting} and other filter-based spectral GNNs. At the start of the epoch, we sample $k$ neighbors of each node from the similarity and diversity based distributions and construct two sparse subgraphs. 
We can then use the subgraph based on similarity with a low-pass filter and the subgraph based on diversity with a high-pass filter.
We can employ  \emph{graph sampling} strategies  (instead of node sampling), 
such as weighted random walks, and use them with existing GNNs.
We can also incorporate heuristics by first computing edge-disjoint subgraphs and then sampling a sparse subgraph. 
We call our graph sampling GNN  \agsgs{}, and its details are provided in Appendix~\ref{sec:graphsampling} for conciseness. We call \agsgsrw{} (Appendix~\ref{subsec:agsgsrw}) the  weighted random walk version,  and call \agsgsdis{} (Appendix~\ref{subsec:agsgsdisjoint}) the  edge-disjoint subgraph-based sampling version).
For the downstream model, there is flexibility to adapt existing models like ChebNet~\citep{he2022convolutional},
GSAINT~\citep{zeng2019graphsaint}, GIN~\citep{xu2018powerful}, GCN~\citep{kipf2016semi}, and GAT~\citep{velivckovic2017graph}. We can also use two separate GNNs in two separate channels. 



\subsection{Computation Complexity}
\label{subsec:method_complexity}

The pre-computation of the probability distribution requires $\mathcal{O}(fn\cdot d\log d)$ and $\mathcal{O}(fnd^2)$ operations for similarity and facility location based ranking, respectively. If the similarity metric is required to be learned the added time complexity is $\mathcal{O}(fm)$. The training and memory complexity depend on the usage of underlying GNNs. Let the number of hidden dimensions ($H$) be fixed and equal to  $H=f$ for all $\ell$ layers. For Stochastic Gradient Descent (SGD)-based approaches, let $b$ be the batch size, and $k$ be the number of sampled neighbors per node. When a single channel is used, each node expands to $k^\ell$ nodes in its computation graph, and requires $f^2$ operations to compute the embedding 
(a matrix-vector multiplication) in every epoch. Therefore, for $n$ nodes, the per epoch training time complexity is $\mathcal{O}(nk^\ell f^2)$.
The memory complexity is $\mathcal{O}(bk^\ell f+\ell f^2)$, where
The first term corresponds to the space for storing the embedding, and 
the second term corresponds to 
storage for  all weights of neurons of size, $\mW \in \R^{f\times f}$.

For single channel node sampling (Fig.~\ref{fig:nodesampling_computation_graph}a,b), the training and memory complexity of \ags{} is similar to GSAGE. The training and memory 
requirements are twice as much for dual channels using the same sampling size, leaving the asymptotic bounds unchanged. However, for the dual channel with computation graph scenario shown in Fig.~\ref{fig:nodesampling_computation_graph}d, where each node generates two types of samples of size $k$, the per epoch computation complexity becomes $\mathcal{O}(n 2^\ell k^\ell f^2)$.
One way to ameliorate  this cost is to reduce the sample neighborhood size by half.

\section{Related Work}
\label{subsec:related_works}

While {\em Spatial} GNNs focus on graph structure (topology) to develop aggregation strategies, {\em spectral} GNNs leverage graph signal processing to design graph filters.
Spectral GNNs use low-pass and high-pass filters to extract low-frequency and high-frequency signals adaptively for heterophilic graphs. 
ACM-GCN~\citep{luan2022revisiting} is one of the best-performing heterophilic GNNs that uses adaptive channels with low-pass and high-pass filters.
Recently, the authors of ~\citep{xu2023node} proposed an adaptive filter-based GNN, ALT, that combines signals from two filters with complementary filter characteristics to classify homophilic and heterophilic graphs. 
These methods perform well for small heterophilic graphs but do not scale to large graphs.
In the \emph{spectral} domain, \ags{} can be used in conjunction with these approaches by computing feature-similarity and feature-diversity-based sparse graphs at first before applying filters for large graphs. \looseness=-1

For applying {\em spatial} GNNs  to heterophilic graphs, rather than using average aggregation (as in homophilic GNNs), edge-aware weights of neighbors can be assigned according to the spatial graph topology and node labels.  DMP~\citep{yang2021diverse} considers node attributes as 
weak labels and aggregates element-wise weights from neighbors.
GGCN~\citep{yan2022two} uses cosine similarity to create signed neighbor features; the intuition is to send positive attributes to neighbors in the same class and negative ones to neighbors in different classes. GPNN~\citep{yang2022graph} considers ranking the neighbors based on similarity and uses a 1D convolution operator on the sequential nodes.
Another related work is SimP-GCN~\citep{jin2021node}, which computes the node similarity and then chooses the top $k$ similar node pairs in terms of feature-level cosine similarity for each ego node, and then  
constructs the neighbor set using the  $k$-NN algorithm. \ags{}, in contrast, uses submodularity, node-similarity, reweighting, and sampling of the subgraph instead of reconstructing neighbor sets.


When learned weight functions are considered, \ags{} can be placed into the category of supervised sampling. 
LAGCN~\citep{chen2020label} trains an edge classifier from the existing graph topology, similar to our regression task of weight approximation.
NeuralSparse~\citep{zheng2020robust} learns a sparsification network to sample a $k$-neighbor subgraph (with a predefined $k$), which is fed to GCN~\citep{kipf2016semi}, GraphSAGE~\citep{hamilton2017inductive} or GAT~\citep{velivckovic2017graph}. Unlike our heuristic-based sampler, NeuralSparse has end-to-end training of the sampler network and GNN,  and may require more iterations to find appropriate sampling probabilities.


There are only a few scalable GNNs for heterophilic graphs. The most notable one is LINKX~\citep{lim2021large}, which is transductive as the model architecture depends on node size.
A recent scalable GNN, LD2~\citep{liao2023ld}, attempts to remedy this by transforming the adjacency and the feature matrix into embeddings as a pre-computation and then applying feature transformation in a mini-batch fashion. 

\section{Experiments}
\label{sec:experiments}

\subsection{Dataset, Setup, and Methods}

\subsubsection{Dataset:} We experimented with $35$  graphs of different sizes and varying homophily. We also generated synthetic graphs of different homophily and degrees while retaining the node features for ablation studies and scaling experiments. 
We considered the node classification task in our experiments.
For heterophily studies, we used:
\texttt{Cornell, Texas}, \texttt{Wisconsin} 
~\citep{pei2020geom}; \texttt{Chameleon}, \texttt{Squirrel} ~\citep{rozemberczki2021multi}; \texttt{Actor} ~\citep{pei2020geom}; \texttt{\justify Wiki, ArXiv-year, Snap-Patents, Penn94, Pokec, Genius, Twitch-Gamers, reed98, amherst41, cornell5}, and \texttt{Yelp}~\citep{lim2021large}. 
We also experiment on some recent benchmark datasets, \texttt{\justify Roman-empire, Amazon-ratings, Mineswe-eper, Tolokers}, and \texttt{Questions} from~\citep{platonov2023critical}.
We converted a few multi-label multiclass classification problems (\texttt{\justify  Flickr, AmazonProducts}) to single-label multiclass node classification, and their homophily values become relatively small, making them heterophilic.
For homophily studies we used
\texttt{\justify Cora}~\citep{sen2008collective}; \texttt{\justify Citeseer}~\citep{giles1998citeseer}; \texttt{\justify pubmed} \citep{namata2012query}; \texttt{Coauthor-cs}, \texttt{Coauthor-physics}~\citep{shchur2018pitfalls}; \texttt{Amazon-computers},  \texttt{Ama-zon-photo} ~\citep{shchur2018pitfalls}; \texttt{Reddit}~\citep{hamilton2017inductive}; \texttt{Reddit2}~\citep{zeng2019graphsaint}; and, \texttt{dblp}~\citep{fu2020magnn}. 
Details of datasets, homophily measures (Table~\ref{tab:dataset_description}), homophily distributions (Fig.~\ref{fig:homophily_distributions_1},~\ref{fig:homophily_distributions_2}), origins, splits, etc.  
(Table~\ref{tab:additonaldataset}) are in Appendix~\ref{sec:appendixdataset}.

\subsubsection{Experimental Setup:} All evaluations are repeated $10$ times 
using a 
split of $60\%/20\%/20\%$ (train/validation/test), unless a specific split is specified. 
All experiments are executed on 24GB NVIDIA A10 Tensor Core GPU. For all benchmark models, we use the settings specified in their corresponding papers, and for \ags{}, we use two message-passing layers ($\ell=2$) and a hidden layer with dimension $H=256$. We use the Adam~\citep{kingma2014adam} optimizer with a learning rate of $10^{-3}$ and train for $250$ epochs or until convergence. A model converges if the standard deviation of the training loss in the most recent $5$ epochs (after at least $5$ epochs) is below $10^{-4}$. 
For node sampling, we use a batch size of $512$ to $1024$ with $k=[25, 10]$ or $[8,4]$ in the two layers unless otherwise specified. For graph sampling, we use a batch size of $6000$ and a random walk step size of $2$.  For reporting accuracy, we select the model that gives the best validation performance. Depending on the models, we use a dropout probability of $0.2$ or $0.5$ during training. All of the implementations are in PyTorch~\citep{paszke2019pytorch}, PyTorch Geometric~\citep{fey2019fast}, and Deep Graph Library (DGL)~\citep{wang2019deep}. 
For computing distribution based on submodular ranking, we use Apricot library~\citep{schreiber2020apricot}. We modified the Apricot code to make the implementation more efficient.
Source codes for all our implementations are provided anonymously on GitHub\footnote{\href{https://github.com/anonymousauthors001/AGS-GNN}{GitHub link for \ags{} and all other methods.}}. \looseness=-1

\subsubsection{Implemented Methods:}
We consider graphs with $\gH_a \le 0.5$ to be heterophilic. The small instances contain graphs with less than $100K$ nodes, and they fit in the GPU memory. The larger instances are compared against only scalable homophilic and heterophilic GNNs. We compare \ags{} (the Node Sampling \agsns{} variant) with other Node Sampling methods (GSAGE~\citep{hamilton2017inductive}), GraphSampling (GSAINT~\citep{zeng2019graphsaint}), and Heterophilic GNNs (LINKX~\citep{lim2021large}, ACM-GCN~\citep{luan2022revisiting}).
LINKX is used for only the small graphs, 
where the entire graph fits into GPU memory. For large graphs, we used the row-wise minibatching of the adjacency matrix (\emph{AdjRowLoader}) for LINKX, which is denoted by LINKX$+$. 
Since, most of the heterophilic GNNs require entire graph and do not support mini-batching, we compare \ags{} with $18$ standard heterophilic and homophilic GNNs on small heterophilic graphs only. Appendix~\ref{subsec:othermethods} provides details on these method and results.

\subsection{Key Results}
\label{sec:keyresults}

\begin{figure*}[!htbp]
\centering
\subfloat[Small Heterophilic Graphs]{
\includegraphics[width=0.24\linewidth]{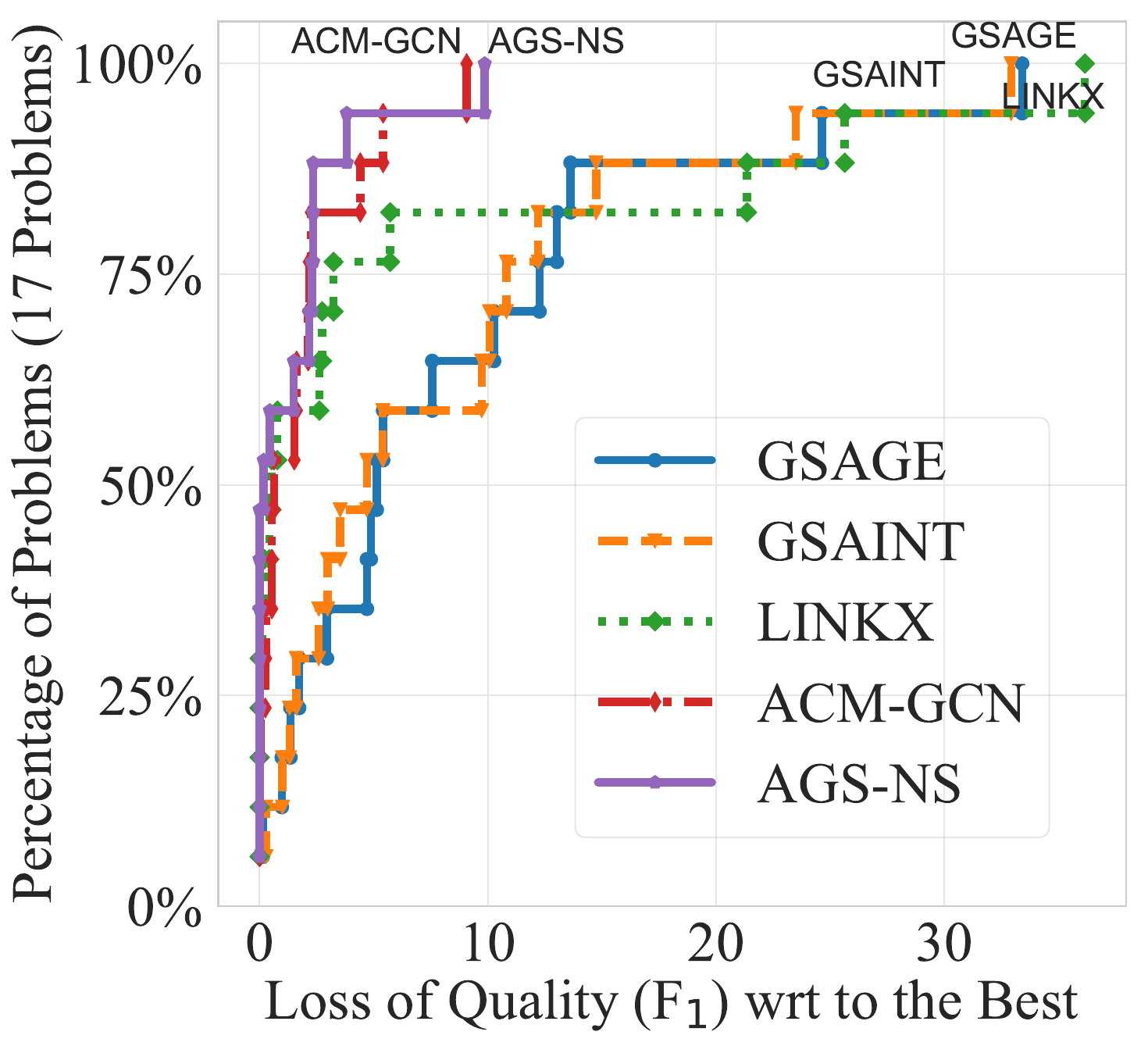}
\label{fig:smallhetero}
}
\subfloat[Small Homophilic Graphs]{
\includegraphics[width=0.24\linewidth]{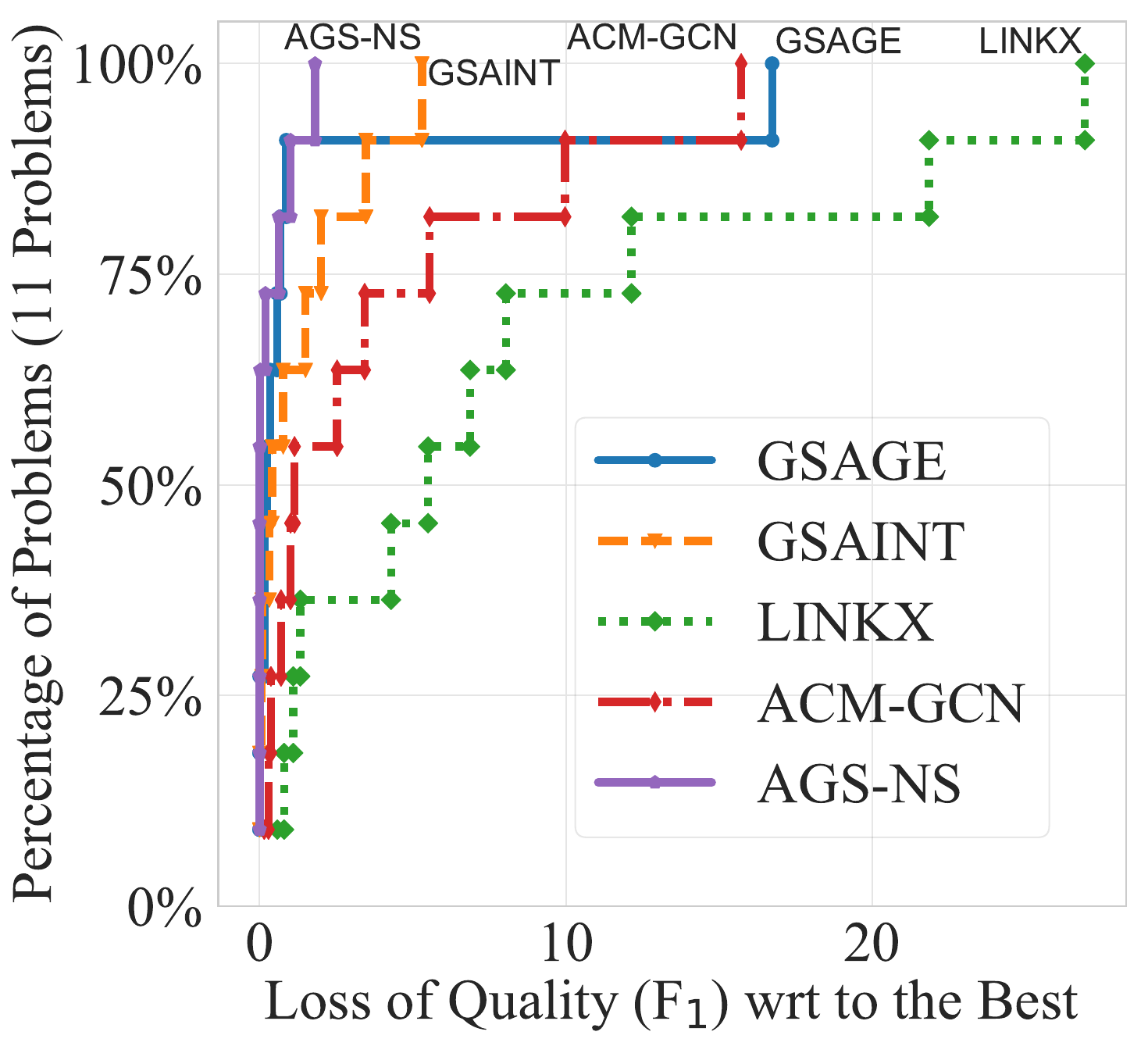}
\label{fig:smallhomo}
}
\subfloat[Large Heterophilic Graphs]{
\includegraphics[width=0.24\linewidth]{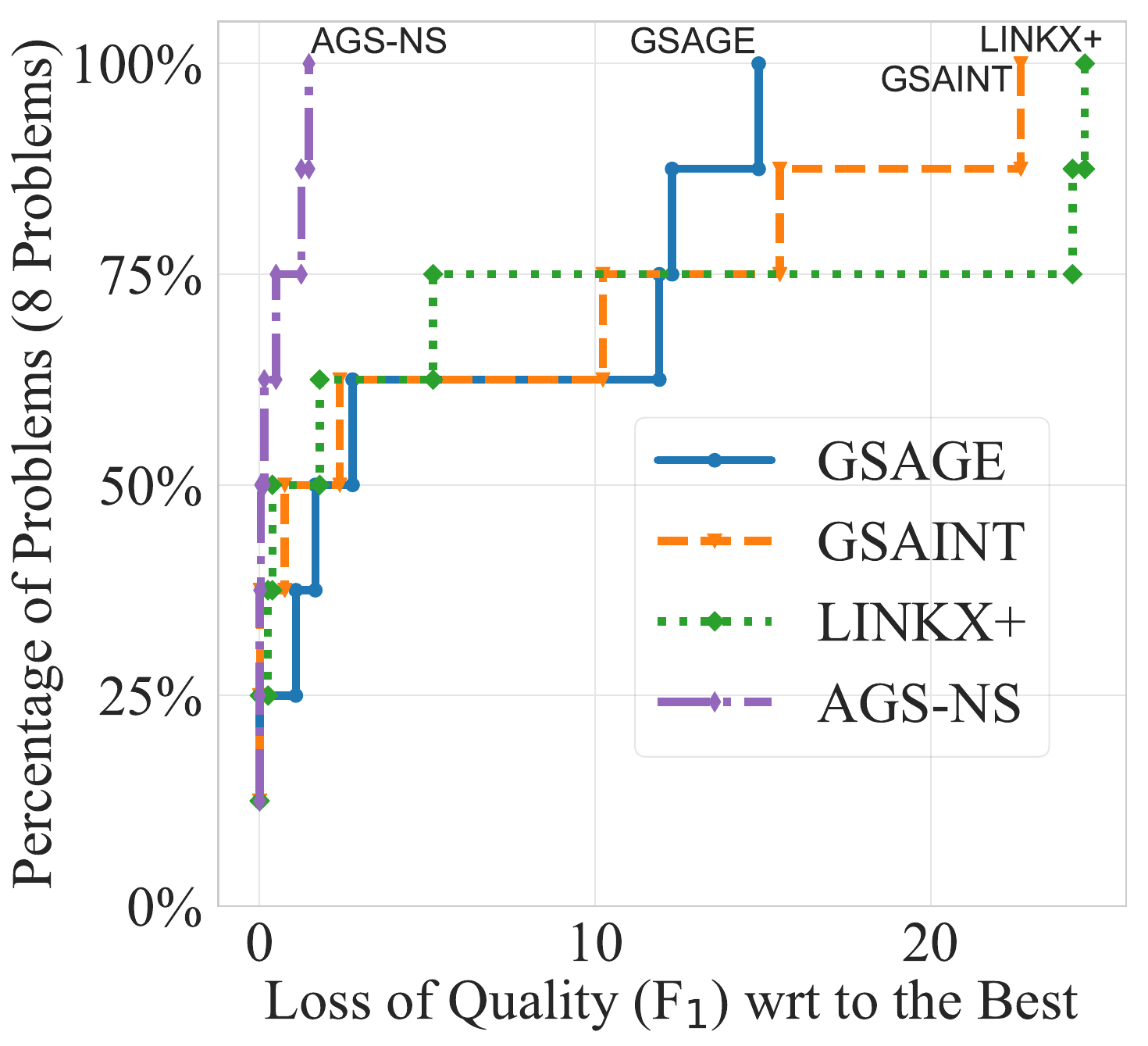}
\label{fig:largehetero}
}
\subfloat[Large Homophilic Graphs]{
\includegraphics[width=0.24\linewidth]{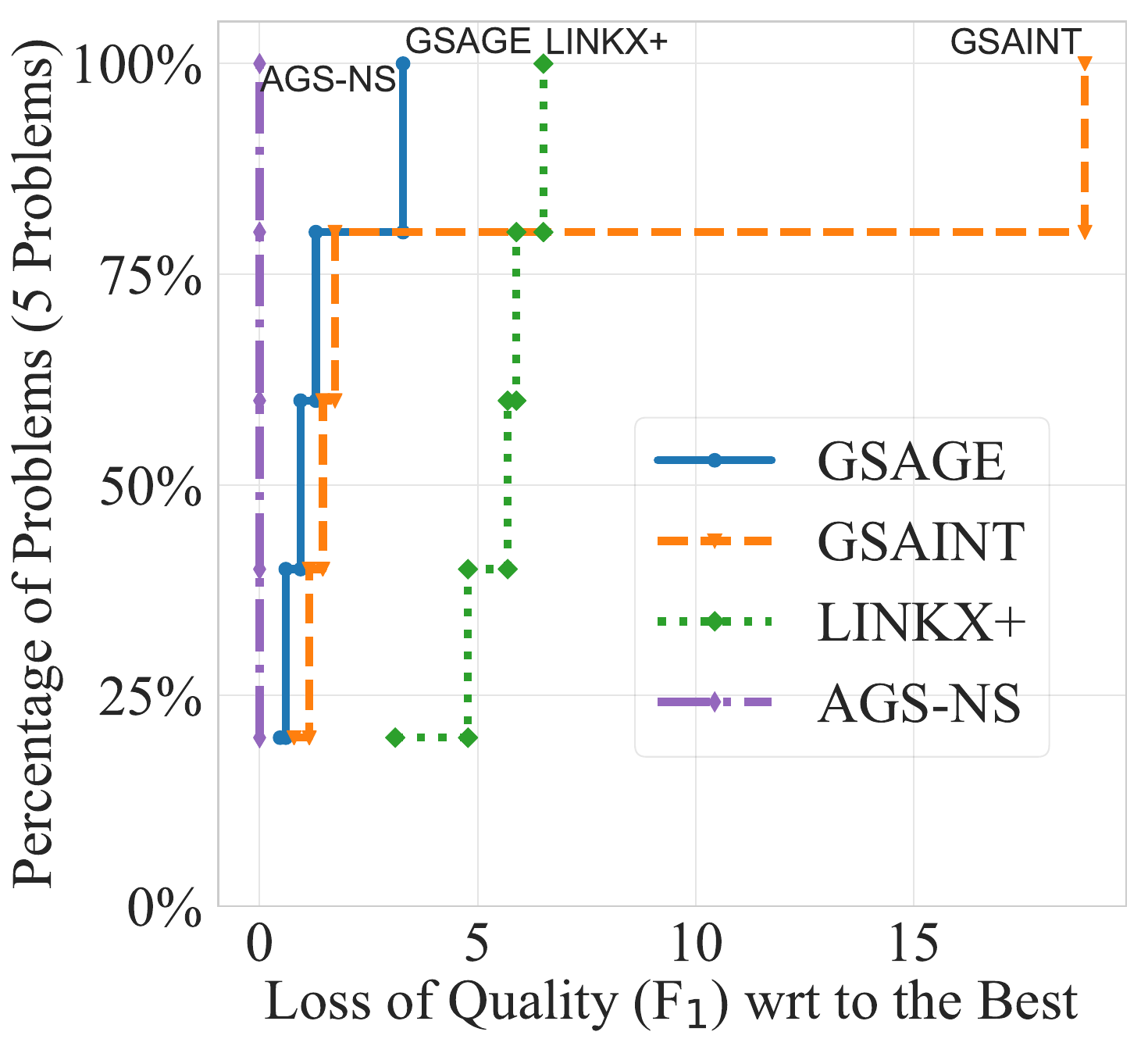}
\label{fig:largehomo}
}

\caption{{\em Performance Profile:} The X-axis shows the differences in $F_1$-scores (scaled to 100) between the best algorithm for a specific problem, and the Y-axis shows the fraction of the problems. We compare AGS to two scalable homophilic GNNs (GSAGE, GSAINT) and two heterophilic GNNs (ACM-GCN, LINKX). For small ($<100K$ vertices) and large ($\ge100K$ vertices)  graphs, we consider LINKX (full-batch) and LINKX$+$ (mini-batch), respectively. 
Full results are in Appendix~\ref{subsec:AGSNSresults}.}
\label{fig:perfplot}

\end{figure*}

In Fig.~\ref{fig:perfplot}, we show the performance profile plot of the algorithms w.r.t. their relative $F_1$-score. For each graph, we compute the relative $F_1$-score for all algorithms by subtracting their  $F_1$-score from the best one. 
Thus the best performing algorithm for a problem receives a score of  $0$, and for all other algorithms the difference is positive. The X-axis of Fig.~\ref{fig:perfplot} represents these relative values from the best-performing algorithms across the graph problems, and the Y-axis shows the fraction of problems that achieve $F_1$-score within the bound on the difference specified by the X-axis value. Thus, the performance plot reveals a ranking of the algorithms in terms of their quality. The closer a curve (algorithm) is to the Y-axis, and the smaller the difference in the X-axis, the better its performance w.r.t. other algorithms across all the problems considered. Fig.~\ref{fig:smallhetero} summarizes results from five algorithms across $17$ test problems for small heterophilic graphs, where we observe that \agsns{} performs the best or close-to-best for about half of the problems and has up to $10\%$ lower $F_1$ scores compared to the best algorithm for the other half. ACM-GCN performs similarly to \agsns{} for these small graphs. 
While LINKX achieves comparable accuracies to \agsns{} for most of the problems (about $80\%$), for a few problems it achieves lower $F_1$ scores.

For large heterophilic graphs (Fig.~\ref{fig:largehetero}),  performance improvement for \agsns{} is considerably better than all algorithms. LINKX$+$ performs second-best for $75\%$ of the problems.
In small homophilic graphs (Fig.~\ref{fig:smallhomo}), \agsns{} and GSAGE are the top two performers for most of the problems, followed by GSAINT, ACM-GCN, and LINKX. This is expected since ACM-GCN and LINKX are tailored for heterophilic graphs. In large homophilic graphs (Fig.~\ref{fig:largehomo}), \agsns{} is the best-performing algorithm in terms of accuracy, followed by GSAGE. We also observe from this figure that for homophilic graphs, LINKX$+$ is not competitive.

In Table.~\ref{tab:smallothergnns} (Appendix~\ref{subsec:othermethods}) we present  performance of \agsns{} compared to $18$ recent competing algorithms on small heterophilic graphs. 
ACM-GCN, \agsns{}, and LINKX remain the best-performing, with \agsns{} as the leading method for these inputs.
The full set of results with numerical values are provided in Appendix (Tables~\ref{tab:smallhetero},~\ref{tab:smallhomophilic},~\ref{tab:large_heterophilic_graphs},~\ref{tab:large_homophilic_graphs},~\ref{tab:smallothergnns}). 
These tables present the mean ($\mu$) and standard deviation ($\sigma$) of the runs. We use a one-tailed statistical hypothesis $t$-test to verify whether one set of values is better than the other. The better-performing results are highlighted in boldface.


\subsection{Ablation study}
\label{subsec:ablation_study} 
We now investigate the contribution of individual components of \ags{},  employing different sampling strategies and GNNs.
For GNNs, we consider GSAGE (\emph{spatial}) and  ChebNet~\cite{he2022convolutional} (\emph{spectral}). 
We sample $\floor{d_u\cdot k^\prime}$ ($k^\prime \in [0,1]$) neighbors for each node $u$ using  \emph{similarity} and \emph{diversity}-based sparsification using precomputed weights (detailed in \cref{subsec:node_sampling}), and \emph{random} sparsification that selects a subgraph uniformly at random. Only a \emph{single} sampled subgraph is used throughout the training.
For experiments, we generate three synthetic versions of \texttt{Cora} with average degree $d=200$, keeping the original nodes and features the same but changing the connectivity to have strong ($0.05$), moderate ($0.25$), and weak ($0.50$) heterophily. The distribution of heterophily for each node is close to uniform. 
Table~\ref{tab:corasynsparse} shows that diversity-based selection performs best with strong heterophily, and with spectral GNN, it even achieves accuracy better than using the entire graph. In contrast, on moderate heterophily, the similarity-based selection performs the best (even better than the whole graph). For weak heterophily (homophily), similar and random sparse perform  alike. 
When we convert these into a sampling paradigm (node sampling or graph sampling), similar behavior can be seen (Table~\ref{tab:corasynsampler}) as we get the best performance from diversity-based selection for strong heterophily and similarity-based for moderate heterophily with a wide margin than random. For weak heterophily or homophily, the similarity-based  performs slightly better than the random sampler. 

\begin{table}[t]
\centering
\resizebox{1.0\linewidth}{!}
{
\def\arraystretch{1.0}
\begin{tabular}{l|cc|cc|cc|cc|cc|cc}
\toprule
Dataset             & \multicolumn{4}{c|}{CoraSyn0.05}   & \multicolumn{4}{c|}{CoraSyn0.25}   & \multicolumn{4}{c}{CoraSyn0.50}                  \\\midrule
GNN &
  \multicolumn{2}{c}{GSAGE} &
  \multicolumn{2}{c|}{ChebNet} &
  \multicolumn{2}{c}{GSAGE} &
  \multicolumn{2}{c|}{ChebNet} &
  \multicolumn{2}{c}{GSAGE} &
  \multicolumn{2}{c}{ChebNet} \\\midrule
Sparse & $\mu$   & $\sigma$ & $\mu$   & $\sigma$ & $\mu$   & $\sigma$ & $\mu$   & $\sigma$ & $\mu$   & $\sigma$ & $\mu$           & $\sigma$        \\\midrule
Random & 57.08 & 0.46 & 56.44 & 0.14 & 61.85 & 0.94 & 62.14 & 1.29 & 92.09 & 0.79 & 92.48 & 0.72 \\
Similar & 48.50 & 0.29 & 50.91 & 0.22 & \textbf{71.66} & 0.42 & \textbf{71.02} & 0.17 & 92.48 & 0.51 & 91.71 & 0.25 \\
Diverse & 61.79 & 0.17 & \textbf{64.67} & 0.22 & 53.91 & 0.22 & 57.03 & 0.30 & 79.48 & 0.36 & 76.07 & 0.36 \\
Whole & \textbf{63.67} & 0.43 & 62.79 & 0.25 & 69.49 & 0.38 & 67.61 & 0.22 & \textbf{96.94} & 0.08 & \textbf{96.47} & 0.14\\\bottomrule
\end{tabular}
}
\caption{$F_1$ score of a single sparse subgraph  keeping $\floor{0.25\cdot d_u}$ neighbors of node $u$. Three synthetic versions of Cora are produced ($d = 200$ and $\gH_n$ $0.05, 0.25$, and $0.50$, respectively). 
}
\label{tab:corasynsparse}
\end{table}
\begin{table}[t]
\centering
\resizebox{1.0\linewidth}{!}
{
\def\arraystretch{1.0}
\begin{tabular}{l|cc|cc|cc|cc|cc|cc}
\toprule
Dataset & \multicolumn{4}{c|}{CoraSyn0.05} & \multicolumn{4}{c|}{CoraSyn0.25} & \multicolumn{4}{c}{CoraSyn0.50} \\\midrule
GNN &
  \multicolumn{2}{c}{GSAGE} &
  \multicolumn{2}{c|}{ChebNet} &
  \multicolumn{2}{c}{GSAGE} &
  \multicolumn{2}{c|}{ChebNet} &
  \multicolumn{2}{c}{GSAGE} &
  \multicolumn{2}{c}{ChebNet} \\\midrule
Sampler &
  $\mu$ &
  $\sigma$ &
  $\mu$ &
  $\sigma$ &
  $\mu$ &
  $\sigma$ &
  $\mu$ &
  $\sigma$ &
  \multicolumn{1}{c}{$\mu$} &
  \multicolumn{1}{c}{$\sigma$} &
  \multicolumn{1}{c}{$\mu$} &
  \multicolumn{1}{c}{$\sigma$} \\\midrule
Random & 59.20 & 0.22 & 57.50 & 0.80 & 62.85 & 0.55 & 61.67 & 0.46 & 91.06 & 0.36 & 88.89 & 1.32 \\
Similar & 54.44 & 0.68 & 54.03 & 0.51 & \textbf{69.19} & 0.08 & \textbf{67.43} & 0.30 & \textbf{92.00} & 0.08 & \textbf{91.71} & 0.29 \\
Diverse & \textbf{60.96} & 0.55 & \textbf{61.73} & 0.43 & 56.67 & 0.46 & 58.38 & 0.25 & 86.65 & 0.44 & 83.54 & 0.79
\\\bottomrule
\end{tabular}
}
\caption{$F_1$ score using different subgraph samples ($k=[25, 25]$). Three synthetic versions of Cora are produced ($d = 200$ and $\gH_n$ $0.05, 0.25$, and $0.50$, respectively).
}
\label{tab:corasynsampler}
\end{table}

Since real-world graphs have nonuniform node homophily, we generated a synthetic version of \texttt{Cora} where individual nodes have different local node homophily in the range $[0.05,0.50]$. Table~\ref{tab:coratwochannel} shows that our proposed two channel (one for homophily and one for heterophily) \ags{} performs the best, significantly outperforming ACM-GCN and LINKX. 
Detailed studies on the parameters and different submodular functions used in diversity-based sampling are presented in Appendix~\ref{subsec:submdoularablation}.



\begin{table}[t]
\centering
\resizebox{1.0\linewidth}{!}
{
\def\arraystretch{1.0}
\begin{tabular}{cc|cc|cc|cc|cc}
\toprule
\multicolumn{2}{c|}{ACM-GCN} & \multicolumn{2}{c|}{LINKX} & \multicolumn{2}{c|}{AGS (Sim. + Sim.)} & \multicolumn{2}{c|}{AGS (Div. + Div.)} & \multicolumn{2}{c}{AGS (Sim. + Div.)} \\
$\mu$ & $\sigma$ & $\mu$ & $\sigma$ & $\mu$ & $\sigma$ & $\mu$ & $\sigma$ & $\mu$ & $\sigma$ \\\midrule
35.3394 & 0.6922 & 34.2081 & 0.7889 & 60.8145 & 0.7759 & 62.8394 & 0.8522 & \textbf{63.0317} & 0.5079\\\bottomrule
\end{tabular}
}
\caption{$F_1$ Score of two-channel GNNs with a differently weighted sampler in synthetic Cora graph with locally heterophilic and homophilic nodes mixed ($\gH_n(u) =[0.05,0.50])$.}
\label{tab:coratwochannel}
\end{table}

\begin{figure}[!htbp]
 \centering
 \includegraphics[width=0.7\linewidth, angle = 0 ]{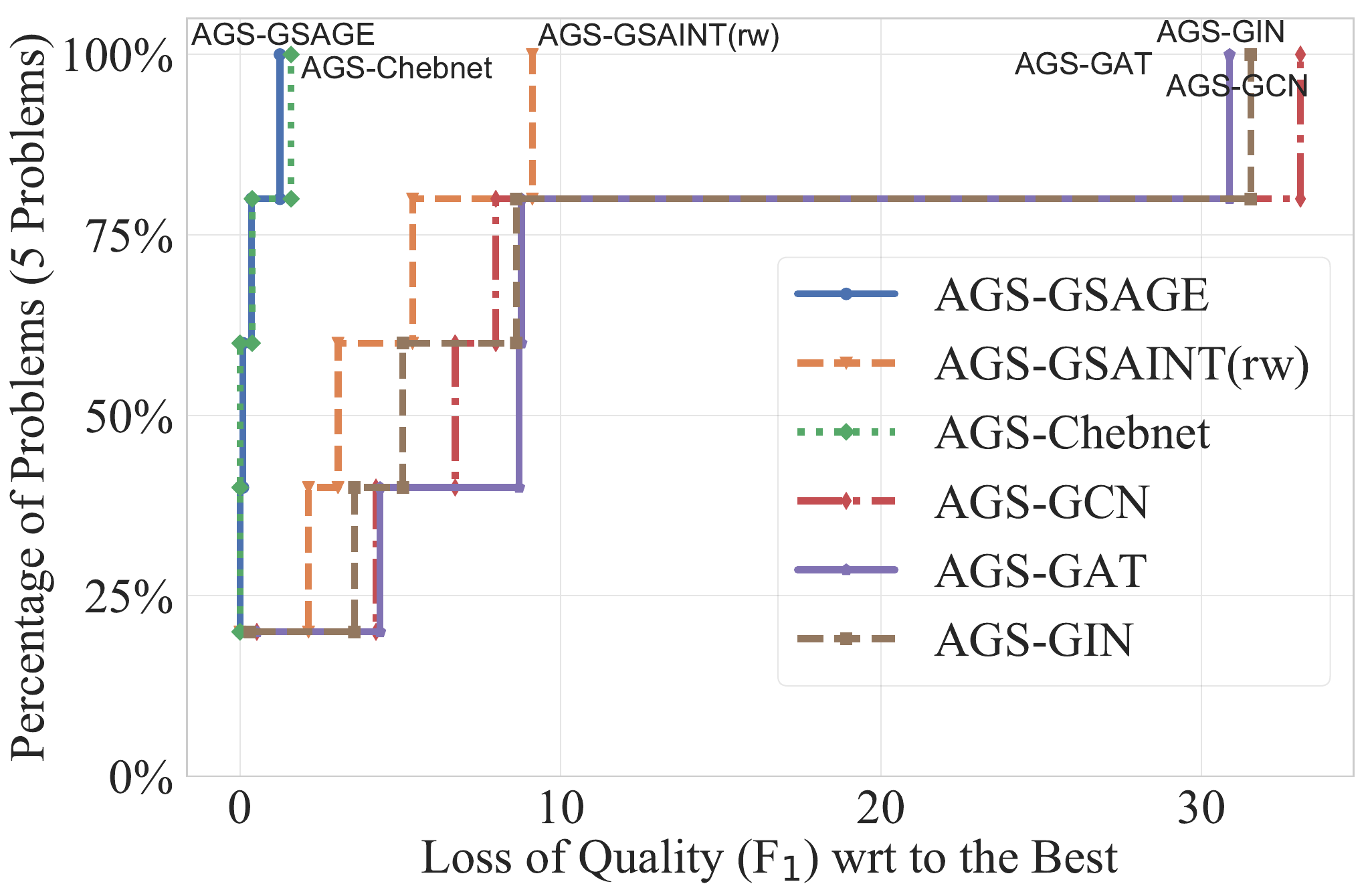}
 \caption{{\em Performance Profile:} X-axis is the difference in $F_1$-scores (scaled to 100) for different GNNs coupled with AGS sampler on five benchmark heterophilic graphs.}
 \label{fig:agsgnnvariants}
\end{figure}

We coupled our sampling strategy (both node sampling and graph sampling (\S\ref{sec:graphsampling})) to existing GNNs (GSAGE, ChebNet, GSAINT, GIN, GAT, and GCN) and evaluated them on five heterophilic graphs (\texttt{\justify Reed98, Roman-empire, Actor, Minesweeper, Tolokers}). While detailed numerical values are provided in Table~\ref{tab:agsgnnvariants} in Appendix~\ref{subsec:ablation}, we summarize the key results as a performance profile in Fig.~\ref{fig:agsgnnvariants}.
We observe that AGS with GSAGE, ChebNet, and GSAINT performed the best. 





\subsection{Experimental Runtime and Convergence}
\label{subsec:runtimeandconvergence}
Table~\ref{tab:training_time} shows per epoch training time of different methods under the same settings. 
For large datasets such as \texttt{Reddit}, with single worker thread, our current implementation of weighted sampling requires around $3$ times more than the random sampling used in GSAGE due to the dual channels and a few implementation differences with PyTorchGeometric. We plan to improve our implementation next. Since our precomputation is embarrassingly parallel, we can accelerate our algorithm by increasing the number of worker threads. Our weighted random walk, is faster than the library implementation even for sequential execution.
Fig.~\ref{fig:convergence_iteration} shows the number of epochs required to converge for random sampling (GSAGE) and our weighted sampling (\ags{}). Using the same settings, we see that \ags{} is more stable and requires fewer epochs on average to converge than GSAGE.

\begin{table}[t]
\scriptsize
\centering
\resizebox{1.0\linewidth}{!}
{
\def\arraystretch{1.0}
\begin{tabular}{l|l|ccc}
\toprule
Method & \multicolumn{1}{c|}{Settings}  & Reddit & Genius & Yelp \\
\midrule
GSAGE  & $k$ = [25, 10], $b$ = 1024, $\ell$ = 2, H=256       & 5.47  & 2.16 & 7.71  \\
GSAINT & $b$ = 4096, rw, step = 2, $\ell$ = 2, H = 256 & 8.01  & 2.10 & 4.34  \\
LINKX$+$      & random 100 parts, b = 2, H = 32        & 10.63 & 2.35 & 3.54  \\
AGS-NS     & $k$ = [25, 10], b = 1024, $\ell$ = 2, H = 256       & 18.44 & 2.43 & 13.30 \\
AGS-GS     & $b$ = 4096, wrw, step = 2, $\ell$ = 2, H = 256 & 5.90  & 1.72 & 4.06 \\
\bottomrule
\end{tabular}
}
\caption{Average training time (seconds/epoch). \texttt{rw} and \texttt{wrw} refer to random walk and weighted random walk, respectively.}
\label{tab:training_time}
\end{table}

\begin{table}[t]
\scriptsize
\centering
\resizebox{1.0\linewidth}{!}{
\def\arraystretch{1.0}
\begin{tabular}{c|c|cccc}
\toprule
Method & Settings & Reddit & Genius & Yelp\\\midrule
GraphSAINT        & Norm Computation             & 414.93 & 92.30 & 215.99 \\
Similarity Ranking  & $k_1=$ 20\%, $k_2=$ 20\%            & 60.03  & 6.00  & 41.20 \\
Diversity Ranking       & $k_1=$ 20\%, $k_2=$ 20\%    & 600.00       & 16.37 & 55.62  \\
LearningSimilarity           & H = 256, b = 10000 &       7.01 &  1.00     &   4.03     \\
\bottomrule  
\end{tabular}
}
\caption{Pre-computation time (seconds) of different components of AGS-GNN and GSAINT with single worker thread.}
\label{tab:experiment_precomputation_time}
\end{table}

\begin{figure}[t]
 \centering
 \includegraphics[width=0.8\linewidth, angle = 0 ]{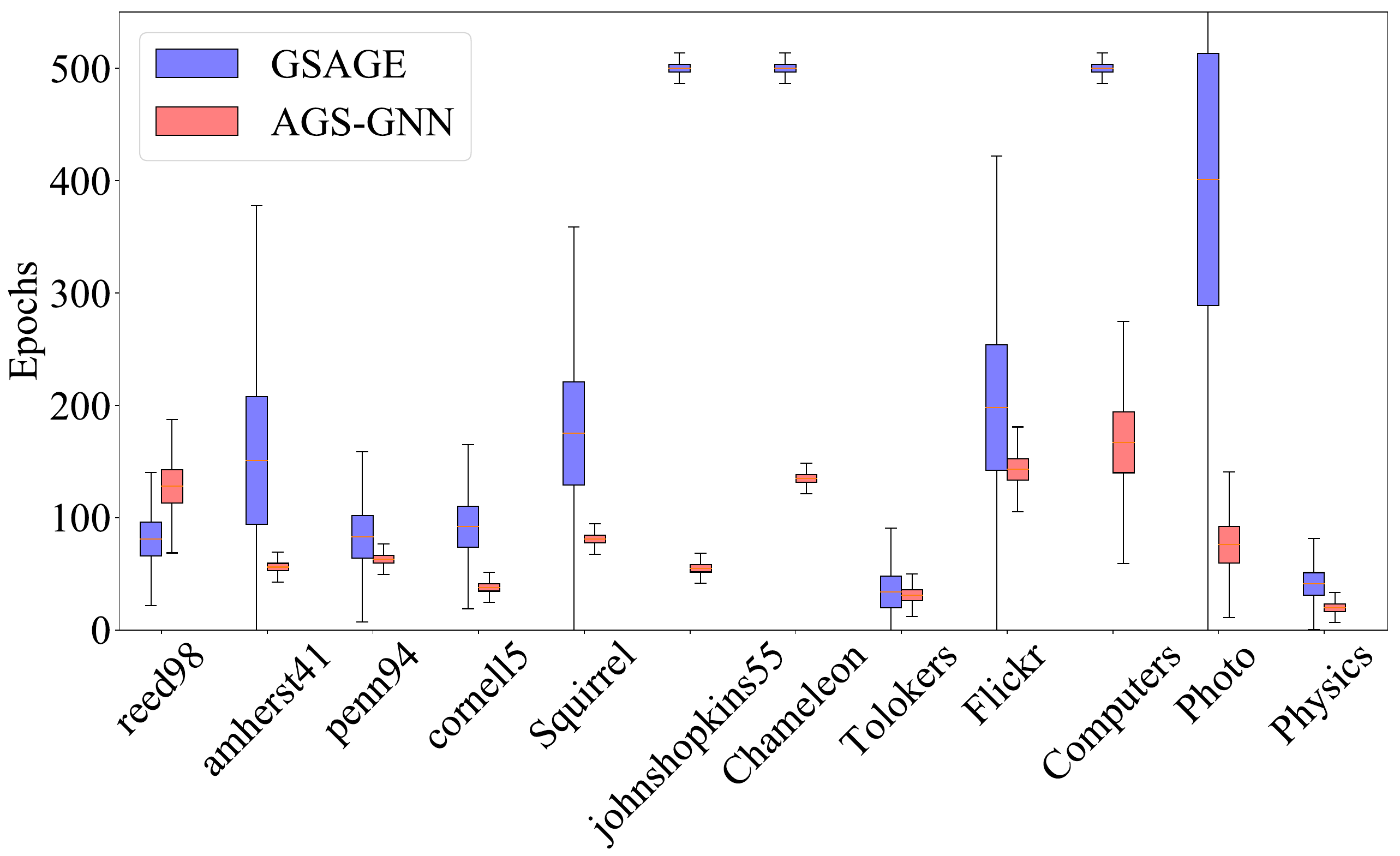}
 \caption{Num. of epochs for \ags{} and GSAGE to converge.}
 \label{fig:convergence_iteration}
\end{figure}

\section{Conclusions}
\label{sec:conclusion}
In this work, we proposed attribute-guided sampling that uses node features in an unsupervised and supervised fashion. We have shown that through a biased sampling of similar and diverse neighbors, we get improved performance in homophilic graphs and can handle challenging heterophilic graphs. We verify our claims through exhaustive experimentation on various benchmark datasets and methods. 
A limitation of our work is the time required for submodular optimization when the facility location function is used; the computation complexity is higher for dense graphs. 
We will optimize implementations in our future work. We will also build an end-to-end process for supervised sampling.

\bibliographystyle{ACM-Reference-Format}
\bibliography{gnn_bibliography}

\section{Appendix: Preliminaries }
\label{sec:prooftheorem}

\subsection{Homophily of similarity-based selection}
\label{subsec:prob_similar}
Let $t$ be an ego node of a graph with degree $d_t$, local node homophily $\gH_n(t)$, and label $y_t$. From the definition of homophily, the probability of selecting a neighbor $i\in \gN(t)$ with the same label as the ego node  uniformly at random is $P_\gU(y_i=y_t)=\gH_n(t)$,  where 
$\gU$ denotes to the distribution obtained by selecting a neighbor uniformly at random.
If features and labels correlate positively, and similarity is computed from the features, we can expect the following assumption to be valid: 
\begin{assumption}
\label{assump:homophily}
    For a node $t$,  the average similarity of neighbors with the same label as $t$  is greater than or equal to the average similarity of all neighbors. 
\end{assumption}

\begin{lemma}
If the probability of selecting a neighbor is proportional to its similarity to the ego node $t$, then the local node homophily of a sampled neighborhood $\gH_n'(t)\ge \gH_n(t)$. If the  sampling probability distribution is $\gS$ then $P_\gS(y_i=y_t)\ge P_\gU(y_i=y_t)$.
\end{lemma}

\begin{proof}
Consider an ego node $t=\{\vx_t,y_t\}$ where $\vx_t$, and $y_t$ are its feature and label, respectively. Let the feature and labels of the neighboring nodes of $t$ be $\gN(t)=\{(\vx_1,y_1),(
\vx_2,y_2),\cdots, (\vx_{d_t},y_{d_t})\}$.

Let  $s(\vx_i,\vx_t)$ be a  similarity function that measures how similar the feature of a neighbor $x_i$  is to the feature of the ego node  $\vx_t$. 
This function returns a positive value, and higher values indicate higher similarity. 

A probability distribution with probability mass function $p_\gS(i)$ assigns a probability to each neighbor $i$ based on its similarity to the ego node. The distribution should satisfy the following properties:
\begin{itemize}
    \item $p_\gS(i)\ge 0$ for all $i$;
    \item $\sum_{i=1}^{d_t} p_\gS(i)=1$;
    \item $p_\gS(i)$ is proportional to $s(\vx_i,\vx_t)$. 
    (An example is $p_\gS(i) = \frac{s(\vx_i,\vx_t)}{\sum_{j=1}^{d_t} s(\vx_j,\vx_t)}$.)
\end{itemize}


Let $I(y_i=y_t)$ be an indicator function that returns $1$ if $y_i=y_t$ and $0$ otherwise.
If a neighbor $i$ of a node $t$  is selected randomly then $p_\gU(i)=\frac{1}{d_t}$. Therefore the probability of selecting a neighbor randomly having the same label as the ego node is
\begin{align*}
    P_\gU(y_i=y_t) & = \sum_{i=1}^{d_t} p_\gU(i)\cdot I(y_i=y_t),  \\
                   & = \frac{1}{{d_t}}\cdot \sum_{i=1}^{d_t} I(y_i=y_t).
\end{align*}

If  the selection probability is based on similarity, then
\begin{align*}
    P_\gS(y_i=y_t) & = \sum_{i=1}^{d_t} p_\gS(i)\cdot I(y_i=y_t)\\
                   & = \sum_{i=1}^{d_t} \frac{s(\vx_i,\vx_t)}{\sum_{j=1}^{d_t} s(\vx_j,\vx_t)} \cdot I(y_i=y_t).    
\end{align*}

\begin{itemize}
    
\item  Let $n$ be the number of neighbors having the same label as ego node $t$, i.e., $n = \sum_{i=1}^{d_t} I(y_i=y_t)$.

\item  Let $s_t$ be the sum of similarities of all neighbors having the same label as $t$, i.e., $s_t=\sum_{i=1}^{d_t} (s(\vx_i,\vx_t)\cdot I(y_i=y_t)$. Also  let $s_d$ denote  the sum of  similarities of all neighbors of $t$, then $s_d=\sum_{j=1}^{d_t} s(\vx_j,\vx_t)$.

\item  If the features and labels are positively correlated, 
the average similarity of neighbors having the same label as the ego node should be higher than the average similarity of all neighbors of the ego node. That is, $\frac{s_t}{n} \ge \frac{s_d}{d_t}$.
\end{itemize}

We can now use these expressions and the substitutions shown below  to derive the result. From Assumption~\ref{assump:homophily}, 
\begin{align*}   
    \frac{s_t}{n} &\ge \frac{s_d}{d_t}\\ 
    \frac{s_t}{s_d} &\ge \frac{n}{d_t}\\    
    \frac{1}{\sum_{j=1}^{d_t} s(x_j,x_t)}\sum_{i=1}^{d_t} s(\vx_i,\vx_t)\cdot I(y_i=y_t) &\ge \frac{1}{d_t}\cdot \sum_{i=1}^{d_t} I(y_i=y_t)\\
    \sum_{i=1}^{d_t} \frac{s(\vx_i,\vx_t)}{\sum_{j=1}^{d_t} s(\vx_j,\vx_t)}\cdot I(y_i=y_t) &\ge  \sum_{i=1}^{d_t} \frac{1}{{d_t}}\cdot I(y_i=y_t) \\
    \sum_{i=1}^{d_t} p_\gS(i)\cdot I(y_i=y_t)& \ge \sum_{i=1}^{d_t} p_\gU(i)\cdot I(y_i=y_t)\\
    P_\gS(y_i=y_t)&\ge P_\gU(y_i=y_t)\\
    \gH_n'(t)&\ge \gH_n(t).
\end{align*}

\end{proof}

\subsection{Homophily of diversity-based selection}
\label{subsec:prob_diverse}

A submodular function is a set valued function that assigns a value to every subset of a ground set that satisfies the diminishing returns property, i.e., for two subsets $S$, $T$ with $S\subseteq T$ of a ground set $A$, and any  element $e \in A\setminus T$, 
$$f(S\cup{e}) - f(S) \geq f(T\cup{e}) - f(T).$$
The quantity on either side of the inequality is the \emph{marginal gain}. 
 An example of a submodular function is  the facility location function
\begin{equation}
    f(S, A) = \sum\limits_{y\in A} \max_{x\in S} \phi(x, y), 
    \label{eq:facility_location_function}
\end{equation}
where $S$ and $A$ are as above  and $\phi$ is a  similarity measure on the elements of $A$. 
The facility location function can be maximized by 
a Greedy algorithm, which starts with the empty subset and iteratively adds an element that gives the largest marginal gain to the function value until a cardinality  (or other) constraint on $S$ is reached. 

In this context, $S$ is the current set of selected nodes initialized with ego node $t$, and ground set $A = \gN(t)\cup \{t\}$. 
The marginal gain is $f_i = f(S\cup \{i\},A)-f(S,A)$ for each neighbor $i\in \gN(t)\setminus S$.  The neighbors are iteratively added based on marginal gain. 
If node features and labels are positively correlated, then we can expect the following assumption to be valid:
\begin{assumption}
\label{assump:hetero}
    The average marginal gain of the neighbors of a node $t$ with the same label as $t$  is less than or equal to the average marginal gain of all neighbors.
\end{assumption}

\begin{lemma}
If the probability of selecting a neighbor of the ego node $t$  is proportional to the marginal gain of the neighbor, then the local node homophily of a sampled neighborhood $\gH_n'(t)\le \gH_n(t)$. If the sampling probability distribution is $\gD$ then $P_\gD(y_i=y_t)\le P_\gU(y_i=y_t)$.
\end{lemma}

\begin{proof}

Let the  probability mass function of the distribution $\gD$ be proportional to the marginal gain value, $p_\gD(i)= \frac{f_i}{\sum_{j=1}^{{d_t}}f_j}$. The higher the marginal gain,  the higher is the selection probability. 


If a neighbor is selected randomly, then $p_\gU(i)=\frac{1}{d_t}$. Therefore, the probability of selecting a neighbor randomly having the same label as the ego node is
\begin{align*}
    P_\gU(y_i=y_t) &= \sum_{i=1}^{d_t} p_\gU(i)\cdot I(y_i=y_t) \\
                   &= \frac{1}{{d_t}}\cdot I(y_i=y_t).
\end{align*}

If the  selection probability is based on diversity, then
\begin{align*}
    P_\gD(y_i=y_t) & = \sum_{i=1}^{d_t} p_\gD(i)\cdot I(y_i=y_t) \\
                   & = \sum_{i=1}^{d_t} \frac{f_i}{\sum_{j=1}^{d_t} f_j} \cdot I(y_i=y_t). 
\end{align*}

\begin{itemize}
    
\item  As in the earlier subsection, let $n$ be the number of neighbors having the same label as the ego node $t$, i.e., $n = \sum_{i=1}^d I(y_i=y_t)$.

\item  Let $m_t$ be the sum of marginal gains of all neighbors having the same label as $t$, i.e., $m_t=\sum_{i=1}^{d_t} (f_i)\cdot I(y_i=y_t)$, and let $m_d$ denote the sum of all marginal gains of the neighbors $m_d=\sum_{j=1}^{d_t} f_j$.

\item  If the features and labels are positively correlated, 
the average marginal gains of neighbors having the same label as the ego node should be  less than or equal to the average marginal gains of all neighbors of the ego node. 
Hence  $\frac{m_t}{n} \le \frac{m_d}{{d_t}}$. 


\end{itemize}

We use the above expressions and the substitutions shown below get the desired result. From Assumption~\ref{assump:hetero},
\begin{align*}
    \frac{m_t}{n} &\le \frac{m_d}{{d_t}}\\
    \frac{m_t}{m_d} & \le \frac{n}{{d_t}}\\
    \frac{1}{\sum_{j=1}^d f_j}\sum_{i=1}^{d_t} f_i\cdot I(y_i=y_t) & \le \frac{1}{{d_t}}\cdot \sum_{i=1}^d I(y_i=y_t)\\
    \sum_{i=1}^{d_t} \frac{f_i}{\sum_{j=1}^{d_t} f_j}\cdot I(y_i=y_t) & \le \sum_{i=1}^{d_t} \frac{1}{{d_t}}\cdot I(y_i=y_t)\\     
    \sum_{i=1}^{d_t} p_\gD(i)\cdot I(y_i=y_t) & \le  \sum_{i=1}^{d_t} p_\gU(i)\cdot I(y_i=y_t)\\ 
    P_\gD(y_i=y_t) & \le P_\gU(y_i=y_t) \\
    \gH_n'(t) & \le \gH_n(t). 
\end{align*}

\end{proof}

\subsection{Empirical verification of the Lemmas}
\label{subsec:emphlemma}
Fig.~\ref{fig:squrrelhomophily} empirically verifies our assumptions that a subgraph selected from a similar neighborhood  increases the homophily,  and selecting a  diverse neighborhood using a  submodular function decreases homophily,  relative to the values in the  original graph.
For each vertex $u$, we take $\floor{d_u \cdot k^\prime}$ ($k^\prime\in [0,1]$) neighbors based on the ranking from similarity and diversity and compare against random selection.

\begin{figure}[!htbp]
\centering
\includegraphics[width=0.8\linewidth]{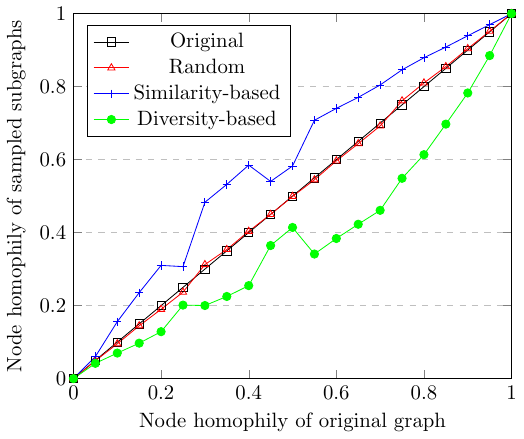}%
\caption{Empirical homophily comparison of different sampled subgraphs keeping $\floor{0.25 \cdot d_u}$  neighbors of node $u$. The synthetic graphs are from \texttt{squirrel} with $d=42$. The figure shows that similar neighbors increase $\gH_n$, diverse samples obtained from a submodular function  decrease $\gH_n$,  and random  selection keeps  $\gH_n$ the same as in the original  graph.}
\label{fig:squrrelhomophily}%
\end{figure}

\subsection{High-Pass and Low-Pass Filters for GNN}
\label{sec:gnnfilter}
We consider the following heterophilic graph in Fig.~\ref{fig:heterophilic_graph} to demonstrate how \emph{high-pass} filters can help in node classification in heterophilic graphs.
Assume a heterophilic graph $\gG$ has $18$ nodes of five different classes where nodes in the range $[0,14]$ are locally heterophilic and $[15,17]$ are locally homophilic. 
We assume the graph's node feature $\mX$ be the one-hot encoding of labels, thus the features and labels are (perfectly) correlated.
This is to show how \emph{low-pass} and \emph{high-pass} filters perform in an ideal scenario with heterophily.

\begin{figure}[!htbp]
\centering
\includegraphics[width=0.9\linewidth]{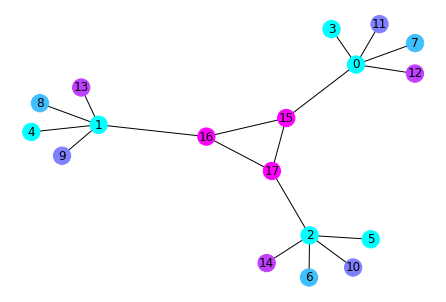}%
\caption{A heterophilic graph to demonstrate the performance of high-pass and low-pass filter. Here, nodes $0-14$ are locally heterophilic, and nodes $15-17$ are locally homophilic. In the figure, vertices with the same color correspond to the same class label.}
\label{fig:heterophilic_graph}
\end{figure}

\begin{table}[!htbp]

\[ \mX =
\begin{bmatrix}
1 &0 &0 &0 &0 \\
1 &0 &0 &0 &0 \\
1 &0 &0 &0 &0 \\
1 &0 &0 &0 &0 \\
1 &0 &0 &0 &0 \\
1 &0 &0 &0 &0 \\
0 &1 &0 &0 &0 \\
0 &1 &0 &0 &0 \\
0 &1 &0 &0 &0 \\
0 &0 &1 &0 &0 \\
0 &0 &1 &0 &0 \\
0 &0 &1 &0 &0 \\
0 &0 &0 &1 &0 \\
0 &0 &0 &1 &0 \\
0 &0 &0 &1 &0 \\
0 &0 &0 &0 &1 \\
0 &0 &0 &0 &1 \\
0 &0 &0 &0 &1 \\
\end{bmatrix}
\]
\end{table}

In the graph convolution step, the feature information of the neighbors gets aggregated with the ego node. So, we get different transformed representations of features depending on the representation of the underlying graph. 
For this example, we use a symmetric normalized adjacency matrix, $\mA_{sym}$ as a low-pass filter and a symmetric normalized graph Laplacian, $\mL_{sym}$ as a high-pass filter~\citep{ekambaram2014graph}. These filters are defined as,

\begin{align*}
    \mA_{sym} &=\mD^{-\frac{1}{2}}\mA\mD^{-\frac{1}{2}}\\
    \mL_{sym} &= \mI-\mD^{-\frac{1}{2}}\mA\mD^{-\frac{1}{2}}    
\end{align*}

If self-loops are added, then, 

\begin{align*}
    \Bar{\mA} &= \mA+\mI\\
    \Bar{\mA}_{sym} &=\Bar{\mD}^{-\frac{1}{2}}\Bar{\mA}\Bar{\mD}^{-\frac{1}{2}}\\
    \Bar{\mL}_{sym} &= \mI-\Bar{\mD}^{-\frac{1}{2}}\Bar{\mA}\Bar{\mD}^{-\frac{1}{2}}\\   
\end{align*}

Graphs under heterophily may perform well if the node features of the neighbors make two heterophilic nodes similar after aggregation. To evaluate this, we first compute aggregated feature information based on the filters and then inspect pairwise similarities.

\begin{align*}
    S(\mA,\mX) & = \mA\mX (\mA\mX )^T\\
    S(\Bar{\mA}_{sym},\mX) & = \Bar{\mA}_{sym}\mX (\Bar{\mA}_{sym}\mX )^T\\
    S(\Bar{\mL}_{sym},\mX) & = \Bar{\mL}_{sym}\mX (\Bar{\mL}_{sym}\mX )^T\\
\end{align*}

In these matrices, for a particular node $u$, if the average similarity of nodes with the same label as $y_u$ is greater than the average similarity of the nodes not having the same label as $y_u$, then we can expect that node to be appropriately classified (Definition 2 and Theorem 1 in~\citep{luan2022revisiting}).

\begin{table}[!htbp]
\centering
\resizebox{1.0\linewidth}{!}
{
\def\arraystretch{1.0}
\begin{tabular}{|l|llllllllllllllllll|}
\hline
\textbf{V/V} & \multicolumn{1}{l|}{\textbf{0}} & \multicolumn{1}{l|}{\textbf{1}} & \multicolumn{1}{l|}{\textbf{2}} & \multicolumn{1}{l|}{\textbf{3}} & \multicolumn{1}{l|}{\textbf{4}} & \multicolumn{1}{l|}{\textbf{5}} & \multicolumn{1}{l|}{\textbf{6}} & \multicolumn{1}{l|}{\textbf{7}} & \multicolumn{1}{l|}{\textbf{8}} & \multicolumn{1}{l|}{\textbf{9}} & \multicolumn{1}{l|}{\textbf{10}} & \multicolumn{1}{l|}{\textbf{11}} & \multicolumn{1}{l|}{\textbf{12}} & \multicolumn{1}{l|}{\textbf{13}} & \multicolumn{1}{l|}{\textbf{14}} & \multicolumn{1}{l|}{\textbf{15}} & \multicolumn{1}{l|}{\textbf{16}} & \multicolumn{1}{l|}{\textbf{17}} \\ \hline
\textbf{0} & \cellcolor[HTML]{A6C9EC}0.59 & \cellcolor[HTML]{A6C9EC}0.59 & \cellcolor[HTML]{A6C9EC}0.59 & \cellcolor[HTML]{A6C9EC}0.12 & \cellcolor[HTML]{A6C9EC}0.12 & \cellcolor[HTML]{A6C9EC}0.12 & -0.3 & -0.3 & -0.3 & -0.3 & -0.3 & -0.3 & -0.3 & -0.3 & -0.3 & -0.16 & -0.16 & -0.16 \\ \cline{1-1}
\textbf{1} & \cellcolor[HTML]{A6C9EC}0.59 & \cellcolor[HTML]{A6C9EC}0.59 & \cellcolor[HTML]{A6C9EC}0.59 & \cellcolor[HTML]{A6C9EC}0.12 & \cellcolor[HTML]{A6C9EC}0.12 & \cellcolor[HTML]{A6C9EC}0.12 & -0.3 & -0.3 & -0.3 & -0.3 & -0.3 & -0.3 & -0.3 & -0.3 & -0.3 & -0.16 & -0.16 & -0.16 \\ \cline{1-1}
\textbf{2} & \cellcolor[HTML]{A6C9EC}0.59 & \cellcolor[HTML]{A6C9EC}0.59 & \cellcolor[HTML]{A6C9EC}0.59 & \cellcolor[HTML]{A6C9EC}0.12 & \cellcolor[HTML]{A6C9EC}0.12 & \cellcolor[HTML]{A6C9EC}0.12 & -0.3 & -0.3 & -0.3 & -0.3 & -0.3 & -0.3 & -0.3 & -0.3 & -0.3 & -0.16 & -0.16 & -0.16 \\ \cline{1-1}
\textbf{3} & \cellcolor[HTML]{A6C9EC}0.12 & \cellcolor[HTML]{A6C9EC}0.12 & \cellcolor[HTML]{A6C9EC}0.12 & \cellcolor[HTML]{A6C9EC}0.04 & \cellcolor[HTML]{A6C9EC}0.04 & \cellcolor[HTML]{A6C9EC}0.04 & -0.06 & -0.06 & -0.06 & -0.06 & -0.06 & -0.06 & -0.06 & -0.06 & -0.06 & -0.04 & -0.04 & -0.04 \\ \cline{1-1}
\textbf{4} & \cellcolor[HTML]{A6C9EC}0.12 & \cellcolor[HTML]{A6C9EC}0.12 & \cellcolor[HTML]{A6C9EC}0.12 & \cellcolor[HTML]{A6C9EC}0.04 & \cellcolor[HTML]{A6C9EC}0.04 & \cellcolor[HTML]{A6C9EC}0.04 & -0.06 & -0.06 & -0.06 & -0.06 & -0.06 & -0.06 & -0.06 & -0.06 & -0.06 & -0.04 & -0.04 & -0.04 \\ \cline{1-1}
\textbf{5} & \cellcolor[HTML]{A6C9EC}0.12 & \cellcolor[HTML]{A6C9EC}0.12 & \cellcolor[HTML]{A6C9EC}0.12 & \cellcolor[HTML]{A6C9EC}0.04 & \cellcolor[HTML]{A6C9EC}0.04 & \cellcolor[HTML]{A6C9EC}0.04 & -0.06 & -0.06 & -0.06 & -0.06 & -0.06 & -0.06 & -0.06 & -0.06 & -0.06 & -0.04 & -0.04 & -0.04 \\ \cline{1-1}
\textbf{6} & -0.3 & -0.3 & -0.3 & -0.06 & -0.06 & -0.06 & \cellcolor[HTML]{F1A983}0.33 & \cellcolor[HTML]{F1A983}0.33 & \cellcolor[HTML]{F1A983}0.33 & \cellcolor[HTML]{FBE2D5}0.08 & \cellcolor[HTML]{FBE2D5}0.08 & \cellcolor[HTML]{FBE2D5}0.08 & \cellcolor[HTML]{FBE2D5}0.08 & \cellcolor[HTML]{FBE2D5}0.08 & \cellcolor[HTML]{FBE2D5}0.08 & \cellcolor[HTML]{FBE2D5}0.06 & \cellcolor[HTML]{FBE2D5}0.06 & \cellcolor[HTML]{FBE2D5}0.06 \\ \cline{1-1}
\textbf{7} & -0.3 & -0.3 & -0.3 & -0.06 & -0.06 & -0.06 & \cellcolor[HTML]{F1A983}0.33 & \cellcolor[HTML]{F1A983}0.33 & \cellcolor[HTML]{F1A983}0.33 & \cellcolor[HTML]{FBE2D5}0.08 & \cellcolor[HTML]{FBE2D5}0.08 & \cellcolor[HTML]{FBE2D5}0.08 & \cellcolor[HTML]{FBE2D5}0.08 & \cellcolor[HTML]{FBE2D5}0.08 & \cellcolor[HTML]{FBE2D5}0.08 & \cellcolor[HTML]{FBE2D5}0.06 & \cellcolor[HTML]{FBE2D5}0.06 & \cellcolor[HTML]{FBE2D5}0.06 \\ \cline{1-1}
\textbf{8} & -0.3 & -0.3 & -0.3 & -0.06 & -0.06 & -0.06 & \cellcolor[HTML]{F1A983}0.33 & \cellcolor[HTML]{F1A983}0.33 & \cellcolor[HTML]{F1A983}0.33 & \cellcolor[HTML]{FBE2D5}0.08 & \cellcolor[HTML]{FBE2D5}0.08 & \cellcolor[HTML]{FBE2D5}0.08 & \cellcolor[HTML]{FBE2D5}0.08 & \cellcolor[HTML]{FBE2D5}0.08 & \cellcolor[HTML]{FBE2D5}0.08 & \cellcolor[HTML]{FBE2D5}0.06 & \cellcolor[HTML]{FBE2D5}0.06 & \cellcolor[HTML]{FBE2D5}0.06 \\ \cline{1-1}
\textbf{9} & -0.3 & -0.3 & -0.3 & -0.06 & -0.06 & -0.06 & \cellcolor[HTML]{C1F0C8}0.08 & \cellcolor[HTML]{C1F0C8}0.08 & \cellcolor[HTML]{C1F0C8}0.08 & \cellcolor[HTML]{47D359}0.33 & \cellcolor[HTML]{47D359}0.33 & \cellcolor[HTML]{47D359}0.33 & \cellcolor[HTML]{C1F0C8}0.08 & \cellcolor[HTML]{C1F0C8}0.08 & \cellcolor[HTML]{C1F0C8}0.08 & \cellcolor[HTML]{C1F0C8}0.06 & \cellcolor[HTML]{C1F0C8}0.06 & \cellcolor[HTML]{C1F0C8}0.06 \\ \cline{1-1}
\textbf{10} & -0.3 & -0.3 & -0.3 & -0.06 & -0.06 & -0.06 & \cellcolor[HTML]{C1F0C8}0.08 & \cellcolor[HTML]{C1F0C8}0.08 & \cellcolor[HTML]{C1F0C8}0.08 & \cellcolor[HTML]{47D359}0.33 & \cellcolor[HTML]{47D359}0.33 & \cellcolor[HTML]{47D359}0.33 & \cellcolor[HTML]{C1F0C8}0.08 & \cellcolor[HTML]{C1F0C8}0.08 & \cellcolor[HTML]{C1F0C8}0.08 & \cellcolor[HTML]{C1F0C8}0.06 & \cellcolor[HTML]{C1F0C8}0.06 & \cellcolor[HTML]{C1F0C8}0.06 \\ \cline{1-1}
\textbf{11} & -0.3 & -0.3 & -0.3 & -0.06 & -0.06 & -0.06 & \cellcolor[HTML]{C1F0C8}0.08 & \cellcolor[HTML]{C1F0C8}0.08 & \cellcolor[HTML]{C1F0C8}0.08 & \cellcolor[HTML]{47D359}0.33 & \cellcolor[HTML]{47D359}0.33 & \cellcolor[HTML]{47D359}0.33 & \cellcolor[HTML]{C1F0C8}0.08 & \cellcolor[HTML]{C1F0C8}0.08 & \cellcolor[HTML]{C1F0C8}0.08 & \cellcolor[HTML]{C1F0C8}0.06 & \cellcolor[HTML]{C1F0C8}0.06 & \cellcolor[HTML]{C1F0C8}0.06 \\ \cline{1-1}
\textbf{12} & -0.3 & -0.3 & -0.3 & -0.06 & -0.06 & -0.06 & \cellcolor[HTML]{F2F2F2}0.08 & \cellcolor[HTML]{F2F2F2}0.08 & \cellcolor[HTML]{F2F2F2}0.08 & \cellcolor[HTML]{F2F2F2}0.08 & \cellcolor[HTML]{F2F2F2}0.08 & \cellcolor[HTML]{F2F2F2}0.08 & \cellcolor[HTML]{BFBFBF}0.33 & \cellcolor[HTML]{BFBFBF}0.33 & \cellcolor[HTML]{BFBFBF}0.33 & \cellcolor[HTML]{F2F2F2}0.06 & \cellcolor[HTML]{F2F2F2}0.06 & \cellcolor[HTML]{F2F2F2}0.06 \\ \cline{1-1}
\textbf{13} & -0.3 & -0.3 & -0.3 & -0.06 & -0.06 & -0.06 & \cellcolor[HTML]{F2F2F2}0.08 & \cellcolor[HTML]{F2F2F2}0.08 & \cellcolor[HTML]{F2F2F2}0.08 & \cellcolor[HTML]{F2F2F2}0.08 & \cellcolor[HTML]{F2F2F2}0.08 & \cellcolor[HTML]{F2F2F2}0.08 & \cellcolor[HTML]{BFBFBF}0.33 & \cellcolor[HTML]{BFBFBF}0.33 & \cellcolor[HTML]{BFBFBF}0.33 & \cellcolor[HTML]{F2F2F2}0.06 & \cellcolor[HTML]{F2F2F2}0.06 & \cellcolor[HTML]{F2F2F2}0.06 \\ \cline{1-1}
\textbf{14} & -0.3 & -0.3 & -0.3 & -0.06 & -0.06 & -0.06 & \cellcolor[HTML]{F2F2F2}0.08 & \cellcolor[HTML]{F2F2F2}0.08 & \cellcolor[HTML]{F2F2F2}0.08 & \cellcolor[HTML]{F2F2F2}0.08 & \cellcolor[HTML]{F2F2F2}0.08 & \cellcolor[HTML]{F2F2F2}0.08 & \cellcolor[HTML]{BFBFBF}0.33 & \cellcolor[HTML]{BFBFBF}0.33 & \cellcolor[HTML]{BFBFBF}0.33 & \cellcolor[HTML]{F2F2F2}0.06 & \cellcolor[HTML]{F2F2F2}0.06 & \cellcolor[HTML]{F2F2F2}0.06 \\ \cline{1-1}
\textbf{15} & -0.16 & -0.16 & -0.16 & -0.04 & -0.04 & -0.04 & \cellcolor[HTML]{F2CEEF}0.06 & \cellcolor[HTML]{F2CEEF}0.06 & \cellcolor[HTML]{F2CEEF}0.06 & \cellcolor[HTML]{F2CEEF}0.06 & \cellcolor[HTML]{F2CEEF}0.06 & \cellcolor[HTML]{F2CEEF}0.06 & \cellcolor[HTML]{F2CEEF}0.06 & \cellcolor[HTML]{F2CEEF}0.06 & \cellcolor[HTML]{F2CEEF}0.06 & \cellcolor[HTML]{D86DCD}0.1 & \cellcolor[HTML]{D86DCD}0.1 & \cellcolor[HTML]{D86DCD}0.1 \\ \cline{1-1}
\textbf{16} & -0.16 & -0.16 & -0.16 & -0.04 & -0.04 & -0.04 & \cellcolor[HTML]{F2CEEF}0.06 & \cellcolor[HTML]{F2CEEF}0.06 & \cellcolor[HTML]{F2CEEF}0.06 & \cellcolor[HTML]{F2CEEF}0.06 & \cellcolor[HTML]{F2CEEF}0.06 & \cellcolor[HTML]{F2CEEF}0.06 & \cellcolor[HTML]{F2CEEF}0.06 & \cellcolor[HTML]{F2CEEF}0.06 & \cellcolor[HTML]{F2CEEF}0.06 & \cellcolor[HTML]{D86DCD}0.1 & \cellcolor[HTML]{D86DCD}0.1 & \cellcolor[HTML]{D86DCD}0.1 \\ \cline{1-1}
\textbf{17} & -0.16 & -0.16 & -0.16 & -0.04 & -0.04 & -0.04 & \cellcolor[HTML]{F2CEEF}0.06 & \cellcolor[HTML]{F2CEEF}0.06 & \cellcolor[HTML]{F2CEEF}0.06 & \cellcolor[HTML]{F2CEEF}0.06 & \cellcolor[HTML]{F2CEEF}0.06 & \cellcolor[HTML]{F2CEEF}0.06 & \cellcolor[HTML]{F2CEEF}0.06 & \cellcolor[HTML]{F2CEEF}0.06 & \cellcolor[HTML]{F2CEEF}0.06 & \cellcolor[HTML]{D86DCD}0.1 & \cellcolor[HTML]{D86DCD}0.1 & \cellcolor[HTML]{D86DCD}0.1 \\ \cline{1-1}
\hline
\end{tabular}
}
\caption{$\Bar{\mA} = \mA+\mI$, Graph Laplacian $\Bar{\mL}_{sym} = \mI-\Bar{\mD}^{-\frac{1}{2}}\Bar{\mA}\Bar{\mD}^{-\frac{1}{2}}$, and the pairwise similarity weight after aggregation is, $S(\Bar{\mL}_{sym},\mX) = \Bar{\mL}_{sym}\mX (\Bar{\mL}_{sym}\mX )^T$. The highlighted values show that this high-pass filter can properly classify heterophilic nodes. Here, nodes $0-2$ are mapped in the same space and have high similarity weight since they have the same dissimilar neighborhood.}
\label{tab:highpass}
\end{table}

\begin{table}[!htbp]
\centering
\resizebox{1.0\linewidth}{!}
{
\def\arraystretch{1.0}
\begin{tabular}{|l|llllllllllllllllll|}

\hline
\textbf{V/V} & \multicolumn{1}{l|}{\textbf{0}} & \multicolumn{1}{l|}{\textbf{1}} & \multicolumn{1}{l|}{\textbf{2}} & \multicolumn{1}{l|}{\textbf{3}} & \multicolumn{1}{l|}{\textbf{4}} & \multicolumn{1}{l|}{\textbf{5}} & \multicolumn{1}{l|}{\textbf{6}} & \multicolumn{1}{l|}{\textbf{7}} & \multicolumn{1}{l|}{\textbf{8}} & \multicolumn{1}{l|}{\textbf{9}} & \multicolumn{1}{l|}{\textbf{10}} & \multicolumn{1}{l|}{\textbf{11}} & \multicolumn{1}{l|}{\textbf{12}} & \multicolumn{1}{l|}{\textbf{13}} & \multicolumn{1}{l|}{\textbf{14}} & \multicolumn{1}{l|}{\textbf{15}} & \multicolumn{1}{l|}{\textbf{16}} & \multicolumn{1}{l|}{\textbf{17}} \\ \hline
\textbf{0} & \cellcolor[HTML]{A6C9EC}0.5 & \cellcolor[HTML]{A6C9EC}0.5 & \cellcolor[HTML]{A6C9EC}0.5 & \cellcolor[HTML]{DAE9F8}0.36 & \cellcolor[HTML]{DAE9F8}0.36 & \cellcolor[HTML]{DAE9F8}0.36 & \cellcolor[HTML]{DAE9F8}0.28 & \cellcolor[HTML]{DAE9F8}0.28 & \cellcolor[HTML]{DAE9F8}0.28 & \cellcolor[HTML]{DAE9F8}0.28 & \cellcolor[HTML]{DAE9F8}0.28 & \cellcolor[HTML]{DAE9F8}0.28 & \cellcolor[HTML]{DAE9F8}0.28 & \cellcolor[HTML]{DAE9F8}0.28 & \cellcolor[HTML]{DAE9F8}0.28 & \cellcolor[HTML]{DAE9F8}0.25 & \cellcolor[HTML]{DAE9F8}0.25 & \cellcolor[HTML]{DAE9F8}0.25 \\ \cline{1-1}
\textbf{1} & \cellcolor[HTML]{A6C9EC}0.5 & \cellcolor[HTML]{A6C9EC}0.5 & \cellcolor[HTML]{A6C9EC}0.5 & \cellcolor[HTML]{DAE9F8}0.36 & \cellcolor[HTML]{DAE9F8}0.36 & \cellcolor[HTML]{DAE9F8}0.36 & \cellcolor[HTML]{DAE9F8}0.28 & \cellcolor[HTML]{DAE9F8}0.28 & \cellcolor[HTML]{DAE9F8}0.28 & \cellcolor[HTML]{DAE9F8}0.28 & \cellcolor[HTML]{DAE9F8}0.28 & \cellcolor[HTML]{DAE9F8}0.28 & \cellcolor[HTML]{DAE9F8}0.28 & \cellcolor[HTML]{DAE9F8}0.28 & \cellcolor[HTML]{DAE9F8}0.28 & \cellcolor[HTML]{DAE9F8}0.25 & \cellcolor[HTML]{DAE9F8}0.25 & \cellcolor[HTML]{DAE9F8}0.25 \\ \cline{1-1}
\textbf{2} & \cellcolor[HTML]{A6C9EC}0.5 & \cellcolor[HTML]{A6C9EC}0.5 & \cellcolor[HTML]{A6C9EC}0.5 & \cellcolor[HTML]{DAE9F8}0.36 & \cellcolor[HTML]{DAE9F8}0.36 & \cellcolor[HTML]{DAE9F8}0.36 & \cellcolor[HTML]{DAE9F8}0.28 & \cellcolor[HTML]{DAE9F8}0.28 & \cellcolor[HTML]{DAE9F8}0.28 & \cellcolor[HTML]{DAE9F8}0.28 & \cellcolor[HTML]{DAE9F8}0.28 & \cellcolor[HTML]{DAE9F8}0.28 & \cellcolor[HTML]{DAE9F8}0.28 & \cellcolor[HTML]{DAE9F8}0.28 & \cellcolor[HTML]{DAE9F8}0.28 & \cellcolor[HTML]{DAE9F8}0.25 & \cellcolor[HTML]{DAE9F8}0.25 & \cellcolor[HTML]{DAE9F8}0.25 \\ \cline{1-1}
\textbf{3} & \cellcolor[HTML]{DAE9F8}0.36 & \cellcolor[HTML]{DAE9F8}0.36 & \cellcolor[HTML]{DAE9F8}0.36 & \cellcolor[HTML]{A6C9EC}0.62 & \cellcolor[HTML]{A6C9EC}0.62 & \cellcolor[HTML]{A6C9EC}0.62 & \cellcolor[HTML]{DAE9F8}0.23 & \cellcolor[HTML]{DAE9F8}0.23 & \cellcolor[HTML]{DAE9F8}0.23 & \cellcolor[HTML]{DAE9F8}0.23 & \cellcolor[HTML]{DAE9F8}0.23 & \cellcolor[HTML]{DAE9F8}0.23 & \cellcolor[HTML]{DAE9F8}0.23 & \cellcolor[HTML]{DAE9F8}0.23 & \cellcolor[HTML]{DAE9F8}0.23 & \cellcolor[HTML]{DAE9F8}0.16 & \cellcolor[HTML]{DAE9F8}0.16 & \cellcolor[HTML]{DAE9F8}0.16 \\ \cline{1-1}
\textbf{4} & \cellcolor[HTML]{DAE9F8}0.36 & \cellcolor[HTML]{DAE9F8}0.36 & \cellcolor[HTML]{DAE9F8}0.36 & \cellcolor[HTML]{A6C9EC}0.62 & \cellcolor[HTML]{A6C9EC}0.62 & \cellcolor[HTML]{A6C9EC}0.62 & \cellcolor[HTML]{DAE9F8}0.23 & \cellcolor[HTML]{DAE9F8}0.23 & \cellcolor[HTML]{DAE9F8}0.23 & \cellcolor[HTML]{DAE9F8}0.23 & \cellcolor[HTML]{DAE9F8}0.23 & \cellcolor[HTML]{DAE9F8}0.23 & \cellcolor[HTML]{DAE9F8}0.23 & \cellcolor[HTML]{DAE9F8}0.23 & \cellcolor[HTML]{DAE9F8}0.23 & \cellcolor[HTML]{DAE9F8}0.16 & \cellcolor[HTML]{DAE9F8}0.16 & \cellcolor[HTML]{DAE9F8}0.16 \\ \cline{1-1}
\textbf{5} & \cellcolor[HTML]{DAE9F8}0.36 & \cellcolor[HTML]{DAE9F8}0.36 & \cellcolor[HTML]{DAE9F8}0.36 & \cellcolor[HTML]{A6C9EC}0.62 & \cellcolor[HTML]{A6C9EC}0.62 & \cellcolor[HTML]{A6C9EC}0.62 & \cellcolor[HTML]{DAE9F8}0.23 & \cellcolor[HTML]{DAE9F8}0.23 & \cellcolor[HTML]{DAE9F8}0.23 & \cellcolor[HTML]{DAE9F8}0.23 & \cellcolor[HTML]{DAE9F8}0.23 & \cellcolor[HTML]{DAE9F8}0.23 & \cellcolor[HTML]{DAE9F8}0.23 & \cellcolor[HTML]{DAE9F8}0.23 & \cellcolor[HTML]{DAE9F8}0.23 & \cellcolor[HTML]{DAE9F8}0.16 & \cellcolor[HTML]{DAE9F8}0.16 & \cellcolor[HTML]{DAE9F8}0.16 \\ \cline{1-1}
\textbf{6} & \cellcolor[HTML]{DAE9F8}0.28 & \cellcolor[HTML]{DAE9F8}0.28 & \cellcolor[HTML]{DAE9F8}0.28 & \cellcolor[HTML]{DAE9F8}0.23 & \cellcolor[HTML]{DAE9F8}0.23 & \cellcolor[HTML]{DAE9F8}0.23 & \cellcolor[HTML]{F1A983}0.33 & \cellcolor[HTML]{F1A983}0.33 & \cellcolor[HTML]{F1A983}0.33 & \cellcolor[HTML]{FBE2D5}0.08 & \cellcolor[HTML]{FBE2D5}0.08 & \cellcolor[HTML]{FBE2D5}0.08 & \cellcolor[HTML]{FBE2D5}0.08 & \cellcolor[HTML]{FBE2D5}0.08 & \cellcolor[HTML]{FBE2D5}0.08 & \cellcolor[HTML]{FBE2D5}0.06 & \cellcolor[HTML]{FBE2D5}0.06 & \cellcolor[HTML]{FBE2D5}0.06 \\ \cline{1-1}
\textbf{7} & \cellcolor[HTML]{DAE9F8}0.28 & \cellcolor[HTML]{DAE9F8}0.28 & \cellcolor[HTML]{DAE9F8}0.28 & \cellcolor[HTML]{DAE9F8}0.23 & \cellcolor[HTML]{DAE9F8}0.23 & \cellcolor[HTML]{DAE9F8}0.23 & \cellcolor[HTML]{F1A983}0.33 & \cellcolor[HTML]{F1A983}0.33 & \cellcolor[HTML]{F1A983}0.33 & \cellcolor[HTML]{FBE2D5}0.08 & \cellcolor[HTML]{FBE2D5}0.08 & \cellcolor[HTML]{FBE2D5}0.08 & \cellcolor[HTML]{FBE2D5}0.08 & \cellcolor[HTML]{FBE2D5}0.08 & \cellcolor[HTML]{FBE2D5}0.08 & \cellcolor[HTML]{FBE2D5}0.06 & \cellcolor[HTML]{FBE2D5}0.06 & \cellcolor[HTML]{FBE2D5}0.06 \\ \cline{1-1}
\textbf{8} & \cellcolor[HTML]{DAE9F8}0.28 & \cellcolor[HTML]{DAE9F8}0.28 & \cellcolor[HTML]{DAE9F8}0.28 & \cellcolor[HTML]{DAE9F8}0.23 & \cellcolor[HTML]{DAE9F8}0.23 & \cellcolor[HTML]{DAE9F8}0.23 & \cellcolor[HTML]{F1A983}0.33 & \cellcolor[HTML]{F1A983}0.33 & \cellcolor[HTML]{F1A983}0.33 & \cellcolor[HTML]{FBE2D5}0.08 & \cellcolor[HTML]{FBE2D5}0.08 & \cellcolor[HTML]{FBE2D5}0.08 & \cellcolor[HTML]{FBE2D5}0.08 & \cellcolor[HTML]{FBE2D5}0.08 & \cellcolor[HTML]{FBE2D5}0.08 & \cellcolor[HTML]{FBE2D5}0.06 & \cellcolor[HTML]{FBE2D5}0.06 & \cellcolor[HTML]{FBE2D5}0.06 \\ \cline{1-1}
\textbf{9} & \cellcolor[HTML]{DAE9F8}0.28 & \cellcolor[HTML]{DAE9F8}0.28 & \cellcolor[HTML]{DAE9F8}0.28 & \cellcolor[HTML]{DAE9F8}0.23 & \cellcolor[HTML]{DAE9F8}0.23 & \cellcolor[HTML]{DAE9F8}0.23 & \cellcolor[HTML]{C1F0C8}0.08 & \cellcolor[HTML]{C1F0C8}0.08 & \cellcolor[HTML]{C1F0C8}0.08 & \cellcolor[HTML]{47D359}0.33 & \cellcolor[HTML]{47D359}0.33 & \cellcolor[HTML]{47D359}0.33 & \cellcolor[HTML]{C1F0C8}0.08 & \cellcolor[HTML]{C1F0C8}0.08 & \cellcolor[HTML]{C1F0C8}0.08 & \cellcolor[HTML]{C1F0C8}0.06 & \cellcolor[HTML]{C1F0C8}0.06 & \cellcolor[HTML]{C1F0C8}0.06 \\ \cline{1-1}
\textbf{10} & \cellcolor[HTML]{DAE9F8}0.28 & \cellcolor[HTML]{DAE9F8}0.28 & \cellcolor[HTML]{DAE9F8}0.28 & \cellcolor[HTML]{DAE9F8}0.23 & \cellcolor[HTML]{DAE9F8}0.23 & \cellcolor[HTML]{DAE9F8}0.23 & \cellcolor[HTML]{C1F0C8}0.08 & \cellcolor[HTML]{C1F0C8}0.08 & \cellcolor[HTML]{C1F0C8}0.08 & \cellcolor[HTML]{47D359}0.33 & \cellcolor[HTML]{47D359}0.33 & \cellcolor[HTML]{47D359}0.33 & \cellcolor[HTML]{C1F0C8}0.08 & \cellcolor[HTML]{C1F0C8}0.08 & \cellcolor[HTML]{C1F0C8}0.08 & \cellcolor[HTML]{C1F0C8}0.06 & \cellcolor[HTML]{C1F0C8}0.06 & \cellcolor[HTML]{C1F0C8}0.06 \\ \cline{1-1}
\textbf{11} & \cellcolor[HTML]{DAE9F8}0.28 & \cellcolor[HTML]{DAE9F8}0.28 & \cellcolor[HTML]{DAE9F8}0.28 & \cellcolor[HTML]{DAE9F8}0.23 & \cellcolor[HTML]{DAE9F8}0.23 & \cellcolor[HTML]{DAE9F8}0.23 & \cellcolor[HTML]{C1F0C8}0.08 & \cellcolor[HTML]{C1F0C8}0.08 & \cellcolor[HTML]{C1F0C8}0.08 & \cellcolor[HTML]{47D359}0.33 & \cellcolor[HTML]{47D359}0.33 & \cellcolor[HTML]{47D359}0.33 & \cellcolor[HTML]{C1F0C8}0.08 & \cellcolor[HTML]{C1F0C8}0.08 & \cellcolor[HTML]{C1F0C8}0.08 & \cellcolor[HTML]{C1F0C8}0.06 & \cellcolor[HTML]{C1F0C8}0.06 & \cellcolor[HTML]{C1F0C8}0.06 \\ \cline{1-1}
\textbf{12} & \cellcolor[HTML]{DAE9F8}0.28 & \cellcolor[HTML]{DAE9F8}0.28 & \cellcolor[HTML]{DAE9F8}0.28 & \cellcolor[HTML]{DAE9F8}0.23 & \cellcolor[HTML]{DAE9F8}0.23 & \cellcolor[HTML]{DAE9F8}0.23 & \cellcolor[HTML]{F2F2F2}0.08 & \cellcolor[HTML]{F2F2F2}0.08 & \cellcolor[HTML]{F2F2F2}0.08 & \cellcolor[HTML]{F2F2F2}0.08 & \cellcolor[HTML]{F2F2F2}0.08 & \cellcolor[HTML]{F2F2F2}0.08 & \cellcolor[HTML]{BFBFBF}0.33 & \cellcolor[HTML]{BFBFBF}0.33 & \cellcolor[HTML]{BFBFBF}0.33 & \cellcolor[HTML]{F2F2F2}0.06 & \cellcolor[HTML]{F2F2F2}0.06 & \cellcolor[HTML]{F2F2F2}0.06 \\ \cline{1-1}
\textbf{13} & \cellcolor[HTML]{DAE9F8}0.28 & \cellcolor[HTML]{DAE9F8}0.28 & \cellcolor[HTML]{DAE9F8}0.28 & \cellcolor[HTML]{DAE9F8}0.23 & \cellcolor[HTML]{DAE9F8}0.23 & \cellcolor[HTML]{DAE9F8}0.23 & \cellcolor[HTML]{F2F2F2}0.08 & \cellcolor[HTML]{F2F2F2}0.08 & \cellcolor[HTML]{F2F2F2}0.08 & \cellcolor[HTML]{F2F2F2}0.08 & \cellcolor[HTML]{F2F2F2}0.08 & \cellcolor[HTML]{F2F2F2}0.08 & \cellcolor[HTML]{BFBFBF}0.33 & \cellcolor[HTML]{BFBFBF}0.33 & \cellcolor[HTML]{BFBFBF}0.33 & \cellcolor[HTML]{F2F2F2}0.06 & \cellcolor[HTML]{F2F2F2}0.06 & \cellcolor[HTML]{F2F2F2}0.06 \\ \cline{1-1}
\textbf{14} & \cellcolor[HTML]{DAE9F8}0.28 & \cellcolor[HTML]{DAE9F8}0.28 & \cellcolor[HTML]{DAE9F8}0.28 & \cellcolor[HTML]{DAE9F8}0.23 & \cellcolor[HTML]{DAE9F8}0.23 & \cellcolor[HTML]{DAE9F8}0.23 & \cellcolor[HTML]{F2F2F2}0.08 & \cellcolor[HTML]{F2F2F2}0.08 & \cellcolor[HTML]{F2F2F2}0.08 & \cellcolor[HTML]{F2F2F2}0.08 & \cellcolor[HTML]{F2F2F2}0.08 & \cellcolor[HTML]{F2F2F2}0.08 & \cellcolor[HTML]{BFBFBF}0.33 & \cellcolor[HTML]{BFBFBF}0.33 & \cellcolor[HTML]{BFBFBF}0.33 & \cellcolor[HTML]{F2F2F2}0.06 & \cellcolor[HTML]{F2F2F2}0.06 & \cellcolor[HTML]{F2F2F2}0.06 \\ \cline{1-1}
\textbf{15} & \cellcolor[HTML]{DAE9F8}0.25 & \cellcolor[HTML]{DAE9F8}0.25 & \cellcolor[HTML]{DAE9F8}0.25 & \cellcolor[HTML]{DAE9F8}0.16 & \cellcolor[HTML]{DAE9F8}0.16 & \cellcolor[HTML]{DAE9F8}0.16 & \cellcolor[HTML]{F2CEEF}0.06 & \cellcolor[HTML]{F2CEEF}0.06 & \cellcolor[HTML]{F2CEEF}0.06 & \cellcolor[HTML]{F2CEEF}0.06 & \cellcolor[HTML]{F2CEEF}0.06 & \cellcolor[HTML]{F2CEEF}0.06 & \cellcolor[HTML]{F2CEEF}0.06 & \cellcolor[HTML]{F2CEEF}0.06 & \cellcolor[HTML]{F2CEEF}0.06 & \cellcolor[HTML]{D86DCD}0.6 & \cellcolor[HTML]{D86DCD}0.6 & \cellcolor[HTML]{D86DCD}0.6 \\ \cline{1-1}
\textbf{16} & \cellcolor[HTML]{DAE9F8}0.25 & \cellcolor[HTML]{DAE9F8}0.25 & \cellcolor[HTML]{DAE9F8}0.25 & \cellcolor[HTML]{DAE9F8}0.16 & \cellcolor[HTML]{DAE9F8}0.16 & \cellcolor[HTML]{DAE9F8}0.16 & \cellcolor[HTML]{F2CEEF}0.06 & \cellcolor[HTML]{F2CEEF}0.06 & \cellcolor[HTML]{F2CEEF}0.06 & \cellcolor[HTML]{F2CEEF}0.06 & \cellcolor[HTML]{F2CEEF}0.06 & \cellcolor[HTML]{F2CEEF}0.06 & \cellcolor[HTML]{F2CEEF}0.06 & \cellcolor[HTML]{F2CEEF}0.06 & \cellcolor[HTML]{F2CEEF}0.06 & \cellcolor[HTML]{D86DCD}0.6 & \cellcolor[HTML]{D86DCD}0.6 & \cellcolor[HTML]{D86DCD}0.6 \\ \cline{1-1}
\textbf{17} & \cellcolor[HTML]{DAE9F8}0.25 & \cellcolor[HTML]{DAE9F8}0.25 & \cellcolor[HTML]{DAE9F8}0.25 & \cellcolor[HTML]{DAE9F8}0.16 & \cellcolor[HTML]{DAE9F8}0.16 & \cellcolor[HTML]{DAE9F8}0.16 & \cellcolor[HTML]{F2CEEF}0.06 & \cellcolor[HTML]{F2CEEF}0.06 & \cellcolor[HTML]{F2CEEF}0.06 & \cellcolor[HTML]{F2CEEF}0.06 & \cellcolor[HTML]{F2CEEF}0.06 & \cellcolor[HTML]{F2CEEF}0.06 & \cellcolor[HTML]{F2CEEF}0.06 & \cellcolor[HTML]{F2CEEF}0.06 & \cellcolor[HTML]{F2CEEF}0.06 & \cellcolor[HTML]{D86DCD}0.6 & \cellcolor[HTML]{D86DCD}0.6 & \cellcolor[HTML]{D86DCD}0.6 \\ \cline{1-1}
\hline
\end{tabular}
}
\caption{$\Bar{\mA} = \mA+\mI$, Symmetric normalized adjacency matrix of $\gG$, $\Bar{\mA}_{sym} =\Bar{\mD}^{-\frac{1}{2}}\Bar{\mA}\Bar{\mD}^{-\frac{1}{2}}$, and the pairwise similarity weight after aggregation is, $S(\Bar{\mA}_{sym},\mX) = \Bar{\mA}_{sym}\mX (\Bar{\mA}_{sym}\mX )^T$. The highlighted values show that this low-pass filter cannot properly classify heterophilic nodes. Here, nodes $0-2$ have positive similarity weights to all other nodes.}
\label{tab:lowpass}
\end{table}

In \ags{} based sampler, along with similarity based sampling, we also sample the neighborhood of vertices based on feature diversity. The feature diversity-based sampler will attempt to create the same diverse neighborhood of two heterophilic nodes, thus potentially mapping in the same vector space. 
In a low-pass filter, however, the neighboring nodes of the same label as the ego node with similar features help to map them in the same space; therefore, the feature-similarity-based sampler will work in conjunction with the low-pass filter.
\ags{} uses a dual channel of feature-similar and feature-diverse samples for each node in the graph and adaptively learns a better representation.

\section{Appendix: Additional Details on \ags{}}
\label{sec:apadditonaldetails}

\subsection{Probability Mass Functions}

Fig.~\ref{fig:densityfunction} shows some possible Probability Mass Functions that take the rank of neighbors as input and return weights for sampling probabilities.

\begin{figure}[!htbp]
\centering
\includegraphics[width=0.9\linewidth]{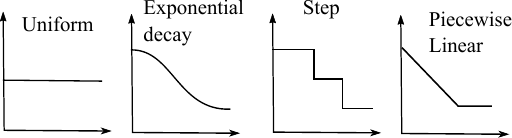}
\caption{Probability Mass Functions (PMFs) for weights used later for selection probabilities}
\label{fig:densityfunction}
\end{figure}

\subsection{Computation of sampling probability for similar samples}
\label{subsec:knnweights}
Algorithm~\ref{alg:nnweights} shows how the similarity based sampling algorithm with the \emph{step} function is used to rank and pre-compute the probability distribution. Here we can use any known similarity functions or a learned function. 
The pseudocode shows how step functions with three different selection probabilities are used to compute sampling probabilities.

\begin{algorithm}[H]
\caption{RankBySimilarity($\gG, \mX, \gS, \gP, \mathrm{*params}$)}

\SetKwInOut{Input}{Input}
\Input{Graph $\gG(\gV, \gE)$, Feature matrix $\mX \in \R^{n\times f}$, Similarity function $\gS$, Probability Mass Function, $\gP$}
\SetKwInOut{Output}{Output}

\begin{algorithmic}[1]

\IF{$\gP= \mathrm{'step'}$}
\STATE $\mathrm{k_1 = params['k_1'], k_2 = params['k_2']}$
\STATE $\mathrm{\lambda_1 = params['\lambda_1'],\lambda_2 = params['\lambda_2'],\lambda_3 = params['\lambda_3']}$
\ENDIF

\FOR{ $u \in \gV$}
    \STATE $\gN(u)$ is the neighbors of $u$   
    \STATE Similarity from $u$ to its neighbor, $S_u = \mathrm{\gS(\mX[u], \mX[\gN(u)])}$


    \STATE $\mathrm{top_{k1} = \floor{k_1\times|\gN(u)|}}$
    \STATE $\mathrm{top_{k2} = \floor{k_2\times|\gN(u)|}}$
    \STATE $\mathrm{l = |\gN(u)|-top_{k1}-top_{k2}}$
    
    \STATE Partition the similarities at $\mathrm{top_{k1}, top_{k2}}$ position in descending order of $S_u$ and rank them
    
    
    \STATE Weight $\mathrm{W_{u,v} = \lambda_1}$ to $\mathrm{top_{k1}}$ neighbors, $v \in \gN(u)$ 
    \STATE Weight $\mathrm{W_{u,v} = \lambda_2}$ to next $\mathrm{top_{k2}}$, $v \in \gN(u)$ 
    \STATE Weight $\mathrm{W_{u,v} = \lambda_3}$ to residual $l$ neighbors, $v \in \gN(u)$
    
    \STATE Sampling probability of an edge, $P_{u,v} = W_{u,v}/\sum_{v\in \gN(u)} W_{u,v}$
\ENDFOR
\STATE \Return $P$
\end{algorithmic}
\label{alg:nnweights}
\end{algorithm}


    

\subsection{Computation of sampling probability for diverse samples}
\label{subsec:submodular_funcs}
We can use submodular functions to rank the neighbors of a vertex. 
Multiple submodular functions are available for finding a useful subset from a larger set;  e.g., \texttt{`MaxCoverage`, `FeatureBased`, `GraphBased`, `FacilityLocation`}, etc. 
Algorithm~\ref{alg:subweights} shows the pseudocode for computing sampling probabilities using ranking by diversity. 
Algorithm~\ref{alg:lazygreddy} shows the pseudocode of a simplified Lazy 
Greedy algorithm using the facility location function. For distance measures such as   `Euclidean`, \texttt{distance.max()-distance} is used to convert into similarity measure. Note that the ego node $u$ is already taken as an initial set for submodular selections, and the remaining nodes are to be selected from the neighbors.


\begin{algorithm}[!htbp]
\caption{RankByDiversity ($\gG, \mX, \gV, \gS, \gP, \mathrm{*params}$)}

\SetKwInOut{Output}{Output}
\SetKwInOut{Input}{Input}

\Input{Graph $\gG(\gV, \gE)$, Feature matrix $\mX \in \R^{n\times f}$, Similarity function $\gS$, Probability Mass Function, $\gP$}


\begin{algorithmic}[1]
\FOR{ $u \in \gV$}
    \STATE $\gN(u)$ is the neighbors of $u$
    \STATE Ranks, $\mathrm{R_u = LazyGreedy(\gN(u), u, X_u|X_{\gN(u)})}$
    \STATE Assign weights $W_{u,v}$ for $v \in \gN(u)$ from the ranks $R_u$ using $\gP$. 
    \STATE Sampling probability of neighbors, $P_{u,v} = W_{u,v}/\sum_{v\in \gN(u)} W_{u,v}$ 
\ENDFOR
\STATE \Return $P$
\end{algorithmic}
\label{alg:subweights}
\end{algorithm}

\begin{algorithm}[!htbp]
\caption{LazyGreedy ($\gN(u), u, \mX', \gS$)}

\SetKwInOut{Input}{Input}

\Input{$\gN(u)$ is the neighbors of ego node $u$, $\mX'\in \R^{|\gN(u)\cup \{u\}|}$ is the feature of $\gN(u)\cup \{u\}$, Similarity function $\gS$}

\begin{algorithmic}[1]

\STATE $\mathrm{kernel = compute\_pairwise\_similarity(\mX', \gS)}$
\STATE Initial set, $S =\{u\}$
\STATE Ground set, $Y$ = $\{u\} \cup \gN(u)$
\STATE Gain of set $S$ is $\mathrm{S_{gain} = Gain (S, Y, kernel)}$ \COMMENT {Eq.\ref{eq:facility_location_function} is used}

\STATE Initialize max-heap $H$, where key is the marginal gain w.r.t $S$ and the element is neighboring vertex $v$

\WHILE{$H\neq\emptyset$}
    \STATE  $\mathrm{(gain, v)=H.pop()}$
    \STATE  $\mathrm{gain_{v} = Gain(S\cup \{v\}, Y, kernel)-S_{gain}}$ 
    \STATE $\mathrm{(gain_2,\cdot ) = H.top()}$    
    \IF {$\mathrm{gain_v\ge gain_2}$}        
        \STATE $\mathrm{S=S\cup\{v\}}$
        \STATE $\mathrm{S_{gain} = Gain (S, Y, kernel)}$
    \ELSE
        \STATE $\mathrm{H.push(gain_v, v)}$
    \ENDIF    
\ENDWHILE

\STATE Rank the vertices by marginal gain, $\mathrm{R_u = Rank(S)}$    
\STATE \Return $R_u$

\end{algorithmic}
\label{alg:lazygreddy}
\end{algorithm}

Ablation studies on different submodular functions are provided in Appendix~\ref{subsec:submdoularablation}, and more details on functions and implementations are provided in Apricot~\citep{schreiber2020apricot} (\href{https://apricot-select.readthedocs.io/en/latest/index.html}{url}). Similar to the pseudocode of \texttt{RankBySimilarity} in Algorithm~\ref{alg:nnweights}, we can also use a step function to reduce computation complexity by ranking only the top few items.

\subsection{Training and inference of a similarity learning model}
\label{subsec:edge_weight}

We use the Siamese~\citep{chicco2021siamese} model to predict the edge weights given two node features. Here are simplified equations showing the architecture used for the regression task of edge weight predictions. 

The network can be described by 
\begin{align*}
    \mathrm{mlp_1} & = \mathrm{MLP}(f, H_1),\\
    \mathrm{mlp_2} &= \mathrm{MLP}(H_1, H_2),\\
    \mathrm{mlp_3} &= \mathrm{MLP}(2*H_2,1). 
\end{align*}
The forward propagation is described by 
\begin{align*}
    \vx_1 &= \sigma_1(\mathrm{mlp_2}(\sigma_1(\mathrm{mlp_1}(\vx_u))))\\
    \vx_2 &=  \sigma_1(\mathrm{mlp_2}(\sigma_1(\mathrm{mlp_1}(\vx_v))))\\
    \vx &= \sigma_2(\mathrm{mlp_3}(\vx_1-\vx_2|
    \vx_1\times \vx_2)). \\   
\end{align*}
Here MLP corresponds to Multi-Layer Perceptron or linear neural network layers,  $\sigma_1$ refers to \emph{ReLU} activation function,  and $\sigma_2$ is \emph{Sigmoid} for predicting edge weight between $0$ to $1$. The \emph{MSELoss} function is used as an optimization function.

\begin{algorithm}[H]
\caption{LearnSimilarity ($\gG, \mX, \vy$)}

\SetKwInOut{Input}{Input}
\Input{Graph $\gG(\gV, \gE)$, Feature matrix $\mX \in \R^{n\times f}$, Label $\vy\in \sY^n$}

\begin{algorithmic}[1]

\STATE Extract the training subgraph, $\gG'$, containing training vertices only.

\STATE Initialize a model, $\mathrm{SIM_{\mW}}$

\WHILE{condition}
    \STATE  Batch $B$ contains an equal number of edges and non-edges in the $\gG'$
    \STATE Target label, $T = [(y_u = y_v): (u,v)\in B]$ 
    \STATE  Train, $\mathrm{SIM_{\mW}}$ using $B, T$
\ENDWHILE

\STATE Get edge weight, $\mathrm{\gE' = SIM_{\mW}(\gE)}$
\STATE \Return $\mathrm{SIM_{\mW}}, \gE'$
\end{algorithmic}
\label{alg:learnweights}
\end{algorithm}
Algorithm~\ref{alg:learnweights} shows the training process for the regression task of predicting edge weights between $0$ and  $1$. We form a mini-batch by sampling an equal number of edges and non-edges from the training subgraph. The target labels are computed using the labels of each endpoint of the edges.
For a few training nodes, training subgraphs mostly contain isolated vertices. In such a scenario, we can construct a new graph where all training nodes with the same class are connected and use that for training purposes.


\section{Appendix: AGS-GS: \ags{} with Graph Sampling}
\label{sec:graphsampling}

This section covers how feature-similarity and feature-diversity-based samplers can be used in graph sampling. The graph sampling paradigm is significantly faster than node sampling since it avoids neighborhood expansions.  

\subsection{\agsgsrw{}: Graph Sampling with Random Walk}
\label{subsec:agsgsrw}
A simple way to perform graph sampling is to use weighted random walks on the precomputed graphs, as shown in Fig.~\ref{fig:wrw_graphsampling}.

\begin{figure}[!htbp]
\centering
\subfloat[Weighted Random walk]{%
\includegraphics[width=1.0\linewidth]{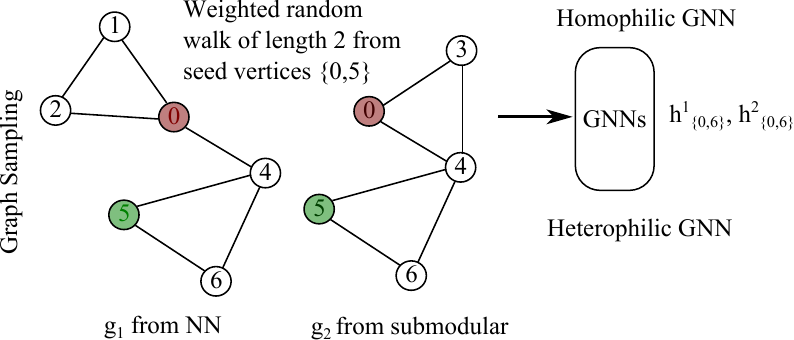}%
}
\caption{Graph Sampling using Weighted Random walk. The weights for sampling are computed based on the ranks by similarity and diversity.}
\label{fig:wrw_graphsampling}
\end{figure}

\begin{algorithm}[!htbp]
\caption{AGS-GS-RW ($\gG, \mX, \vy$) [Weighted Random Walk]}
\SetKwInOut{Input}{Input}
\Input{Graph $\gG(\gV, \gE)$, Feature matrix $\mX \in \R^{n\times f}$, Label, $\vy\in \sY^n$}
\label{alg:agsgraphsamplingrw}
\begin{algorithmic}[1]  
    \STATE  $\gS =  \mathrm{LearnSimilarity (\gG,\mX,\vy)}$
    \tcc{model for feature similarity, if an appropriate function not known, Alg.~\ref{alg:learnweights}}
    \STATE  $\gP =  \mathrm {'step'}$ \tcc{probability mass function}
    
    \STATE $\mathrm{R_1 = RankBySimilarity(\gG, \mX, \gS, \gP)}$ \tcc{Algorithm~\ref{alg:nnweights}}
    
    \STATE $\mathrm{R_1 = RankByDiversity(\gG, \mX, \gS, \gP)}$ \tcc{Algorithm~\ref{alg:subweights}}
    
    \FOR{$\mathrm{epoch}$ \textbf{in} $\mathrm{num\_epochs}$}
        \STATE $\mathrm{nodes = BatchOfSeedNodes(\gG)}$ \tcc{random nodes from $\gV$}
        \STATE $\mathrm{g_1 = SubgraphSample(R_1, nodes)}$ \tcc{for homophily}
        \STATE $\mathrm{g_2 = SubgraphSample(R_2, nodes)}$ \tcc{for heterophily}
        \STATE 
        \STATE - \tcc{take output for homophilic GNN, heterophilic GNN, and combined representation}
        $\mathrm{out_{g_1}, out_{g_2}, out_{com} = GNN_{\mW}(g_1, g_2)}$
        \STATE $l_1 = \mathrm{loss(out_{g_1},\vy[g_1.train])}$
        \STATE $l_2 = \mathrm{loss(out_{g_2},\vy[g_1.train])}$
        \STATE $l_3 = \mathrm{loss(out_{com},\vy[nodes.train])}$
        \STATE $l = l_1+l_2+l_3$

        \STATE Backpropagate from $l$ and update $\mW$        
    \ENDFOR
\end{algorithmic}
\end{algorithm}

Algorithm~\ref{alg:agsgraphsamplingrw} shows the pseudocode when the weighted random walk is performed. Here for a minibatch we first generate random seed nodes from $\gV$ and then start the random walk of  $k$ steps from these nodes to generate subgraphs. If we use a single channel, the training process can be similar to GSAINT. However, if two channels are used, the nodes of sampled subgraphs from two random walks will differ, and  there are  a few choices to compute loss: 
\begin{enumerate}[label=(\roman*)]
    \item Only compute the loss of training nodes from the seed nodes.
    \item Compute loss from common training nodes of the two sampled subgraphs.
    \item Combine the loss of training nodes from seed nodes, training nodes of the first and second subgraphs
\end{enumerate}

During inference, we can sample two subgraphs starting from the test nodes and check predictions from the combination network.

\subsection{\agsgsdis{}: \ags{} with edge-disjoint subgraph sampling}
\label{subsec:agsgsdisjoint}

\subsubsection{Edge Disjoint Sparse Graph for Graph Sampling}
We can use the pre-computed weighted graph from Section~\ref{subsec:precomputation} and go one step further in graph sampling and incorporate some heuristics into the sampling process. To do this, we take some high-quality sparse subgraphs with unique properties and compute edge-disjoint subgraphs. Subgraph selection examples include \texttt{Spanning Trees}, \texttt{$k$-Nearest Neighbor}, \texttt{$b$-matching subgraph}, and \texttt{Spectral Sparsifier}.
Algorithm~\ref{alg:disjointgraph} shows the process of computing edge-disjoint subgraphs and their corresponding weights. Note that the weight of the residual subgraph is set to zero.

Consider the heuristic of ensuring connectivity among sampled nodes. We 
compute $k$-edge-disjoint Maximum Spanning Trees (MSTs) for a graph, sample a few of them, and combine these subgraphs to get a sparse representation of the original graph. 
Note that the best-quality sparse graphs will be in in the order of selection
when similarity and diversity-based edge weights are considered for edge-disjoint subgraph computation. Therefore, the importance of the sparse graph quality should also be considered in the sparse graph sampling.
Fig.~\ref{fig:graphsampling_paradigm} shows how an edge-disjoint subgraph can be used to sample a sparse representation. In the example scenario, the original graph $\gG$ is split into $k = 4$ edge-disjoint subgraphs $\{(g_1, w_1), (g_1, w_1), (g_1, w_1),(g_4, w_4 )\}$. The $w_i$ is the weight of the subgraph used for sampling.
Here $g_4 = \gG-g_1-g_2-g_3$ is the residual subgraph. In the {\tt DisjointSubgraphSample} process, we randomly sample $k=2$ of these subgraphs from $g_1, g_2, g_3$, excluding the residual graph. We combine these sampled subgraphs and add random edges from the residual graph.

\begin{figure}[!htbp]
 \centering
 \includegraphics[width=1.0\linewidth, angle = 0 ]{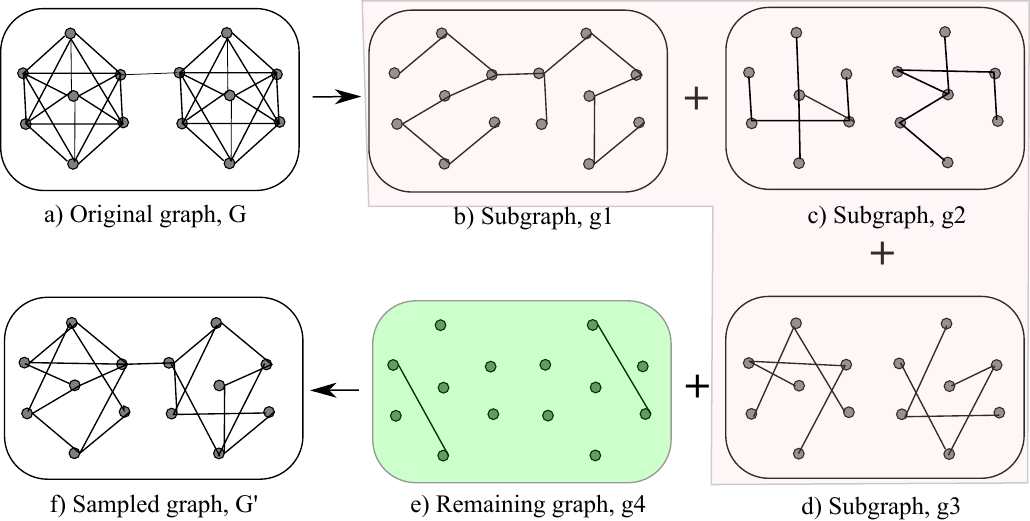}
 \caption{We split the original graph $G$ into a collection of disjoint edges (here, a spanning tree is used), and then we sample such subgraphs to get a sparse graph, $G' = g_1+g_2$ for an epoch/batch.}
 \label{fig:graphsampling_paradigm}
\end{figure}

\begin{algorithm}[!htbp]
\caption{Disjoint Graphs ($\gG, \mX, R, K, \mathrm{\gF = 'MST_{max}'}$)}
\SetKwInOut{Input}{Input}
\Input{Graph $\gG(\gV, \gE)$, Feature matrix $\mX \in \R^{n\times f}$, $R$ is the rank from similarity or diversity of the neighborhood, $K$ is the number of edge-disjoint graph, {\tt Subgraph} selection $\gF$ could be MST, $k$-NN, $b$-matching, or combinations of these}


\begin{algorithmic}[1]

\STATE Get edge weights, $W$, from rank $R$ using Probability Mass Function $\gP$

\STATE Graph collections, $\mathrm{\gG_{col}=\{\}}$, Graph Weights, $\mathrm{\gG_{w}=\{\}}$

\WHILE{True}
    \STATE  Compute, $\mathrm{Sub_\gG = \gF(\gG, W)}$
    \STATE  $\mathrm{\gG_{col} = \gG_{col}
    \cup Sub_\gG}$
    \STATE - \tcc{ make sure the sparse graph weight doesn't become too small}
    \STATE $\mathrm{\gG_w = 
    \gG_w \cup max(Weight(Sub_\gG), K\times1e^{-3})}$
    \STATE $\mathrm{\gG = \gG - Sub_\gG}$  
    
    \STATE $K = K-1$    
    \IF{$K = 0$}
        \STATE $\mathrm{\gG_{col} = \gG_{col} \cup Sub_\gG}$        
        \STATE $\mathrm{\gG_w = 
    \gG_w \cup \{0\}}$
        \STATE break
    \ENDIF
    
\ENDWHILE
\STATE \Return $\mathrm{\gG_{col}, \gG_w}$
\end{algorithmic}
\label{alg:disjointgraph}
\end{algorithm}

\subsubsection{AGS-GS-Disjoint: \ags{} with edge-disjoint graph sampling}
Algorithm~\ref{alg:agsgraphsampling} shows overall \ags{} as a graph sampling process for training when disjoint graphs sampling with heuristics are used. For simplicity, we include the pseudocode on a single channel; the dual-channel implementation is similar to the Random Walk version. 
Here, we sample a sparse graph at the start of the epoch and perform random walk-based graph sampling as a mini-batch process to ensure that the subgraph remains small for GNN iterations. The second step is unnecessary if the initial sampled sparse graph fits in the  GPU memory.
We can also directly perform random-walk on the collection of edge-disjoint subgraphs instead of the two-stage process.

\begin{algorithm}[!htbp]
\caption{AGS-GS-Disjoint ($\gG, \mX, \vy$) [Disjoint subgraph sample* (single channel)]}
\SetKwInOut{Input}{Input}
\Input{Graph $\gG(\gV, \gE)$, Feature matrix $\mX \in \R^{n\times f}$, Label, $\vy\in \sY^n$}

\BlankLine
\begin{algorithmic}[1]  

    \STATE  $\gS =  \mathrm{LearnSimilarity (\gG,\mX,\vy)}$
    \tcc{model for feature similarity, if an appropriate function not known, Alg.~\ref{alg:learnweights}}
    \STATE  $\gP =  \mathrm {'step'}$ \tcc{probability mass function}
    
    \STATE $\mathrm{R = RankBySimilarity(\gG, \mX, \gS, \gP)}$ \tcc{Algorithm~\ref{alg:nnweights}}
    
    \STATE $\mathrm{(\gG_{col}, \gG_w) = DisjointGraphs( \gG, \mX, R, K, \gF = `ST_{max}`)}$ \tcc{Algorithm~\ref{alg:disjointgraph}}
    
    \FOR{$\mathrm{epoch}$ \textbf{in} $\mathrm{num\_epochs}$}    
        \STATE - \tcc{sample a sparse graph using the process explained in Fig.~\ref{fig:graphsampling_paradigm}}
        
        \STATE $\mathrm{\gG' = DisjointSubgraphSample(\gG_{col}, \gG_w, *params)}$

        \FOR{$\mathrm{mini\_batch}$ \textbf{in} $\mathrm{num\_batches}$}  
            \STATE $\mathrm{nodes = BatchOfSeedNodes(\gG')}$
            \STATE $\mathrm{g = SubgraphSample(\gG',nodes)}$\tcc{weighted random walk on the sparse graph}
            \STATE $\mathrm{output = GNN_{\mW}(g)}$
            \STATE $l = \mathrm{loss(output, \vy[g.train])}$
            \STATE Backpropagate from $l$ and update $\mW$
        \ENDFOR
    \ENDFOR
\end{algorithmic}
\label{alg:agsgraphsampling}
\end{algorithm}

\section{Appendix: Dataset Details}
\label{sec:appendixdataset}

\subsection{Dataset Description}
\label{subsec:appdatasetdescription}

\subsubsection{Dataset with Heterophily:} \texttt{Cornell, Texas}, and \texttt{Wisconsin} are three sub-datasets from the \texttt{WebKB} set~\citep{pei2020geom}. \texttt{Chameleon} and \texttt{Squirrel} are web page datasets collected from Wikipedia~\citep{rozemberczki2021multi}. The \texttt{actor} is an actor co-occurrence network~\citep{pei2020geom}, where nodes are actors and edges represent two actors whose names occur together on the same Wikipedia page. 
Recently ~\citep{lim2021large} collected and released a series of large-scale benchmark datasets from diverse areas which include a Wikipedia web page dataset (\texttt{Wiki}), two citation network datasets (\texttt{ArXiv-year} and \texttt{Snap-Patents}), online social network datasets (\texttt{\justify Penn94, Pokec, Genius}, and \texttt{Twitch-Gamers}).
We also use the dataset from ~\citep{lim2021large}, which includes \texttt{\justify reed98, amherst41, cornell5, Yelp} and \texttt{johnshopkins55}. 
In~\citep{platonov2023critical}, the authors proposed new benchmark datasets, \texttt{\justify Roman-empire, Amazon-ratings, Minesweeper, Tolokers}, and \texttt{Questions} to show the effect of homophilic and heterophilic GNNs.
We also converted a few multi-label multiclass classification problems (\texttt{\justify  Flickr, AmazonProducts}) to single-label multiclass node classification problems. Through this process,  their homophily values are reduced,   rendering them heterophilic.

\subsubsection{Dataset with Homophily:}
\texttt{\justify Cora~\citep{sen2008collective}, Citeseer}~\citep{giles1998citeseer} and \texttt{pubmed}~\citep{namata2012query} are three classic \emph{paper citation network} benchmarks. For \texttt{Cora} and \texttt{Citeseer}, class labels correspond to paper topics, while for \texttt{pubmed}, labels correspond to  the type of diabetes.  
\texttt{Coauthor-cs} and \texttt{Coauthor-physics}~\citep{shchur2018pitfalls} are coauthorship networks, where nodes represent authors and an edge connects two authors if they co-authored a paper. Node labels correspond to fields of study. The \texttt{Amazon-computers} and \texttt{Amazon-photo}~\citep{shchur2018pitfalls} are co-purchasing networks with nodes representing products and edges representing two products that are frequently bought together. Here labels correspond to product categories. \texttt{Reddit}~\citep{hamilton2017inductive} is one of the most popular large homophilic benchmark datasets. The graph contains Reddit posts belonging to different communities; the node label is the community, or subreddit, that a post belongs to. \texttt{Reddit2}~\citep{zeng2019graphsaint} is the sparse version of the original \texttt{Reddit} graph with average degree $100$ instead of the original $492$. 
The \texttt{dblp}~\citep{fu2020magnn} is from a subset of the DBLP computer science bibliography website.


\subsection{Homophily measures of Dataset}
\label{subsec:appendixhomophily_measure}
\begin{table*}[t]
\centering
\resizebox{0.8\linewidth}{!}{
\def\arraystretch{1.2}
\begin{tabular}{c|rrcc|ccc>{\columncolor[gray]{0.8}}cccccc}
\toprule
{Dataset} &
  {${N}$} &
  {${E}$} &
  {${d}$} &
  {${C}$} &
  {${\gH_{n}}$} &
  {${\gH_{e}}$} &
  {${\gH_{i}}$} &
  {${\gH_{a}}\downarrow$} &
  {${\gH_{agg}^{LP}}$} &
  {${\gH_{agg}^{HP}}$} &
  {${\gH_{t}}$} &
  {${\gH_{u}}$} &
  {${\gA}$} \\
  \midrule
Cornell        & 183     & 298         & 1.63   & 5   & 0.11 & 0.13 & 0.15  & -0.42 & 0.62  & 0.61  & 0.05  & 0.01  & -0.38 \\
Texas          & 183     & 325         & 1.78   & 5   & 0.07 & 0.11 & 0.00  & -0.26 & 0.55  & 0.56  & 0.09  & 0.00  & -0.35 \\
Wisconsin      & 251     & 515         & 2.05   & 5   & 0.17 & 0.20 & 0.08  & -0.20 & 0.50  & 0.51  & 0.06  & 0.00  & -0.27 \\
reed98         & 962     & 37,624      & 39.11  & 2   & 0.45 & 0.45 & $-$ & -0.10 & 0.55  & 0.54  & $-$ & $-$ & 0.02  \\
amherst41      & 2,235   & 181,908     & 81.39  & 2   & 0.47 & 0.46 & $-$ & -0.07 & 0.50  & 0.58  & $-$ & $-$ & 0.06  \\
penn94         & 41,554  & 2,724,458   & 65.56  & 2   & 0.48 & 0.47 & $-$ & -0.06 & $-$ & $-$ & $-$ & $-$ & 0.00  \\
Roman-empire   & 22,662  & 65,854      & 2.91   & 18  & 0.05 & 0.05 & 0.02  & -0.05 & $-$ & $-$ & 0.31  & 0.00  & -0.03 \\
cornell5       & 18,660  & 1,581,554   & 84.76  & 2   & 0.48 & 0.48 & $-$ & -0.04 & $-$ & $-$ & $-$ & $-$ & 0.02  \\
Squirrel       & 5,201   & 217,073     & 41.74  & 5   & 0.09 & 0.22 & 0.04  & -0.01 & 0.43  & 0.55  & 0.68  & 0.28  & 0.37  \\
johnshopkins55 & 5,180   & 373,172     & 72.04  & 2   & 0.50 & 0.50 & $-$ & 0.00  & 0.49  & 0.56  & $-$ & $-$ & 0.08  \\
AmazonProducts &
  1,569,960 &
  264,339,468 &
  168.37 &
  107 &
  0.48 &
  0.09 &
  0.01 &
  0.01 &
  $-$ &
  $-$ &
  $-$ &
  $-$ &
  -0.03 \\
Actor          & 7,600   & 30,019      & 3.95   & 5   & 0.16 & 0.22 & 0.01  & 0.01  & 0.36  & 0.36  & 0.34  & 0.11  & -0.11 \\
Minesweeper    & 10,000  & 78,804      & 7.88   & 2   & 0.68 & 0.68 & 0.01  & 0.01  & 0.58  & 0.59  & $-$ & 0.86  & 0.39  \\
Questions      & 48,921  & 307,080     & 6.28   & 2   & 0.90 & 0.84 & 0.08  & 0.02  & $-$ & $-$ & $-$ & 0.23  & -0.15 \\
Chameleon      & 2,277   & 36,101      & 15.85  & 5   & 0.10 & 0.24 & 0.06  & 0.03  & 0.51  & 0.53  & 0.55  & 0.17  & -0.11 \\
Tolokers       & 11,758  & 1,038,000   & 88.28  & 2   & 0.63 & 0.59 & 0.18  & 0.09  & $-$ & $-$ & $-$ & 0.92  & -0.08 \\
Flickr         & 89,250  & 899,756     & 10.08  & 7   & 0.32 & 0.32 & 0.07  & 0.09  & $-$ & $-$ & 0.47  & 0.01  & -0.04 \\
Yelp           & 716,847 & 13,954,819  & 19.47  & 100 & 0.52 & 0.44 & 0.02  & 0.12  & $-$ & $-$ & $-$ & $-$ & 0.10  \\
Amazon-ratings & 24,492  & 186,100     & 7.60   & 5   & 0.38 & 0.38 & 0.13  & 0.14  & $-$ & $-$ & 0.37  & 0.03  & -0.09 \\
genius         & 421,961 & 984,979     & 2.33   & 2   & 0.48 & 0.62 & 0.22  & 0.17  & $-$ & $-$ & $-$ & 0.03  & -0.11 \\
\hline
cora           & 19,793  & 126,842     & 6.41   & 70  & 0.59 & 0.57 & 0.50  & 0.56  & $-$ & $-$ & 0.09  & 0.00  & -0.05 \\
CiteSeer       & 3,327   & 9,104       & 2.74   & 6   & 0.71 & 0.74 & 0.63  & 0.67  & 0.72  & 0.73  & 0.03  & 0.00  & 0.05  \\
dblp           & 17,716  & 105,734     & 5.97   & 4   & 0.81 & 0.83 & 0.65  & 0.68  & $-$ & $-$ & 0.05  & 0.01  & -0.01 \\
Computers      & 13,752  & 491,722     & 35.76  & 10  & 0.79 & 0.78 & 0.70  & 0.68  & $-$ & $-$ & 0.15  & 0.00  & -0.06 \\
pubmed         & 19,717  & 88,648      & 4.50   & 3   & 0.79 & 0.80 & 0.66  & 0.69  & $-$ & $-$ & 0.04  & 0.05  & -0.04 \\
Reddit         & 232,965 & 114,615,892 & 491.99 & 41  & 0.81 & 0.76 & 0.65  & 0.74  & $-$ & $-$ & 0.28  & 0.00  & 0.11  \\
cora\_ml       & 2,995   & 16,316      & 5.45   & 7   & 0.81 & 0.79 & 0.74  & 0.75  & 0.84  & 0.78  & 0.06  & 0.00  & -0.07 \\
Cora           & 2,708   & 10,556      & 3.90   & 7   & 0.83 & 0.81 & 0.77  & 0.77  & 0.85  & 0.81  & 0.04  & 0.00  & -0.07 \\
Reddit2        & 232,965 & 23,213,838  & 99.65  & 41  & 0.81 & 0.78 & 0.69  & 0.77  & $-$ & $-$ & 0.21  & 0.00  & 0.10  \\
CS             & 18,333  & 163,788     & 8.93   & 15  & 0.83 & 0.81 & 0.75  & 0.78  & $-$ & $-$ & 0.07  & 0.00  & 0.11  \\
Photo          & 7,650   & 238,162     & 31.13  & 8   & 0.84 & 0.83 & 0.77  & 0.79  & 0.77  & 0.81  & 0.13  & 0.00  & -0.04 \\
Physics        & 34,493  & 495,924     & 14.38  & 5   & 0.92 & 0.93 & 0.85  & 0.87  & $-$ & $-$ & 0.04  & 0.00  & 0.20  \\
citeseer       & 4,230   & 10,674      & 2.52   & 6   & 0.96 & 0.95 & 0.94  & 0.94  & 0.96  & 0.95  & 0.00  & 0.00  & -0.08\\
\bottomrule
\end{tabular}
}
\caption{Properties of all datasets used in the experiments and their different homophily measure. The datasets are sorted (in ascending order) based on their adjusted homophily, $h_{adj}$. The '$-$' represents the value that is not computed due to memory complexity or an error in computation. The datasets are termed heterophily and homophily based on the $h_{adj}$ cutoff $0.5$}
\label{tab:dataset_description}
\end{table*}

Here we provide more homophily measures than included in the paper, and provide additional details. 

\begin{enumerate}
    
\item {\em Edge homophily}~\citep{zhu2020beyond}:  is defined as the fraction of edges in a graph that join nodes that have the same class label 
\begin{equation}
\gH_{e} = \frac{|(u,v):(u,v)\in \gE  \land y_u = y_v|}{|\gE|}.    
\end{equation}
Recall that $y_u, y_v$ represent labels of node $u$ and $v$, respectively.

\item {\em Node homophily}~\citep{pei2020geom}: 
\begin{equation}
\gH_{n} = \frac{1}{|\gV|}\sum_{u\in V}\frac{|(u,v): v\in N(u)  \land y_u = y_v|}{|\gN(u)|}.
\end{equation}
Here $\gN(u)$ represents neighbors of node $u$.

\item {\em Class insensitive edge homophily}~\citep{lim2021large}: Edge homophily is modified to be insensitive to the number of classes and size of each class as follows. 
\begin{equation}
\gH_{i} = \frac{1}{c-1}\sum_{k=1}^{c} max(0,\gH_k-\frac{|c_k|}{|V|}), 
\end{equation}
where $c$ denotes the number of classes, $|c_k|$ denotes the number of nodes of class $k$, and $\gH_k$ denotes the edge homophily ratio of nodes of class $k$.

\end{enumerate}

These homophily score range between $[0,1]$.
The definitions of homophily scores given above  may not represent heterophilic graphs well; hence we consider a few more. In addition, we define two more intuitive definitions to understand the distribution of labels in the neighborhood of a vertex.

\begin{enumerate}
\item {\em Aggregated Homophily}~\citep{luan2022revisiting}: is the only measure that considers associated node features. It is defined as
\begin{align}
\centering
\gH_{agg} &= [2S_{agg}(S(\mA, \mX))-1]_+\\ 
S = S(\mA, \mX) &= (\mA\mX) \times (\mA\mX)^T\\
S_{agg}(S(\mA, \mX)) &= 
\frac{1}{|V|}\{v|\mu_{u}(S_{v,u}|y_u = y_v)\ge \mu_{u}(S_{v,u}|y_u \ne y_v).\}
\end{align}
Here 
$S(\mA, \mX)$ is an $N\times N$ matrix, and the method computes the average similarity weights for edges that share the same label for each row. Moreover, it counts the expected number of correctly classified nodes with a single-layer GNN.

\item {\em Adjusted Homophily}~\citep{platonov2022characterizing} is defined as 
\begin{equation}
\gH_{a} = \frac{\gH_{e}-\sum_{k=1}^{C} D_k^2/(2|\gE|^2)}{1-\sum_{k=1}^C D_k^2/2|\gE|^2}; 
\end{equation}
here $d_v$ denotes the degree of node $v$ and  $D_k = \sum_{v:y_v=k}d_v$.  

\item {\em Entropy Score}: A vertex may have edges connecting it to clusters with labels that differ from its own, and we evaluate if these edges are connected to specific clusters or distributed across all cluster labels. Entropy is an excellent measure for this, and it is defined as
\begin{equation}
\gH_{t} = \sum_{u\in V}\mathrm{Entropy}(u)/ |\gV|.
\end{equation}

\item {\em Uniformity Score}: Another metric with the same inspiration comes from checking the distribution of neighboring node labels of the vertex.
\begin{equation} 
\gH_{u} = \sum_{u\in V}\mathrm{Uniformity(u)}/ |\gV|, 
\end{equation}
where $\mathrm{Uniformity(u)}$ does the Chi-square test with expected uniform distribution with $95\%$ confidence. 

\end{enumerate}

Table~\ref{tab:dataset_description} shows different homophily measures of all the datasets used for experimentation. Here, $\gA$ refers to the degree assortativity coefficient of the graph.
Assortivity measures how similar the nodes are in terms of some degree. It is calculated as the correlation coefficient between the degree values for pairs of connected nodes. A positive assortivity means that nodes tend to connect with other nodes that have similar degrees. In contrast, a negative assortivity means that nodes tend to connect with other nodes with different degrees.

\begin{equation}
    \gA = \frac{\sum_{ij} (e_{ij} - p_i q_j)}{\sqrt{(\sum_i p_i^2 - (\sum_i p_i)^2) (\sum_j q_j^2 - (\sum_j q_j)^2)}}. 
\end{equation}
Here $e_{ij}$ is the fraction of edges in the network that join nodes of degrees $i$ and $j$, $p_i$ is the fraction of edges connecting nodes of degree $i$ to each other, and  $q_j$ is the fraction of edges that connect nodes of degree $j$ to each other.


\subsection{Additional Dataset Properties}
\label{subsec:additionalproperties}

\begin{table}[!htbp]
\centering
\resizebox{1.0\linewidth}{!}
{
\def\arraystretch{1.2}
\begin{tabular}{c|rcccll}
\toprule
{Dataset} &
  \multicolumn{1}{c}{{Feat.}} &
  {Isolated} &
  {Self-loop} &
  {Directed} &
  \multicolumn{1}{c}{{Tr./Va./Te.}} &
  \multicolumn{1}{l}{{Context}} \\
  \midrule
Cornell        & 1,703 & No  & Yes & No  & 0.48/0.32/0.20 & Web Pages        \\
Texas          & 1,703 & No  & Yes & No  & 0.48/0.32/0.20 & Web Pages        \\
Wisconsin      & 1,703 & No  & Yes & No  & 0.48/0.32/0.20 & Web Pages        \\
reed98         & 745   & No  & No  & Yes & 0.60/0.20/0.20 & Social Network                 \\
amherst41      & 1,193 & No  & No  & Yes & 0.60/0.20/0.20 & Social Network                 \\
penn94         & 4,814 & No  & No  & Yes & 0.47/0.23/0.23 & Social Network                 \\
Roman-empire   & 300   & No  & No  & Yes & 0.50/0.25/0.25 & Wikipedia                 \\
cornell5       & 4,735 & No  & No  & Yes & 0.60/0.20/0.20 & Web pages                 \\
Squirrel       & 2,089 & No  & Yes & No  & 0.48/0.32/0.20 & Wikipedia       \\
johnshopkins55 & 2,406 & No  & No  & Yes & 0.60/0.20/0.20 & Web Pages                 \\
AmazonProducts & 200   & Yes & Yes & Yes & 0.80/0.05/0.15 & Reviews                 \\
Actor          & 932   & No  & Yes & No  & 0.48/0.32/0.20 & Actors in Movies \\
Minesweeper    & 7     & No  & No  & Yes & 0.50/0.25/0.25 & Synthetic                 \\
Questions      & 301   & No  & No  & Yes & 0.50/0.25/0.25 & Yandex Q                 \\
Chameleon      & 2,325 & No  & Yes & No  & 0.48/0.32/0.20 & Wiki Pages       \\
Tolokers       & 10    & No  & No  & Yes & 0.50/0.25/0.25 & Toloka Platform                 \\
Flickr         & 500   & No  & No  & Yes & 0.50/0.25/0.25 & Social network                \\
Yelp           & 300   & No  & Yes & Yes & 0.75/0.15/0.10 & Review                 \\
Amazon-ratings & 300   & No  & No  & Yes & 0.50/0.25/0.25 & Co-purchase network                 \\
genius         & 12    & Yes & No  & No  & 0.60/0.20/0.20 & Social Network                 \\
cora           & 8,710 & No  & No  & Yes & 0.60/0.20/0.20 & Citation  Network               \\
CiteSeer       & 3,703 & Yes & No  & Yes & 0.04/0.15/0.30 & Citation  Network               \\
dblp           & 1,639 & No  & No  & Yes & 0.60/0.20/0.20 & Citation  Network               \\
Computers      & 767   & Yes & No  & Yes & 0.60/0.20/0.20 & Co-purchase Network              \\
pubmed         & 500   & No  & No  & Yes & 0.60/0.20/0.20 & Social Network                 \\
Reddit         & 602   & No  & No  & Yes & 0.66/0.10/0.24 & Social Network                 \\
cora\_ml       & 2,879 & No  & No  & Yes & 0.60/0.20/0.20 & Citation Network                \\
Cora           & 1,433 & No  & No  & Yes & 0.05/0.18/0.37 & Citation Network                \\
Reddit2        & 602   & Yes & No  & Yes & 0.66/0.10/0.24 & Social Network                 \\
CS             & 6,805 & No  & No  & Yes & 0.60/0.20/0.20 & Co-author Network                 \\
Photo          & 745   & Yes & No  & Yes & 0.60/0.20/0.20 & Co-purchase Network                 \\
Physics        & 8,415 & No  & No  & Yes & 0.60/0.20/0.20 & Co-author Network                 \\
citeseer       & 602   & No  & No  & Yes & 0.60/0.20/0.20 & Citation  Network              \\
\bottomrule
\end{tabular}
}
\caption{Additional details of the dataset used for experiments}
\label{tab:additonaldataset}
\end{table}

Additional dataset properties
are included in Table~\ref{tab:additonaldataset}. The table includes feature size, whether isolated nodes exist, whether self-loops are present, and whether the graphs are directed or undirected. We also include train, test, validation split provided from the benchmark or our random split size and the context from which  the graphs arise.


\subsection{Feature Similarity versus Label Matching}
\label{subsec:featurevslabel}
\begin{assumption}
    The nodes with {\em similar features} tend to have a {\em similar labels}~\citep{van2020survey}. 
\end{assumption} 
Here, we provide evidence to support this assumption. 
Table~\ref{tab:featurevslabel} shows the {\em Pearson correlation} between feature similarity and labels for our benchmark dataset. The Pearson correlation between two vectors $\vx$ and $\vy$ with $n$ elements is,
\begin{equation}
r = \frac{\sum_{i=1}^{n} (x_i - \bar{x})(y_i - \bar{y})}{\sqrt{\sum_{i=1}^{n}(x_i - \bar{x})^2}\sqrt{\sum_{i=1}^{n}(y_i - \bar{y})^2}}. 
\end{equation}
Here $\bar{x}$ and $\bar{y}$ are the means  of the vectors $\vx$ and $\vy$, respectively. In Table~\ref{tab:featurevslabel}, we show a comparison between {\em Cosine Similarity}, {\em Euclidean distance}, and {\em Learned model} for similarity measure. 

\begin{table*}[!htbp]
\centering
\resizebox{1.0\linewidth}{!}
{
\def\arraystretch{1.0}

\begin{tabular}{c|ccc|ccc|cccc}
\hline
 &
   {Rand.} &
   {Rand.} &
   {Rand.} &
   {Balance.} &
   {Balance.} &
   {Balance.} &
   {Edge} &
   {Edge} &
   {Edge} &
   {Edge} \\
 {Dataset} &
   {Learned} &
   {Cosine Sim.} &
  \multicolumn{1}{l|}{ {Euclidean$\times-1$}} &
   {Learned} &
   {Cosine Sim.} &
  \multicolumn{1}{l|}{ {Euclidean$\times-1$}} &
   {Learned} &
   {Learned (Balanced)} &
  \multicolumn{1}{l}{ {Cosine Sim}} &
   {Euclidean$\times-1$} \\
  \hline
Cornell &
  \cellcolor[HTML]{FFC7CE}{\color[HTML]{9C0006} 0.71} &
  0.29 &
  0.19 &
  \cellcolor[HTML]{FFC7CE}{\color[HTML]{9C0006} 0.72} &
  0.29 &
  0.19 &
  0.44 &
  \cellcolor[HTML]{FFC7CE}{\color[HTML]{9C0006} 0.63} &
  0.30 &
  0.09 \\
Texas &
  \cellcolor[HTML]{FFC7CE}{\color[HTML]{9C0006} 0.73} &
  0.28 &
  0.21 &
  \cellcolor[HTML]{FFC7CE}{\color[HTML]{9C0006} 0.77} &
  0.28 &
  0.21 &
  0.50 &
  0.52 &
  \cellcolor[HTML]{FFC7CE}{\color[HTML]{9C0006} 0.58} &
  0.47 \\
Wisconsin &
  \cellcolor[HTML]{FFC7CE}{\color[HTML]{9C0006} 0.77} &
  0.45 &
  0.29 &
  \cellcolor[HTML]{FFC7CE}{\color[HTML]{9C0006} 0.77} &
  0.45 &
  0.29 &
  \cellcolor[HTML]{FFC7CE}{\color[HTML]{9C0006} 0.48} &
  0.47 &
  0.39 &
  0.31 \\
reed98 &
  \cellcolor[HTML]{FFC7CE}{\color[HTML]{9C0006} 0.38} &
  -0.02 &
  -0.01 &
  \cellcolor[HTML]{FFC7CE}{\color[HTML]{9C0006} 0.37} &
  -0.02 &
  -0.01 &
  0.36 &
  \cellcolor[HTML]{FFC7CE}{\color[HTML]{9C0006} 0.36} &
  0.04 &
  0.05 \\
amherst41 &
  \cellcolor[HTML]{FFC7CE}{\color[HTML]{9C0006} 0.44} &
  -0.02 &
  -0.01 &
  \cellcolor[HTML]{FFC7CE}{\color[HTML]{9C0006} 0.45} &
  -0.02 &
  -0.01 &
  0.44 &
  \cellcolor[HTML]{FFC7CE}{\color[HTML]{9C0006} 0.44} &
  0.01 &
  0.01 \\
penn94 &
  \cellcolor[HTML]{FFC7CE}{\color[HTML]{9C0006} 0.37} &
  -0.03 &
  -0.03 &
  \cellcolor[HTML]{FFC7CE}{\color[HTML]{9C0006} 0.24} &
  0.17 &
  0.20 &
  0.41 &
  \cellcolor[HTML]{FFC7CE}{\color[HTML]{9C0006} 0.41} &
  0.07 &
  0.07 \\
Roman-empire &
  \cellcolor[HTML]{FFC7CE}{\color[HTML]{9C0006} 0.60} &
  0.26 &
  0.06 &
  \cellcolor[HTML]{FFC7CE}{\color[HTML]{9C0006} 0.34} &
  0.24 &
  0.09 &
  \cellcolor[HTML]{FFC7CE}{\color[HTML]{9C0006} 0.29} &
  0.22 &
  0.10 &
  0.06 \\
cornell5 &
  \cellcolor[HTML]{FFC7CE}{\color[HTML]{9C0006} 0.44} &
  -0.02 &
  -0.01 &
  \cellcolor[HTML]{FFC7CE}{\color[HTML]{9C0006} 0.40} &
  0.15 &
  0.19 &
  \cellcolor[HTML]{FFC7CE}{\color[HTML]{9C0006} 0.44} &
  0.44 &
  0.05 &
  0.05 \\
Squirrel &
  \cellcolor[HTML]{FFC7CE}{\color[HTML]{9C0006} 0.16} &
  0.02 &
  0.02 &
  \cellcolor[HTML]{FFC7CE}{\color[HTML]{9C0006} 0.15} &
  0.02 &
  0.02 &
  \cellcolor[HTML]{FFC7CE}{\color[HTML]{9C0006} 0.19} &
  0.18 &
  0.05 &
  -0.05 \\
johnshopkins55 &
  \cellcolor[HTML]{FFC7CE}{\color[HTML]{9C0006} 0.42} &
  -0.02 &
  -0.02 &
  \cellcolor[HTML]{FFC7CE}{\color[HTML]{9C0006} 0.31} &
  0.18 &
  0.21 &
  0.41 &
  \cellcolor[HTML]{FFC7CE}{\color[HTML]{9C0006} 0.42} &
  0.00 &
  0.00 \\
AmazonProducts &
  \cellcolor[HTML]{FFC7CE}{\color[HTML]{9C0006} 0.29} &
  0.02 &
  -0.01 &
  0.01 &
  \cellcolor[HTML]{FFC7CE}{\color[HTML]{9C0006} 0.02} &
  0.00 &
  \cellcolor[HTML]{FFC7CE}{\color[HTML]{9C0006} 0.23} &
  -0.11 &
  0.12 &
  - \\
Actor &
  \cellcolor[HTML]{FFC7CE}{\color[HTML]{9C0006} 0.22} &
  0.00 &
  -0.05 &
  \cellcolor[HTML]{FFC7CE}{\color[HTML]{9C0006} 0.20} &
  0.01 &
  0.02 &
  \cellcolor[HTML]{FFC7CE}{\color[HTML]{9C0006} 0.22} &
  0.22 &
  0.02 &
  -0.05 \\
Minesweeper &
  \cellcolor[HTML]{FFC7CE}{\color[HTML]{9C0006} 0.03} &
  0.00 &
  -0.01 &
  \cellcolor[HTML]{FFC7CE}{\color[HTML]{9C0006} 0.00} &
  0.00 &
  0.00 &
  \cellcolor[HTML]{FFC7CE}{\color[HTML]{9C0006} 0.09} &
  0.08 &
  0.03 &
  0.03 \\
Questions &
  \cellcolor[HTML]{FFC7CE}{\color[HTML]{9C0006} 0.29} &
  0.00 &
  0.03 &
  0.01 &
  \cellcolor[HTML]{FFC7CE}{\color[HTML]{9C0006} 0.04} &
  0.02 &
  \cellcolor[HTML]{FFC7CE}{\color[HTML]{9C0006} 0.17} &
  0.17 &
  0.00 &
  -0.03 \\
Chameleon &
  \cellcolor[HTML]{FFC7CE}{\color[HTML]{9C0006} 0.43} &
  0.04 &
  0.05 &
  \cellcolor[HTML]{FFC7CE}{\color[HTML]{9C0006} 0.43} &
  0.04 &
  0.05 &
  \cellcolor[HTML]{FFC7CE}{\color[HTML]{9C0006} 0.31} &
  0.29 &
  0.09 &
  -0.04 \\
Tolokers &
  \cellcolor[HTML]{FFC7CE}{\color[HTML]{9C0006} 0.22} &
  -0.04 &
  -0.04 &
  0.05 &
  0.06 &
  0.06 &
  0.20 &
  \cellcolor[HTML]{FFC7CE}{\color[HTML]{9C0006} 0.20} &
  -0.03 &
  -0.03 \\
Flickr &
  \cellcolor[HTML]{FFC7CE}{\color[HTML]{9C0006} 0.07} &
  0.03 &
  0.02 &
  \cellcolor[HTML]{FFC7CE}{\color[HTML]{9C0006} 0.02} &
  0.02 &
  0.00 &
  0.04 &
  0.05 &
  \cellcolor[HTML]{FFC7CE}{\color[HTML]{9C0006} 0.05} &
  0.02 \\
Yelp &
  \cellcolor[HTML]{FFC7CE}{\color[HTML]{9C0006} 0.48} &
  0.08 &
  0.09 &
  \cellcolor[HTML]{FFC7CE}{\color[HTML]{9C0006} 0.10} &
  0.03 &
  0.05 &
  0.35 &
  \cellcolor[HTML]{FFC7CE}{\color[HTML]{9C0006} 0.35} &
  0.20 &
  0.28 \\
Amazon-ratings &
  \cellcolor[HTML]{FFC7CE}{\color[HTML]{9C0006} 0.05} &
  0.00 &
  0.00 &
  0.00 &
  \cellcolor[HTML]{FFC7CE}{\color[HTML]{9C0006} 0.00} &
  0.00 &
  \cellcolor[HTML]{FFC7CE}{\color[HTML]{9C0006} 0.05} &
  -0.03 &
  -0.03 &
  -0.03 \\
genius &
  0.00 &
  \cellcolor[HTML]{FFC7CE}{\color[HTML]{9C0006} 0.46} &
  -0.02 &
  0.01 &
  \cellcolor[HTML]{FFC7CE}{\color[HTML]{9C0006} 0.28} &
  0.01 &
  \cellcolor[HTML]{FFC7CE}{\color[HTML]{9C0006} 0.05} &
  0.05 &
  0.04 &
  -0.17 \\
cora &
  \cellcolor[HTML]{FFC7CE}{\color[HTML]{9C0006} 0.22} &
  0.10 &
  0.01 &
  0.04 &
  \cellcolor[HTML]{FFC7CE}{\color[HTML]{9C0006} 0.11} &
  0.01 &
  0.12 &
  -0.09 &
  \cellcolor[HTML]{FFC7CE}{\color[HTML]{9C0006} 0.16} &
  0.03 \\
CiteSeer &
  \cellcolor[HTML]{FFC7CE}{\color[HTML]{9C0006} 0.24} &
  0.18 &
  0.04 &
  \cellcolor[HTML]{FFC7CE}{\color[HTML]{9C0006} 0.21} &
  0.18 &
  0.04 &
  0.15 &
  0.11 &
  \cellcolor[HTML]{FFC7CE}{\color[HTML]{9C0006} 0.15} &
  0.10 \\
dblp &
  \cellcolor[HTML]{FFC7CE}{\color[HTML]{9C0006} 0.50} &
  0.10 &
  -0.01 &
  \cellcolor[HTML]{FFC7CE}{\color[HTML]{9C0006} 0.48} &
  0.11 &
  0.02 &
  \cellcolor[HTML]{FFC7CE}{\color[HTML]{9C0006} 0.40} &
  0.33 &
  0.01 &
  -0.06 \\
Computers &
  \cellcolor[HTML]{FFC7CE}{\color[HTML]{9C0006} 0.68} &
  -0.04 &
  0.09 &
  \cellcolor[HTML]{FFC7CE}{\color[HTML]{9C0006} 0.48} &
  0.05 &
  0.05 &
  0.37 &
  \cellcolor[HTML]{FFC7CE}{\color[HTML]{9C0006} 0.41} &
  -0.10 &
  0.06 \\
pubmed &
  \cellcolor[HTML]{FFC7CE}{\color[HTML]{9C0006} 0.47} &
  0.17 &
  0.08 &
  \cellcolor[HTML]{FFC7CE}{\color[HTML]{9C0006} 0.61} &
  0.20 &
  0.05 &
  0.28 &
  \cellcolor[HTML]{FFC7CE}{\color[HTML]{9C0006} 0.29} &
  0.12 &
  0.08 \\
Reddit &
  \cellcolor[HTML]{FFC7CE}{\color[HTML]{9C0006} 0.48} &
  0.18 &
  0.02 &
  \cellcolor[HTML]{FFC7CE}{\color[HTML]{9C0006} 0.23} &
  0.14 &
  0.02 &
  \cellcolor[HTML]{FFC7CE}{\color[HTML]{9C0006} 0.21} &
  0.15 &
  0.13 &
  0.02 \\
cora\_ml &
  \cellcolor[HTML]{FFC7CE}{\color[HTML]{9C0006} 0.66} &
  0.13 &
  0.00 &
  \cellcolor[HTML]{FFC7CE}{\color[HTML]{9C0006} 0.65} &
  0.13 &
  0.00 &
  0.26 &
  \cellcolor[HTML]{FFC7CE}{\color[HTML]{9C0006} 0.59} &
  0.12 &
  -0.04 \\
Cora &
  \cellcolor[HTML]{FFC7CE}{\color[HTML]{9C0006} 0.23} &
  0.14 &
  0.00 &
  \cellcolor[HTML]{FFC7CE}{\color[HTML]{9C0006} 0.24} &
  0.14 &
  0.00 &
  0.14 &
  \cellcolor[HTML]{FFC7CE}{\color[HTML]{9C0006} 0.14} &
  0.13 &
  0.03 \\
Reddit2 &
  \cellcolor[HTML]{FFC7CE}{\color[HTML]{9C0006} 0.17} &
  0.07 &
  0.03 &
  0.01 &
  \cellcolor[HTML]{FFC7CE}{\color[HTML]{9C0006} 0.04} &
  0.00 &
  -0.09 &
  0.01 &
  0.09 &
  \cellcolor[HTML]{FFC7CE}{\color[HTML]{9C0006} 0.37} \\
CS &
  \cellcolor[HTML]{FFC7CE}{\color[HTML]{9C0006} 0.86} &
  0.57 &
  0.01 &
  \cellcolor[HTML]{FFC7CE}{\color[HTML]{9C0006} 0.54} &
  0.45 &
  0.05 &
  0.41 &
  \cellcolor[HTML]{FFC7CE}{\color[HTML]{9C0006} 0.51} &
  0.29 &
  0.20 \\
Photo &
  \cellcolor[HTML]{FFC7CE}{\color[HTML]{9C0006} 0.63} &
  0.08 &
  0.07 &
  \cellcolor[HTML]{FFC7CE}{\color[HTML]{9C0006} 0.66} &
  0.08 &
  0.06 &
  0.28 &
  \cellcolor[HTML]{FFC7CE}{\color[HTML]{9C0006} 0.51} &
  -0.15 &
  0.12 \\
Physics &
  \cellcolor[HTML]{FFC7CE}{\color[HTML]{9C0006} 0.81} &
  0.48 &
  0.10 &
  \cellcolor[HTML]{FFC7CE}{\color[HTML]{9C0006} 0.84} &
  0.44 &
  0.06 &
  0.52 &
  \cellcolor[HTML]{FFC7CE}{\color[HTML]{9C0006} 0.53} &
  0.17 &
  0.14 \\
citeseer &
  \cellcolor[HTML]{FFC7CE}{\color[HTML]{9C0006} 0.69} &
  0.27 &
  0.06 &
  \cellcolor[HTML]{FFC7CE}{\color[HTML]{9C0006} 0.66} &
  0.27 &
  0.06 &
  0.26 &
  \cellcolor[HTML]{FFC7CE}{\color[HTML]{9C0006} 0.29} &
  0.09 &
  0.02\\\hline
\end{tabular}
}
\caption{Pearson correlation coefficient between feature similarity and label matching. The table shows results in three categories: 1) evaluation done on nodes taken randomly, the model is also trained with random training node samples; 2) training and evaluation done where nodes are taken from each class of equal size; 3) evaluation done only to the given edges, the learned model is trained both random sampling and balanced sampling (from the training nodes).}
\label{tab:featurevslabel}
\end{table*}

To compute the similarity between two node features, we use cosine similarity and Euclidean distances. We also train a similarity function learning regression model that learns from training nodes and evaluates on test nodes. Here is a brief explanation of each column in the table.

The three categories are based on how we evaluated or computed the Pearson correlations of the nodes.

\begin{enumerate}
    \item  In the first category, we take random sets of test nodes and consider all pair of edges to compute the Pearson correlation between them. 

\begin{enumerate}
    \item \textit{Rand Learned:} We train a regression model to learn similarity functions based on training subgraphs. Then, we evaluate $r$ on random test nodes.
    
    \item \textit{Rand Cosine Sim.} We consider the same random test nodes and use cosine similarity measures.
    
    \item \textit{Rand Euclidean $\times-1$}: We consider the negative Euclidean distance of two node features since the distance is minimum when similarity is maximum.
\end{enumerate}

\item  The problem with the previous type of evaluation is that the number of edges with non-matching labels at the endpoints will be high due to the consideration of all pairs, which introduces a bias in the  Pearson correlation computation.
To handle that, we evaluated on random sets of test nodes but with an equal number of edges and non-edges.

\begin{enumerate}
    \item \textit{Balance. Learned:} The regression model is trained on a training subgraph with an equal number of positive and negative edges.
    
    \item \textit{Balance. Cosine Sim:} The Pearson correlation with cosine similarity is evaluated on the same number of equal edges and non-edges.
    
    \item \textit{Balance. Euclidean$\times-1$:} The same evaluation is done with the negative of the Euclidean distance.
\end{enumerate}

\item The previous two categories were evaluated on random sets of training nodes. In this category, we computed Pearson correlation only on the given edges for cosine similarity and Euclidean distance. For learned models, we evaluated on test subgraph edges only.

\begin{enumerate}
    \item \textit{Edge Learned:} We trained a regression model with training edges and set a target with equal endpoint labels as one and zero otherwise.
    
    \item \textit{Edge Learned Balanced:} We trained the regression model with equal training edges and non-edges from the training subgraph.
    
    \item \textit{Edge Cosine Sim.:} We evaluated only on given edges using cosine similarity.
    
    \item \textit{Edge Euclidean$\times-1$:}  We evaluated only on given edges using the negative of the Euclidean distance.
\end{enumerate}

\end{enumerate}

The table shows that the feature similarity and labels are positively correlated for all datasets. We specifically have a higher correlation value on the trained model. It highlights the importance of training a model for computing  similar and diverse neighborhoods of a subgraph in the pre-computation step.
This experiment also validates the importance of our supervised sampling strategy.

\FloatBarrier
\subsection{Local Node Homophily Distribution of Dataset}
\label{subsec:distribution}

Figs.~\ref{fig:homophily_distributions_1} and \ref{fig:homophily_distributions_2} show a histogram of \emph{local node homophily, $\gH_n(u)$} values of all vertices of a graph. 
For example, Fig.~\ref{subfig:homophily_Cornell} shows that most of the nodes of \texttt{Cornell} graphs have local node homophily $0$ and only a few have $1$, meaning the graph is strongly heterophilic with the majority of nodes having no neighbors of the same label as ego node. 
In contrast, \texttt{Reddit} (in Fig.~\ref{subfig:homophily_Reddit}) has the most nodes with high homophily. 
We can use a single channel \ags{} based on a feature-similarity sampler for graphs like \texttt{Reddit}, where we have only a few heterophilic nodes. 
In contrast, when a graph consists of both local homophilic and heterophilic (as in \texttt{Amherst41} in Fig.~\ref{subfig:homophily_amherst41}), we need both feature-similarity and feature-diversity samplers.
Therefore, the histogram plots of graphs give us a good insight into the distribution of local homophily. It explains when and why dual channel samplers work for homophilic and heterophilic graphs.


\begin{figure*}[htbp]
\centering

\newcommand{\datasetfig}[1]{
    \subfloat[#1]{
        \includegraphics[width=0.22\linewidth]{Figures/HomophilyDist/homophily_#1.pdf}
        \label{subfig:homophily_#1}
    }
}

\begin{tabular}{cccccccccccccccc}
    \datasetfig{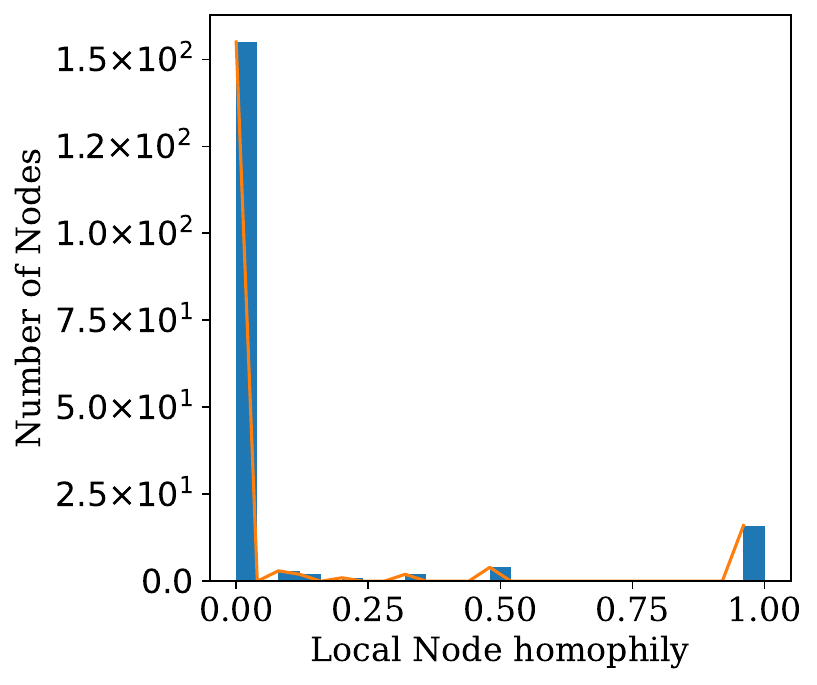} &
    \datasetfig{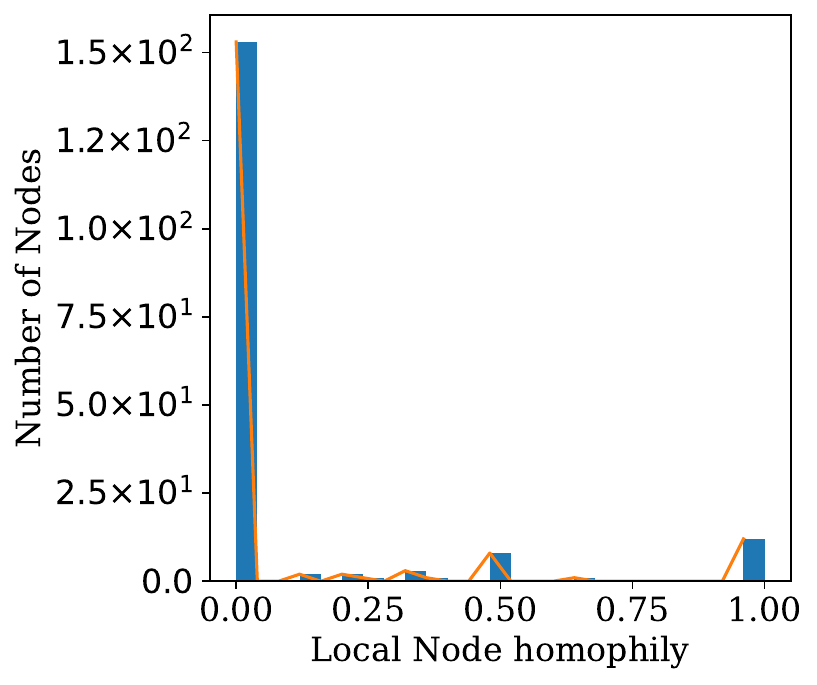} &
    \datasetfig{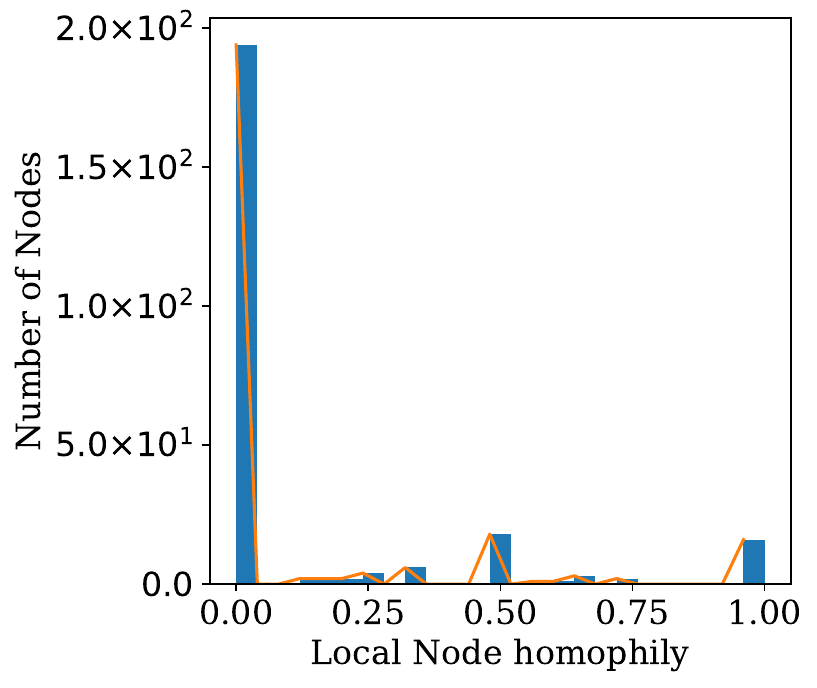} &
    \datasetfig{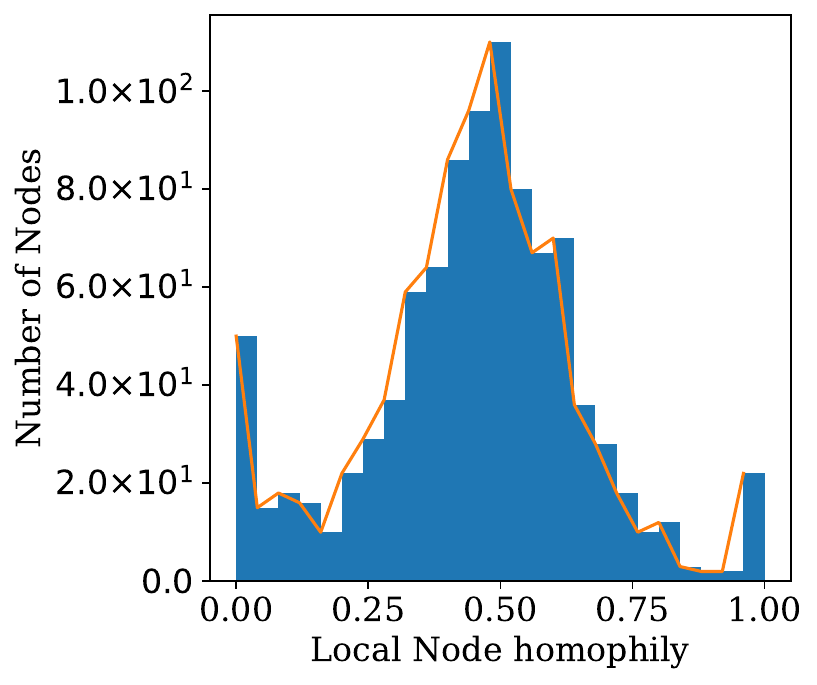} \\
    
    \datasetfig{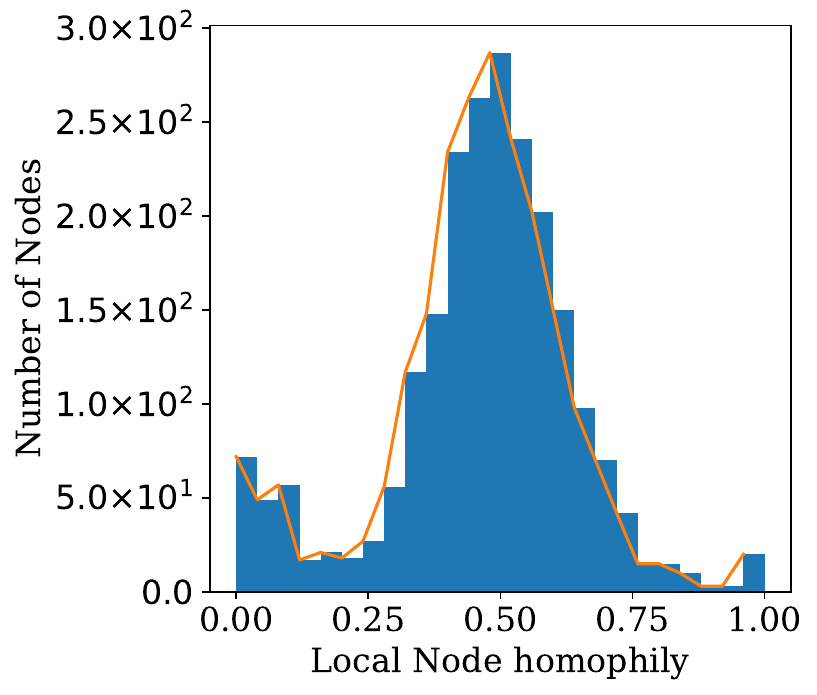} &
    \datasetfig{penn94} &
    \datasetfig{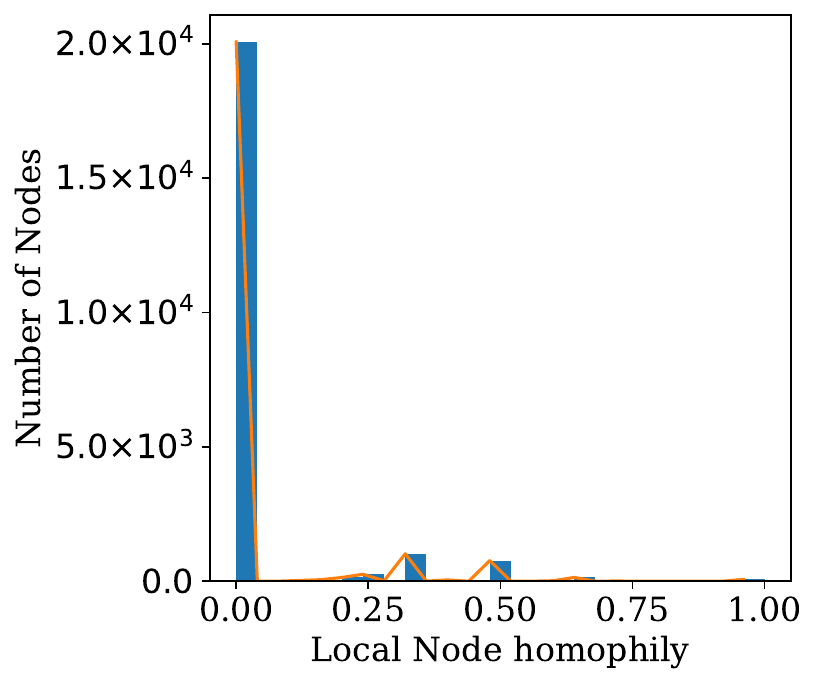} &
    \datasetfig{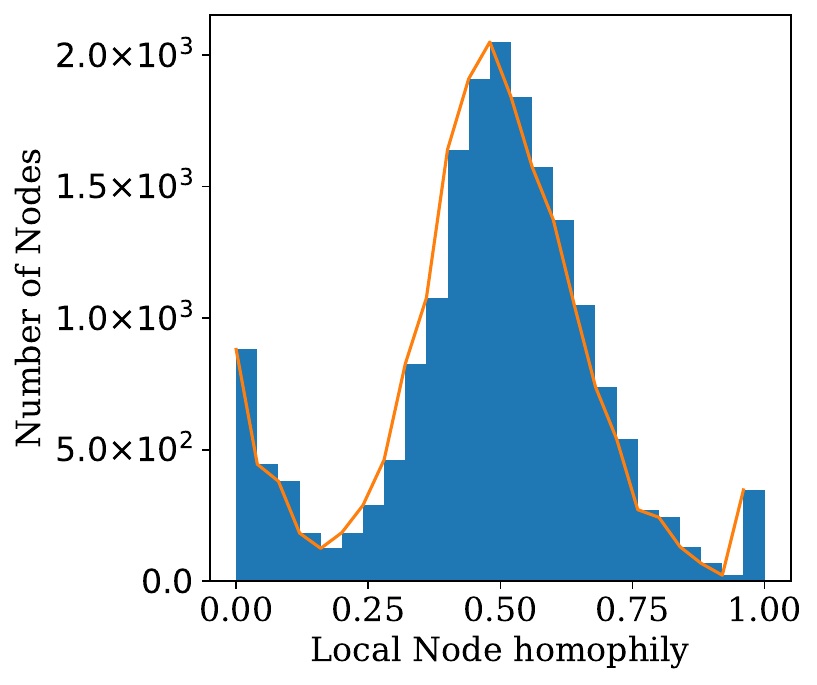} \\

    \datasetfig{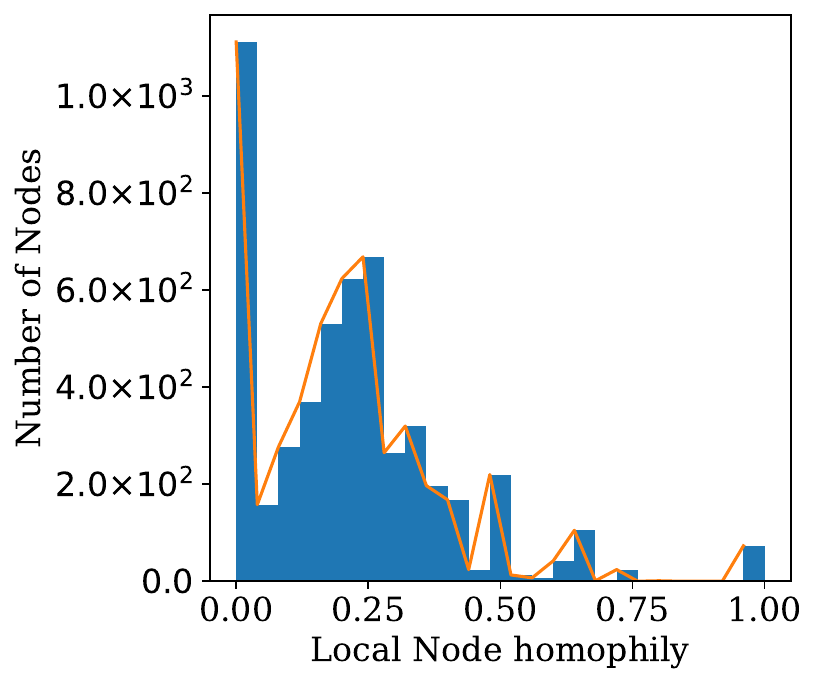} &
    \datasetfig{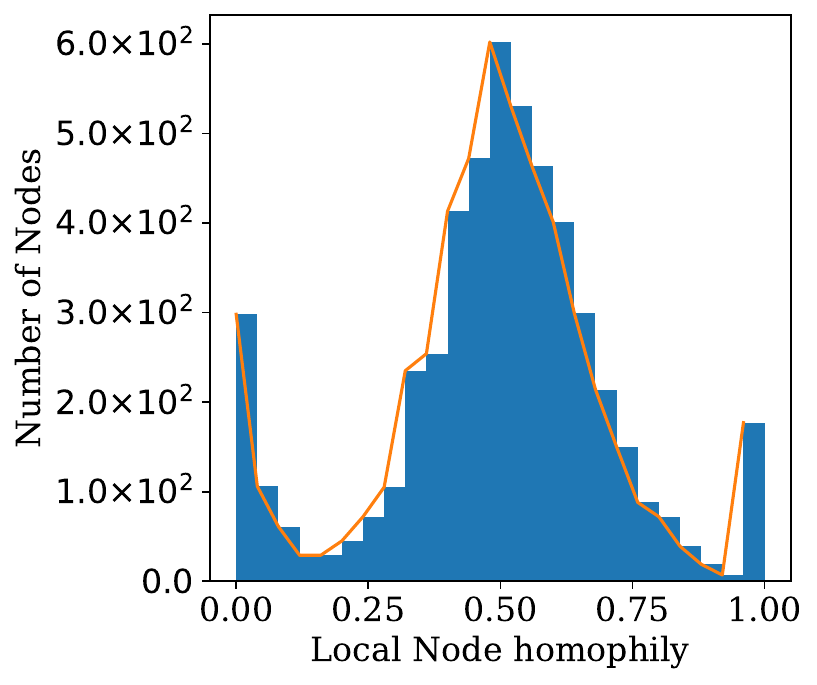} &
    \datasetfig{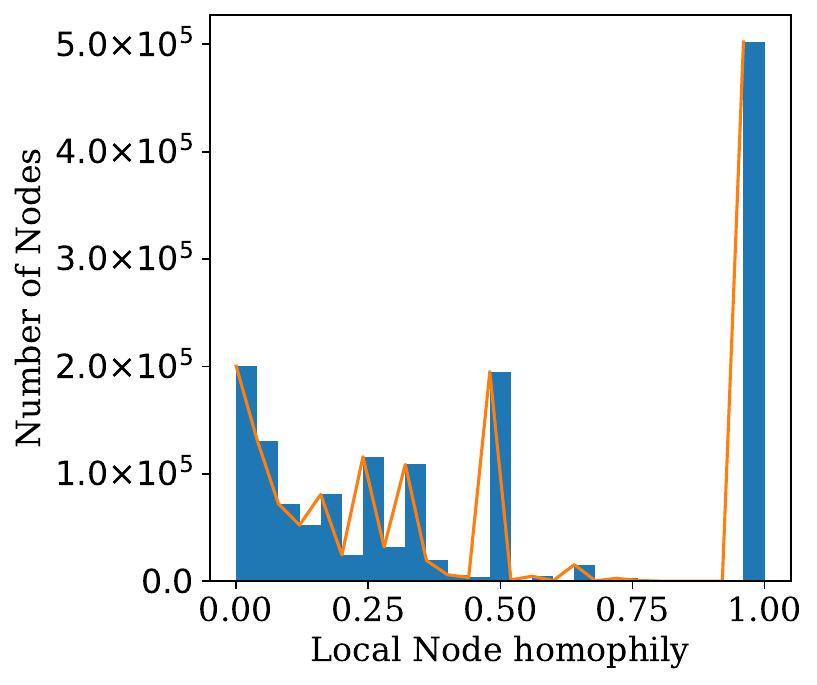} &
    \datasetfig{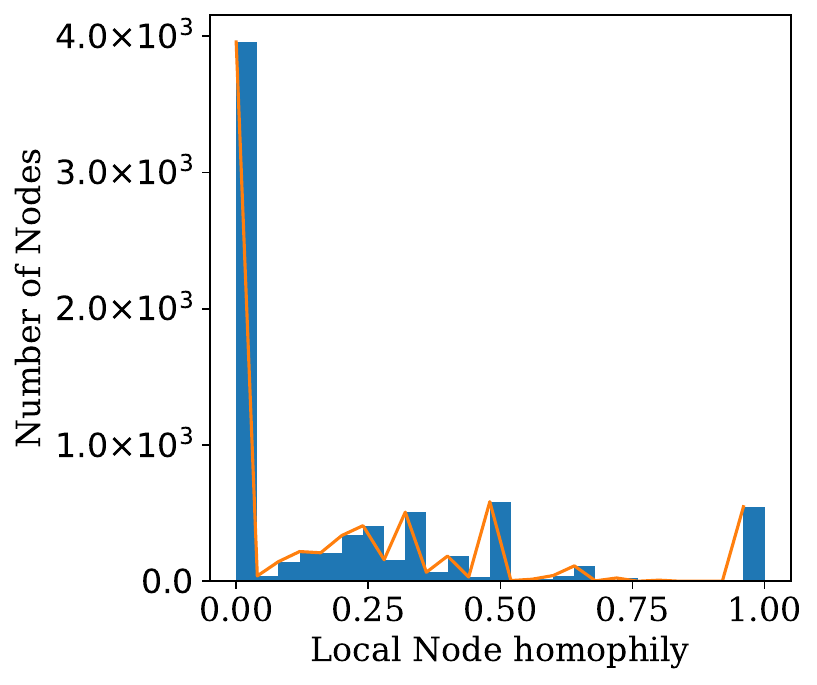} \\
    
    \datasetfig{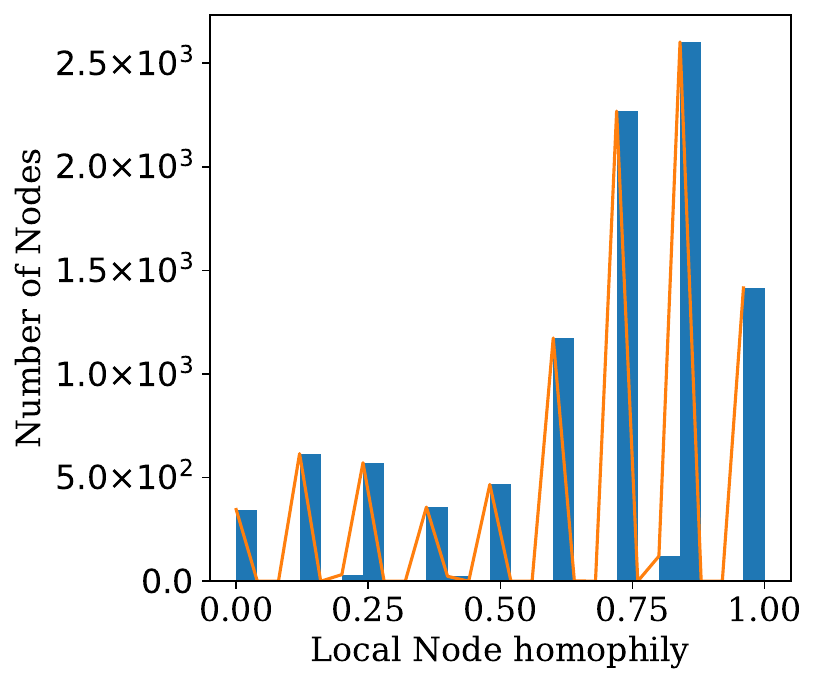} &
    \datasetfig{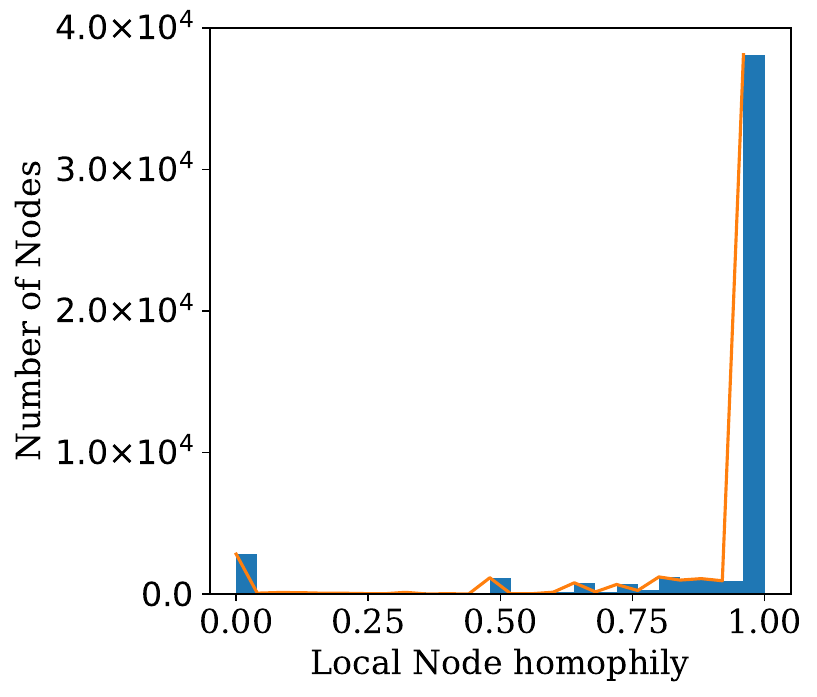} &
    \datasetfig{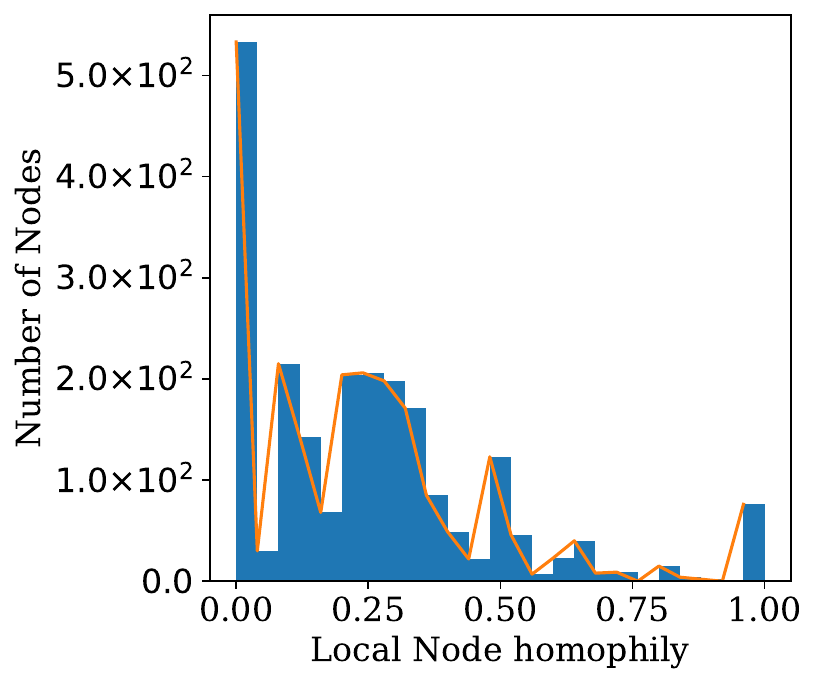} &
    \datasetfig{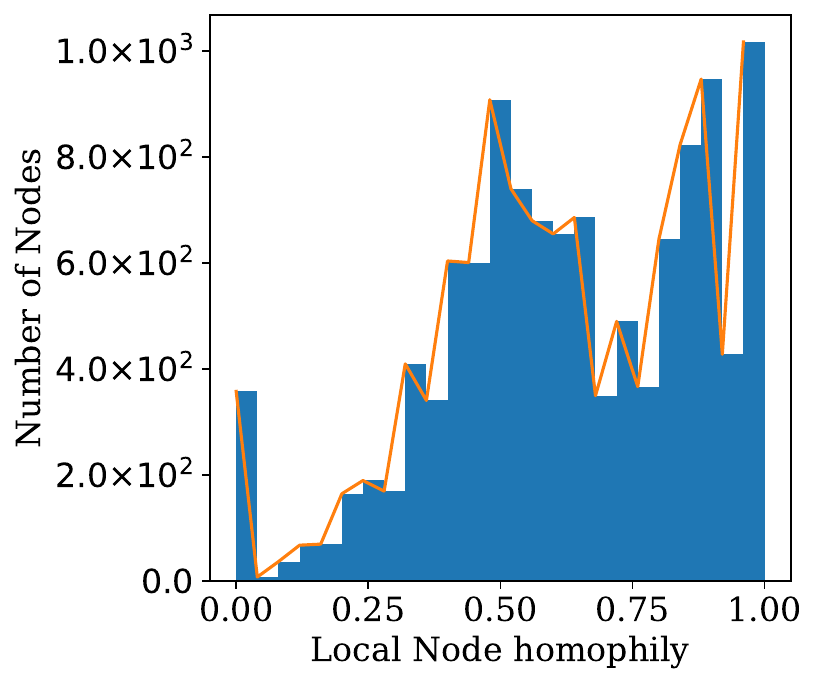} \\

    \datasetfig{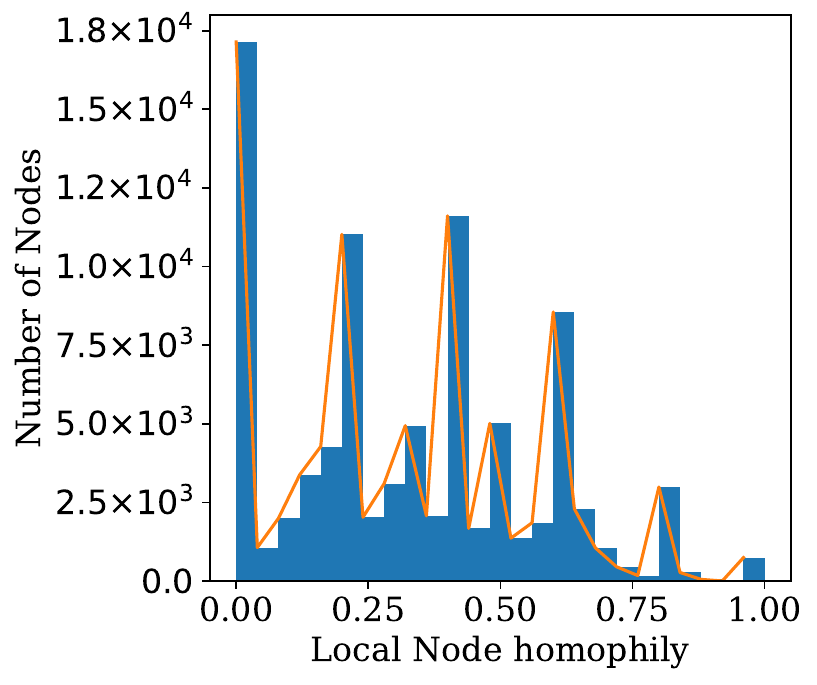} &
    \datasetfig{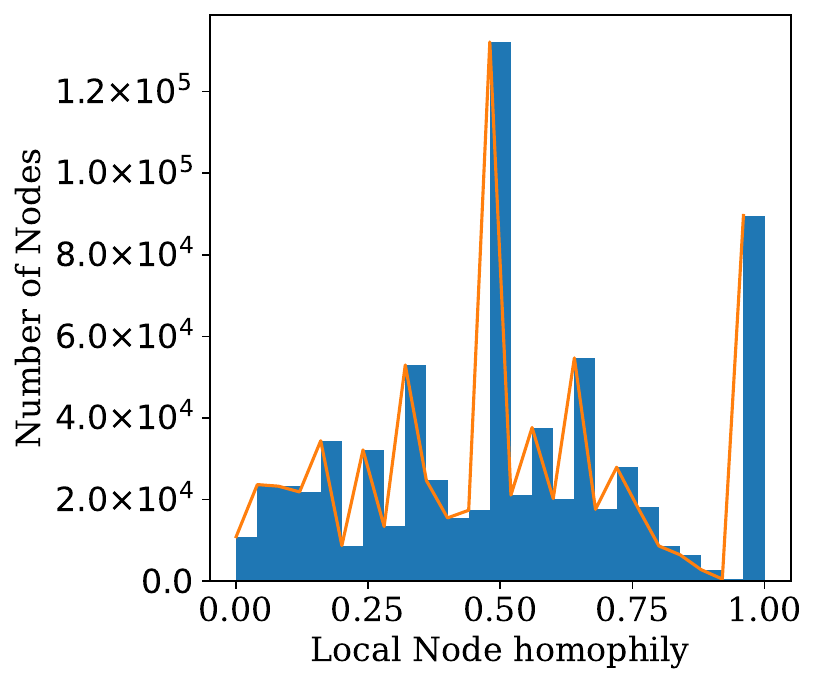} &
    \datasetfig{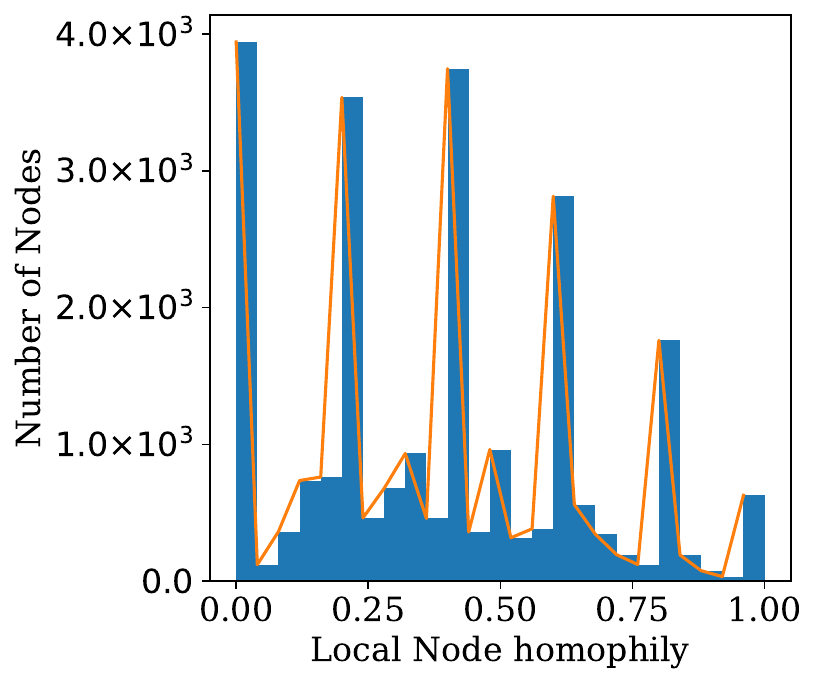} &
    \datasetfig{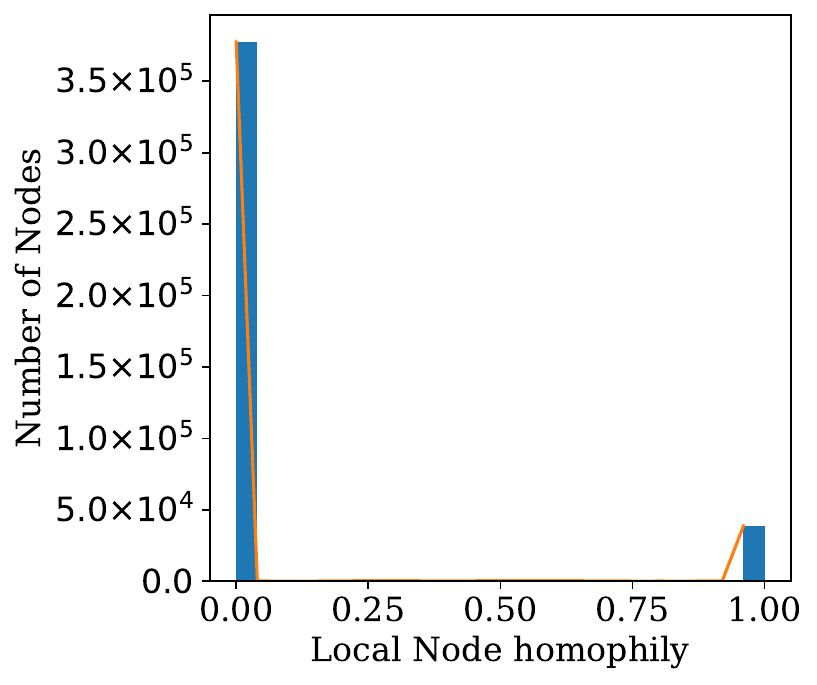}
\end{tabular}

\caption{\emph{Local Node Homophily} distributions of our benchmark datasets (Part 1)}
\label{fig:homophily_distributions_1}
\end{figure*}

\begin{figure*}[htbp]
\centering

\newcommand{\datasetfig}[2]{%
    \subfloat[#1]{%
        \includegraphics[width=0.22\linewidth]{Figures/HomophilyDist/homophily_#2.pdf}%
        \label{subfig:homophily_#2}%
    }%
}

\begin{tabular}{cccccccccccccccc}
    
    \datasetfig{cora}{cora} &
    \datasetfig{CiteSeer}{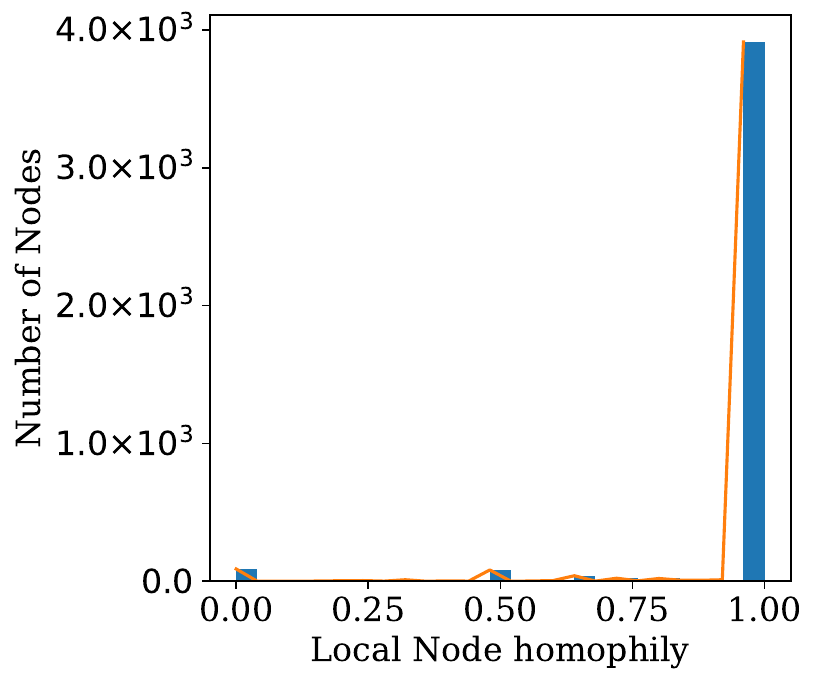} &
    \datasetfig{dblp}{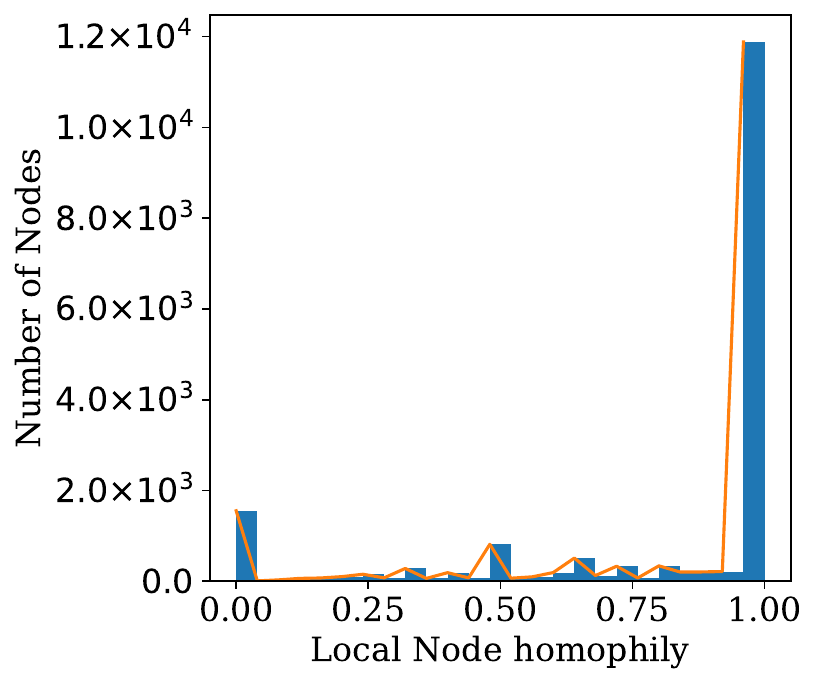} &
    \datasetfig{Computers}{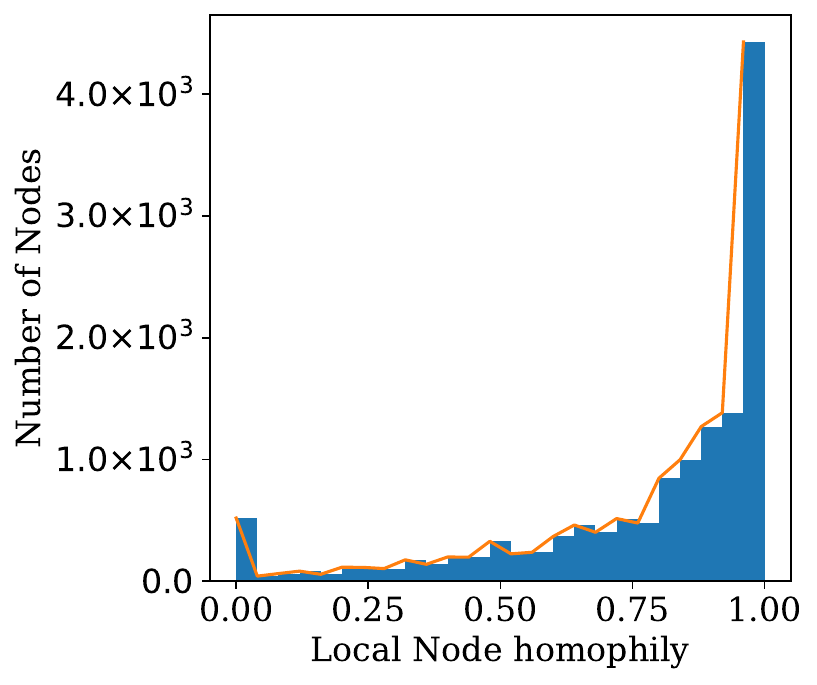} \\
    
    \datasetfig{pubmed}{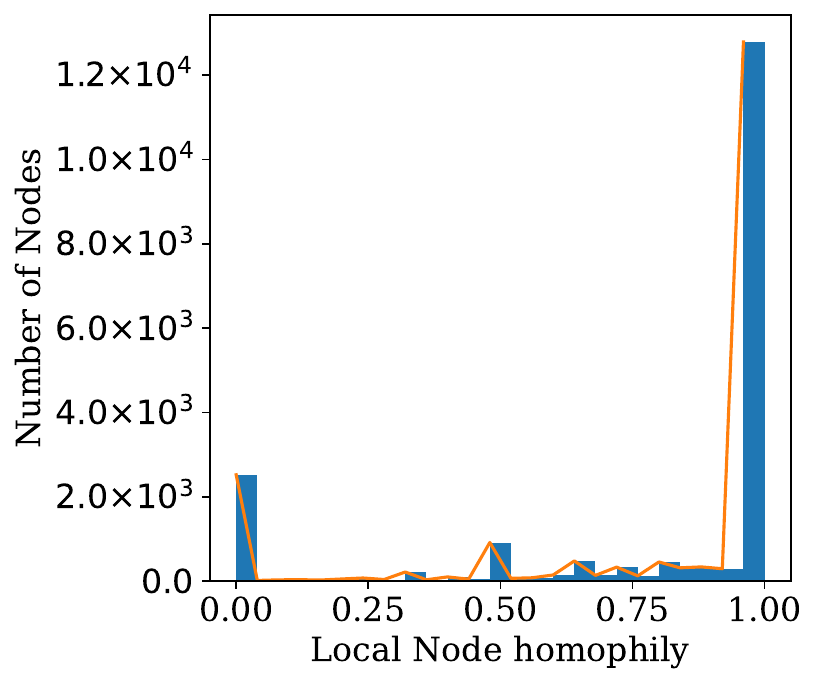} &
    \datasetfig{Reddit}{Reddit} &
    \datasetfig{cora\_ml}{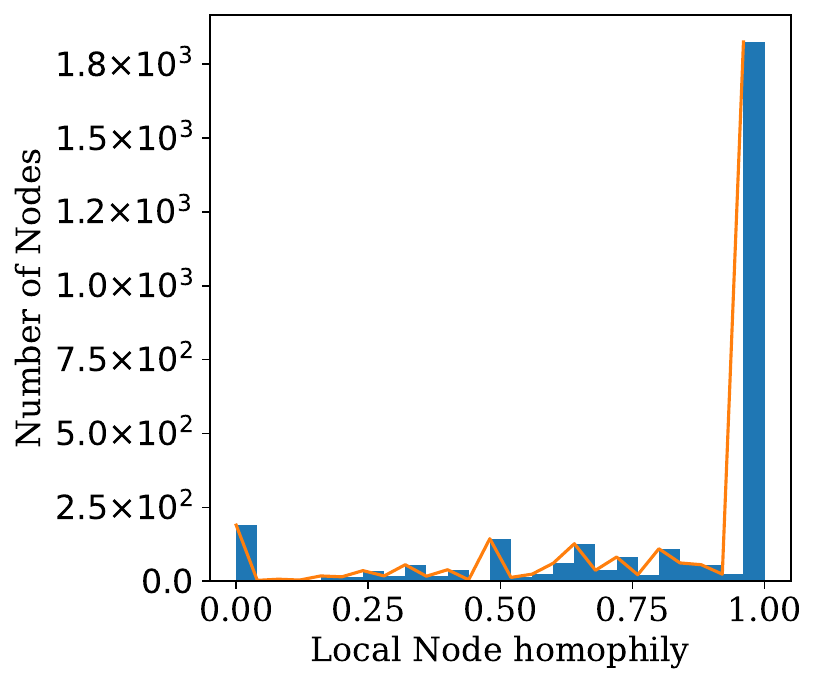} &
    \datasetfig{Cora}{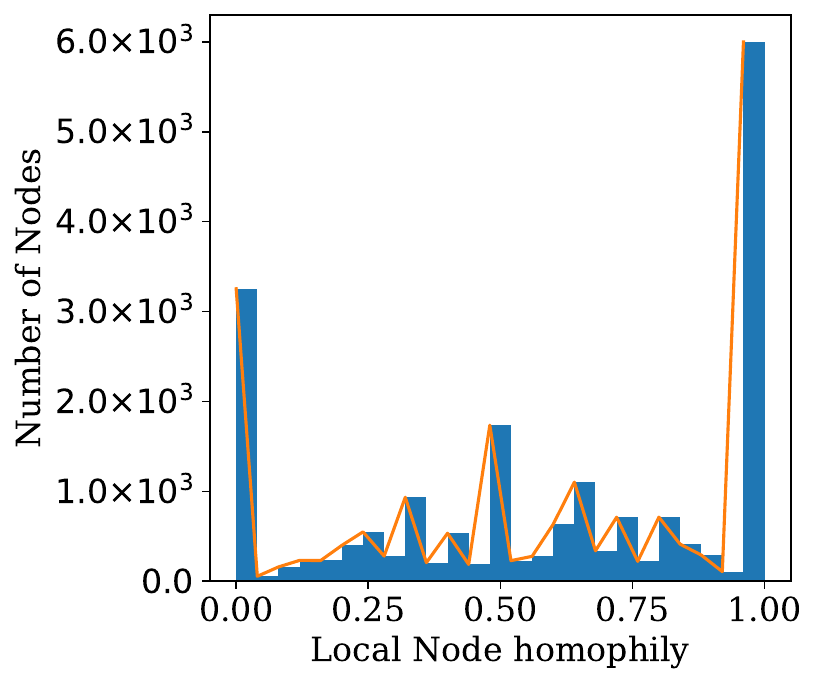} \\
    
    \datasetfig{Reddit2}{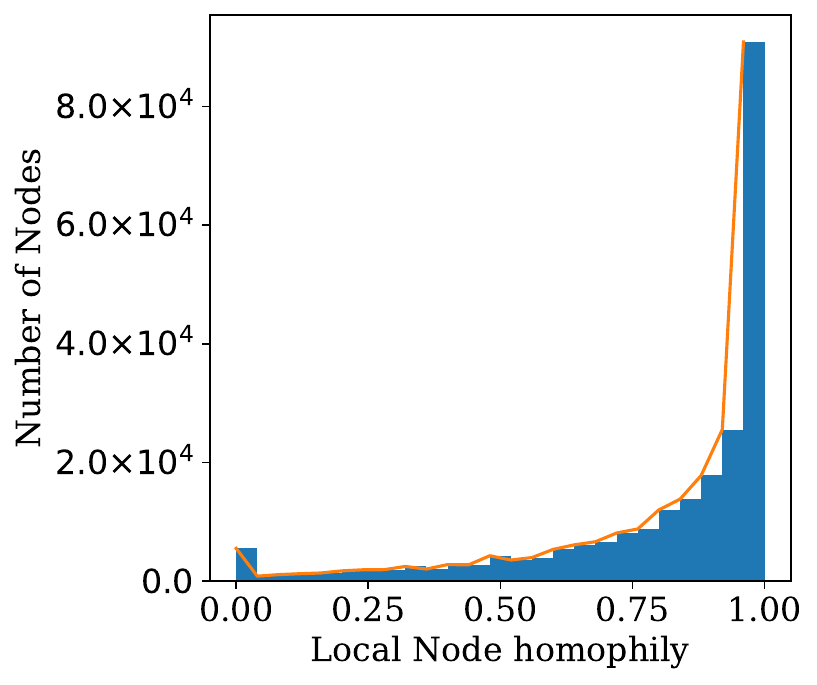} &
    \datasetfig{CS}{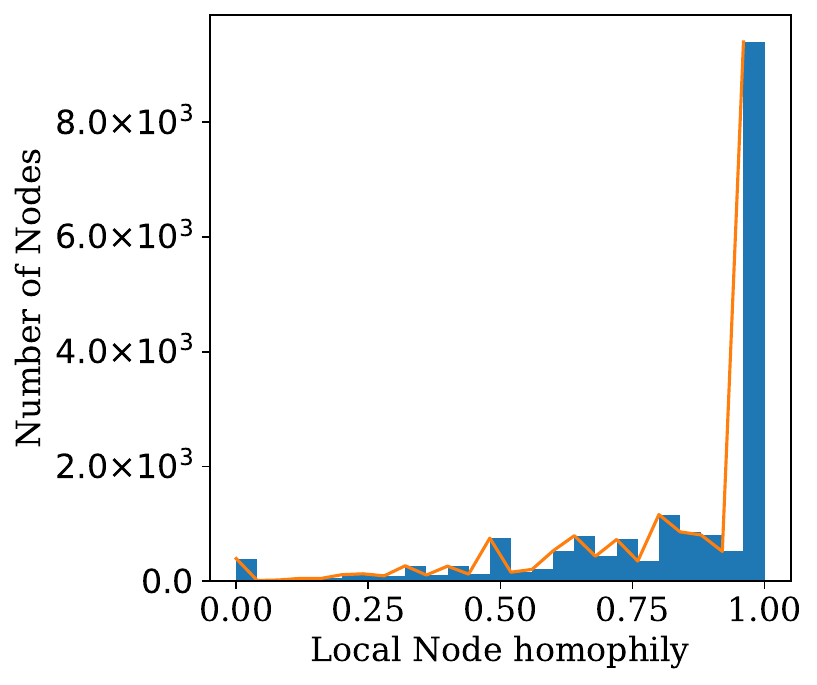} &
    \datasetfig{Photo}{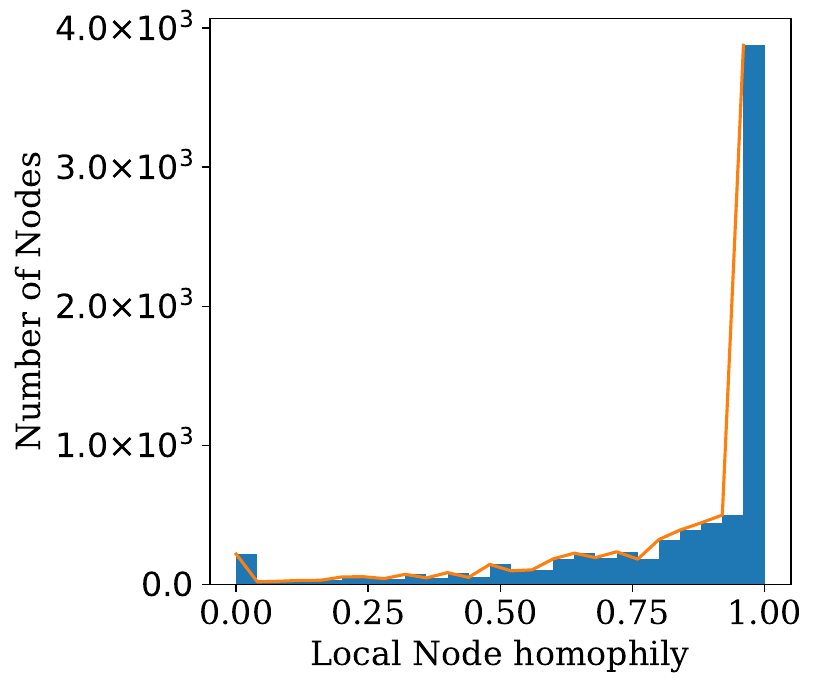} &
    \datasetfig{Physics}{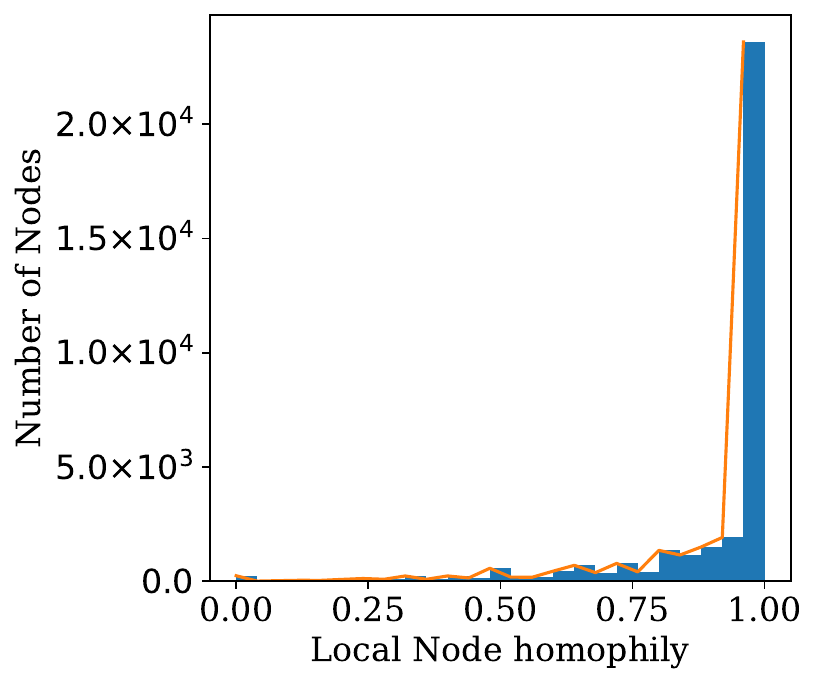} \\
    
    \datasetfig{citeseer}{citeseer} &
    \datasetfig{pokec}{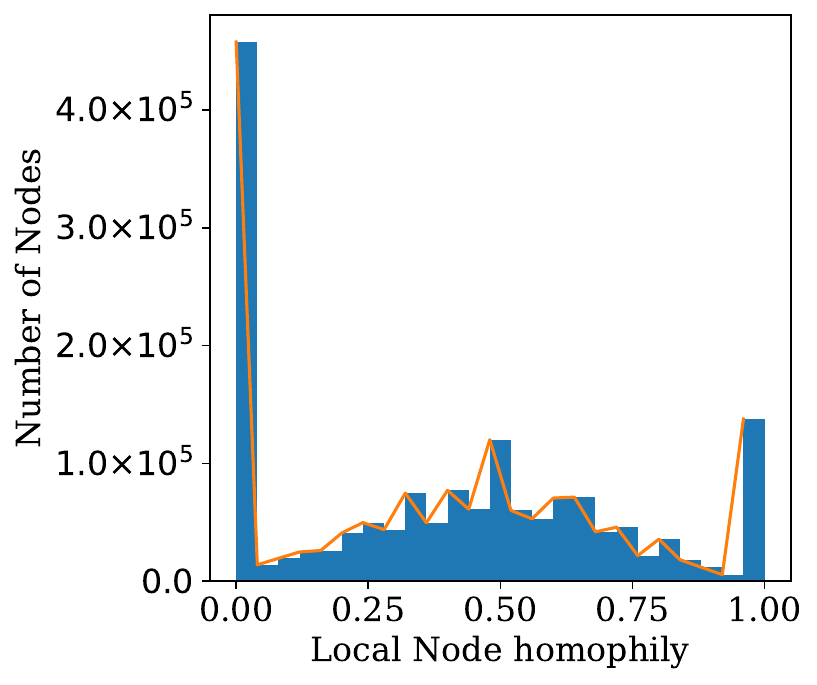} &
    \datasetfig{arxiv-year}{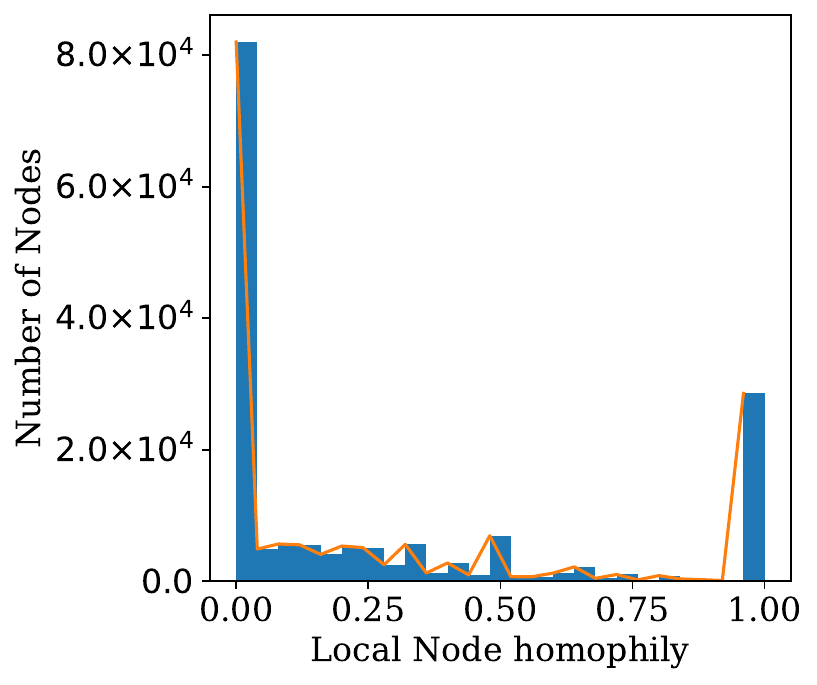} &
    \datasetfig{snap-patents}{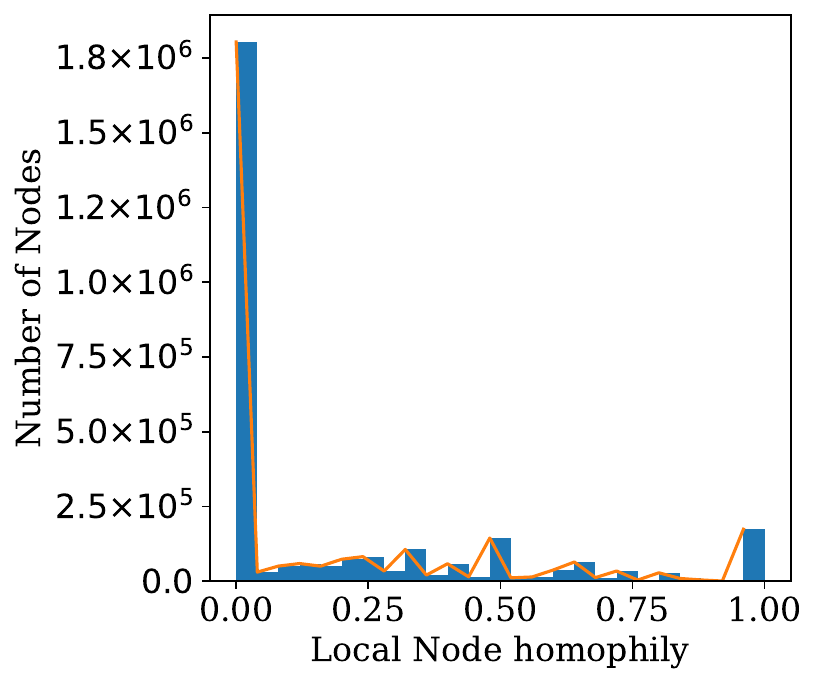} \\
    
    \datasetfig{twitch-gamer}{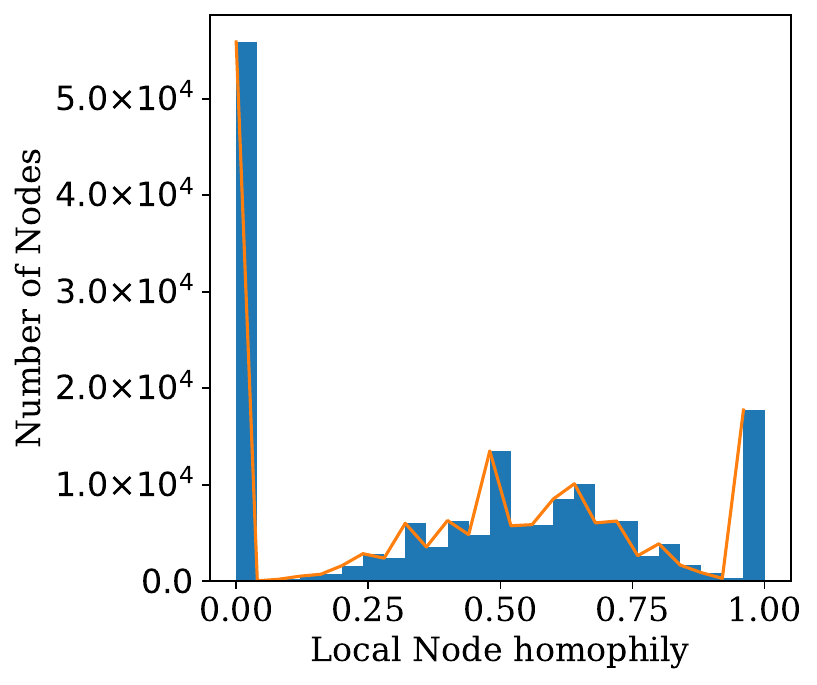} &
    \datasetfig{wiki}{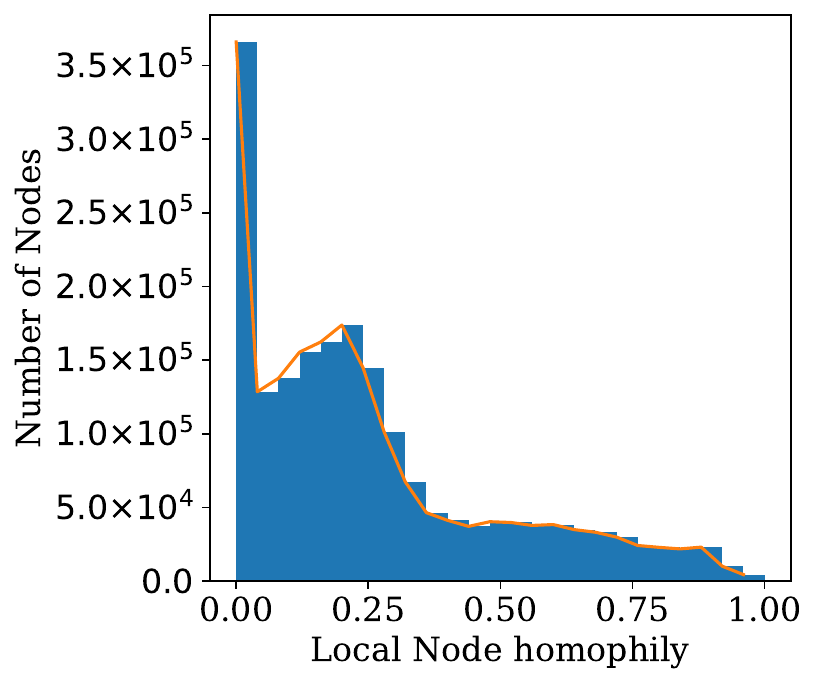} \\
\end{tabular}

\caption{\emph{Local Node Homophily} distributions of our benchmark datasets (Part 2)}
\label{fig:homophily_distributions_2}
\end{figure*}

\FloatBarrier

\section{Appendix: Detailed Experimental Results}
\label{sec:appendixexperiment}

\subsection{AGS-GNN performance comparison}
\label{subsec:AGSNSresults}

Table~\ref{tab:smallhetero} shows numerical results of the algorithms for small heterophilic graphs ($N<100K$). The performance plot in Fig.~\ref{fig:smallhetero} (\S\ref{sec:experiments}) is drawn based on the data from this table.

Table~\ref{tab:smallhomophilic} shows the comparison for small homophilic graphs ($N<100K$). The performance plot in Fig.~\ref{fig:smallhomo} (\S\ref{sec:experiments}) is drawn based on the data from this table.

Table~\ref{tab:large_heterophilic_graphs} shows the comparison of algorithms for large heterophilic graphs ($N>=100K$). The performance plot in Fig.~\ref{fig:largehetero} (\S\ref{sec:experiments}) is drawn based on the data from this table.

Table~\ref{tab:large_homophilic_graphs} compares the algorithms for large heterophilic graphs ($N>=100K$). The performance plot in Fig.~\ref{fig:largehomo} (\S\ref{sec:experiments}) is drawn based on the data from this table.

\begin{table*}[!htbp]
\centering

\begin{center}
\begin{small}

\resizebox{0.7\linewidth}{!}
{
\def\arraystretch{1.2}
\begin{tabular}{c|cc|cc|cc|cc|cc|cc}
\toprule
\multirow{2}{*}{\begin{tabular}[c]{@{}c@{}}Small Heterophilic\\ Graphs\end{tabular}} &
  \multicolumn{2}{c|}{GSAGE} &
  \multicolumn{2}{c|}{GSAINT} &
  \multicolumn{2}{c|}{LINKX$\dag$} &
  \multicolumn{2}{c|}{ACMGCN} &
  \multicolumn{2}{c|}{\agsns{}} &
  \multicolumn{2}{c}{\agsgs{}} \\
 &
  $\mu$ &
  $\sigma$ &
  $\mu$ &
  $\sigma$ &
  $\mu$ &
  $\sigma$ &
  $\mu$ &
  $\sigma$ &
  $\mu$ &
  $\sigma$ &
  $\mu$ &
  $\sigma$ \\
  \midrule
Cornell &
    71.35 &
  6.07 &
  67.03 &
  3.15 &
  \textbf{76.76} &
  {4.32} &
  74.59 &
  1.32 &
  74.59 &
  2.16 &
  70.27 &
  4.83 \\
Texas &
  77.30 &
  5.57 &
  79.46 &
  6.53 &
  81.62 &
  2.02 &
  84.32 &
  7.13 &
  \textbf{84.86} &
  {6.19} &
  80.00 &
  5.57 \\
Wisconsin &
  79.61 &
  3.64 &
  79.61 &
  6.86 &
  83.53 &
  4.74 &
  84.31 &
  3.28 &
  81.96 &
  4.71 &
  \textbf{85.10} &
  {6.49} \\
reed98 &
  61.87 &
  0.53 &
  64.15 &
  0.69 &
  66.63 &
  1.37 &
  66.11 &
  1.25 &
  \textbf{66.74} &
  {1.37} &
  64.66 &
  0.89 \\
amherst41 &
  66.62 &
  0.33 &
  69.57 &
  0.71 &
  78.64 &
  0.35 &
  78.12 &
  0.30 &
  \textbf{79.19} &
  {0.47} &
  77.14 &
  0.90 \\
penn94 &
  75.65 &
  0.42 &
  75.11 &
  0.33 &
  \textbf{85.92} &
  {0.32} &
  85.38 &
  0.53 &
  76.06 &
  0.41 &
  81.56 &
  0.45 \\
Roman-empire &
  79.52 &
  0.42 &
  77.51 &
  0.47 &
  59.14 &
  0.45 &
  71.42 &
  0.39 &
  \textbf{80.49} &
  {0.48} &
  75.38 &
  0.25 \\
cornell5 &
  69.22 &
  0.12 &
  68.10 &
  0.15 &
  80.10 &
  0.27 &
  78.43 &
  0.50 &
  \textbf{82.84} &
  {0.01} &
  74.84 &
  0.35 \\
Squirrel &
  38.66 &
  1.24 &
  39.14 &
  1.45 &
  35.91 &
  1.09 &
  \textbf{72.06} &
  {2.21} &
  68.24 &
  0.97 &
  51.73 &
  1.30 \\
johnshopkins55 &
  67.37 &
  0.54 &
  67.43 &
  0.20 &
  \textbf{79.63} &
  {0.16} &
  77.37 &
  0.61 &
  78.13 &
  0.42 &
  75.93 &
  0.48 \\
Actor &
  34.82 &
  0.55 &
  35.24 &
  0.81 &
  33.93 &
  0.82 &
  34.42 &
  1.08 &
  \textbf{36.55} &
  {0.93} &
  34.88 &
  0.63 \\
Minesweeper &
  \textbf{85.74} &
  {0.25} &
  85.46 &
  0.49 &
  80.02 &
  0.03 &
  80.33 &
  0.23 &
  85.56 &
  0.28 &
  85.25 &
  0.71 \\
Questions &
  97.13 &
  0.01 &
  97.18 &
  0.04 &
  97.06 &
  0.03 &
  97.02 &
  0.00 &
  \textbf{97.27} &
  {0.04} &
  97.23 &
  0.04 \\
Chameleon &
  51.18 &
  2.70 &
  52.32 &
  2.47 &
  50.18 &
  2.01 &
  \textbf{75.81} &
  {1.67} &
  73.46 &
  2.29 &
  66.67 &
  1.65 \\
Tolokers &
  79.15 &
  0.32 &
  78.89 &
  0.37 &
  80.07 &
  0.53 &
  80.45 &
  0.54 &
  \textbf{80.52} &
  {0.41} &
  80.50 &
  0.61 \\
Flickr &
  50.86 &
  0.32 &
  50.28 &
  0.11 &
  \textbf{53.81} &
  {0.31} &
  52.19 &
  0.24 &
  51.52 &
  0.13 &
  50.79 &
  0.13 \\
Amazon-ratings &
  48.08 &
  0.38 &
  52.21 &
  0.27 &
  52.68 &
  0.26 &
  52.94 &
  0.23 &
  \textbf{53.21} &
  {0.46} &
  52.25 &
  0.34 \\
  \bottomrule
\end{tabular}
}
\end{small}
\end{center}

\caption{The table shows micro $F_1$-measure in small heterophilic graphs.}
\label{tab:smallhetero}
\end{table*}

\begin{table*}[!htbp]
\centering
\resizebox{0.7\linewidth}{!}
{
\def\arraystretch{1.2}
\begin{tabular}{c|cc|cc|cc|cc|cc|cc}
\toprule
\multirow{2}{*}{\begin{tabular}[c]{@{}c@{}}Small Homophilic\\ Graphs\end{tabular}} &
  \multicolumn{2}{c|}{GSAGE} &
  \multicolumn{2}{c|}{GSAINT} &
  \multicolumn{2}{c|}{LINKX$\dag$} &
  \multicolumn{2}{c|}{ACMGCN} &
  \multicolumn{2}{c|}{\agsns{}} &
  \multicolumn{2}{c}{\agsgs{}} \\
          & $\mu$          & $\sigma$      & $\mu$ & $\sigma$ & $\mu$ & $\sigma$ & $\mu$ & $\sigma$ & $\mu$          & $\sigma$      & $\mu$ & $\sigma$ \\\midrule
cora      & 52.59          & 0.25          & 65.85 & 0.27     & 57.18 & 0.16     & 66.79 & 0.24     & \textbf{69.32} & {0.11} & 69.18 & 0.25     \\
CiteSeer  & \textbf{71.14}          & {0.10}          & 65.84 & 3.86     & 44.20 & 2.49     & 55.42 & 3.43     & 69.33 & 0.31 & 68.14 & 0.56     \\
dblp      & 85.80 & 0.16 & 85.65 & 0.21     & 80.47 & 0.04     & 85.68 & 0.24     & \textbf{85.97} & {0.11} & 85.40 & 0.05     \\
Computers & 91.31          & 0.11          & 91.77 & 0.14     & 91.38 & 0.17     & 92.03 & 0.48     & \textbf{92.18} & {0.08} & 91.28 & 0.14     \\
pubmed    & 89.00          & 0.09          & 88.57 & 0.14     & 85.05 & 0.25     & 83.79 & 0.15     & \textbf{89.34} & {0.07} & 87.41 & 0.17     \\
cora\_ml   & \textbf{88.71} & {0.23} & 87.21 & 0.37     & 80.67 & 0.32     & 85.28 & 0.12     & 87.70          & 0.08          & 87.95 & 0.41     \\
Cora      & 80.56          & 0.52          & 79.12 & 1.22     & 59.28 & 3.80     & 71.16 & 0.98     & \textbf{81.13} & {0.90} & 79.54 & 0.14     \\
CS        & 95.13          & 0.19          & 95.81 & 0.07     & 94.48 & 0.12     & 94.80 & 0.34     & 95.18 & 0.06 & \textbf{95.81} & {0.06}     \\
Photo     & \textbf{96.58} & {0.09} & 96.51 & 0.10     & 95.48 & 0.15     & 96.21 & 0.11     & 96.56          & 0.16          & 95.86 & 0.12     \\
Physics   & 96.64          & 0.05          & 96.81 & 0.09     & 96.23 & 0.05     & 96.11 & 0.24     & 96.62          & 0.04          & \textbf{96.90} & {0.04 }    \\
citeseer  & 95.15          & 0.20          & 95.32 & 0.16     & 88.53 & 0.27     & 94.26 & 0.19     & \textbf{95.40} & {0.10} & 94.61 & 0.45  \\\bottomrule  
\end{tabular}
}
\caption{The table shows micro $F_1$-measure in small homophilic graphs.}
\label{tab:smallhomophilic}
\end{table*}

\begin{table*}[!htbp]
\centering
\resizebox{0.6\linewidth}{!}
{
\def\arraystretch{1.2}
\begin{tabular}{c|cc|cc|cc|cc|cc}
\toprule
\multirow{2}{*}{\begin{tabular}[c]{@{}c@{}}Large Heterophilic\\ Graphs\end{tabular}} &
  \multicolumn{2}{c|}{GSAGE} &
  \multicolumn{2}{c|}{GSAINT} &
  \multicolumn{2}{c|}{LINKX} &
  \multicolumn{2}{c|}{\agsns{}} &
  \multicolumn{2}{c}{\agsgs{}} \\
               & $\mu$          & $\sigma$      & $\mu$ & $\sigma$ & $\mu$ & $\sigma$ & $\mu$          & $\sigma$      & $\mu$ & $\sigma$ \\\midrule
genius         & 81.76          & 0.33 & 82.09          & 0.19 & 82.59          & 0.02 & \textbf{82.84} & {0.01} & 81.36          & 0.52 \\
pokec          & 68.91          & 0.03 & 68.18          & 0.12 & \textbf{70.57} &{ 0.3}  & 70.08 & 0.00 & 65.77          & 0.00    \\
arxiv-year     & 47.5           & 0.15 & 40.04          & 0.32 & 49.89          & 0.18 & \textbf{50.27} & {0.10} & 38.14          & 0.00    \\
snap-patents   & \textbf{48.36} & {0.01}  & 32.86          & 0.11 & 43.19          & 1.86 & \textbf{48.22} & {0.02} & 28.68          & 0.00    \\
twitch-gamer   & \textbf{61.41} & {0.00 }  & 61.39 & 0.32 & 59.62          & 0.18 & 61.38 & 0.05 & 60.9           & 0.00    \\
AmazonProducts & 62.96          & 0.00    & \textbf{75.25} & {0.05} & 50.66          & 0.28 & 73.78          & 0.01 & 75.07          & 0.06 \\
Yelp           & 65.15          & 0.00    & 77.06 & 0.07 & 52.84          & 2.4  & 75.82          & 0.01 & \textbf{77.09} & {0.08 }\\
\bottomrule
\end{tabular}
}
\caption{The table shows micro $F_1$-measure in large heterophilic graphs.}
\label{tab:large_heterophilic_graphs}
\end{table*}

\begin{table*}[!htbp]
\centering
\resizebox{0.6\linewidth}{!}
{
\def\arraystretch{1.2}
\begin{tabular}{c|cc|cc|cc|cc|cc}
\toprule
\multirow{2}{*}{\begin{tabular}[c]{@{}c@{}}Large Homophilic\\ Graphs\end{tabular}} &
  \multicolumn{2}{c|}{GSAGE} &
  \multicolumn{2}{c|}{GSAINT} &
  \multicolumn{2}{c|}{LINKX} &
  \multicolumn{2}{c|}{\agsns{}} &
  \multicolumn{2}{c}{\agsgs{}} \\
               & $\mu$          & $\sigma$      & $\mu$ & $\sigma$ & $\mu$ & $\sigma$ & $\mu$          & $\sigma$      & $\mu$ & $\sigma$ \\\midrule
Reddit         & 95.17 & 0.01    & 93.91          & 0.03 & 92.53          & 0.26 & \textbf{95.64} & {0.06} & 94.64          & 0.03 \\
Reddit2        & 88.94          & 0.32 & 73.3           & 2.2  & 86.54          & 0.44 & \textbf{92.23} & {0.11} & 70.14          & 0.82\\
Reddit0.525 & 91.31 & 0.06 & 90.46 & 0.05 & 87.13 & 0.18 & \textbf{91.91} & {0.06} & 91.40 & 0.05 \\
Reddit0.425 & 89.33 & 0.07 & 89.13 & 0.06 & 84.38 & 0.61 & \textbf{90.27} & {0.1} & 90.17 & 0.20 \\
Reddit0.325 & 87.19 & 0.12 & 87.69 & 0.15 & 81.97 & 0.31 & \textbf{88.48} & {0.05} & 87.50 & 0.05\\
\bottomrule
\end{tabular}
}
\caption{The table shows micro $F_1$-measure in large homophilic graphs.}
\label{tab:large_homophilic_graphs}
\end{table*}

\subsection{AGS-GNN vs other homophilic and heterophilic GNNs}
\label{subsec:othermethods}

\textbf{Homophilic GNNs:} Some notable homophilic GNNs from the literature are Graph Isomorphic Network (GIN)~\citep{xu2018powerful}, Graph Attention Network (GAT)~\citep{velivckovic2017graph}, Graph Convolutional Network (GCN)~\citep{kipf2016semi}. 
For scaling widely used method for node sampling is  GSAGE~\citep{hamilton2017inductive} and for sampling based methods ClusterGCN~\citep{chiang2019cluster} and GSAINT~\citep{zeng2019graphsaint}. 

\textbf{Heterophlic GNNs:} Currently, the SOTA performance in heterophilic graph is from ACM-GCN~\citep{luan2022revisiting}. Some other notable heterophilic GNNs in the literature are
H2GCN~\citep{zhu2020beyond}, 
MixHop~\citep{abu2019mixhop},
GPRGNN~\citep{chien2020adaptive}, GCNII~\citep{chen2020simple},
However, most of these are not scalable for large graphs as they need to process entire graphs for improved performance. 
LINKX~\citep{lim2021large} is the only scalable GNN for heterophilic graphs. However, the input dimension of MLP depends on the number of nodes and may cause scaling issues for very large graphs. 
Some other related GNN for heterophilic graphs are jumping knowledge networks (GCNJK, GATJK)\citep{xu2018representation}, and APPNP~\citep{gasteiger2018predict}. The method uses only node features, MLP~\citep{goodfellow2016deep}. Uses only graph structure, label
propagation, LINK~\citep{zheleva2009join}, MultiLP~\citep{zhou2003learning}.
And simple network methods, Simple Graph Convolutional Network (SGC)~\citep{wu2019simplifying}, C\&S~\citep{huang2020combining}.

Table.~\ref{tab:smallothergnns} shows the performance of AGS relative to the $18$ recent algorithms for  small heterophilic graphs. We can see that ACM-GCN, AGS-NS, and LINKX are the best-performing, with AGS-NS the best among them.

\begin{table*}[!htbp]
\centering
\resizebox{0.9\linewidth}{!}
{
\def\arraystretch{1.2}
\begin{tabular}{c|cc|cc|cc|cc|cc|cc|cc|cc|cc}
\toprule
\textbf{} &
  \multicolumn{2}{c|}{{Cornell}} &
  \multicolumn{2}{c|}{{Texas}} &
  \multicolumn{2}{c|}{{Wisconsin}} &
  \multicolumn{2}{c|}{{reed98}} &
  \multicolumn{2}{c|}{{amherst41}} &
  \multicolumn{2}{c|}{{penn94}} &
  \multicolumn{2}{c|}{{Roman-empire}} &
  \multicolumn{2}{c|}{{cornell5}} &
  \multicolumn{2}{c}{{Squirrel}} \\
{GNNs} &
  \textbf{$\mu$} &
  \textbf{$\sigma$} &
  \textbf{$\mu$} &
  \textbf{$\sigma$} &
  \textbf{$\mu$} &
  \textbf{$\sigma$} &
  \textbf{$\mu$} &
  \textbf{$\sigma$} &
  \textbf{$\mu$} &
  \textbf{$\sigma$} &
  \textbf{$\mu$} &
  \textbf{$\sigma$} &
  \textbf{$\mu$} &
  \textbf{$\sigma$} &
  \textbf{$\mu$} &
  \textbf{$\sigma$} &
  \textbf{$\mu$} &
  \textbf{$\sigma$} \\
  \midrule
GSAGE &
  71.35 &
  6.07 &
  77.3 &
  5.57 &
  79.61 &
  3.64 &
  61.87 &
  0.53 &
  66.62 &
  0.33 &
  75.65 &
  0.42 &
  79.52 &
  0.42 &
  69.22 &
  0.12 &
  38.66 &
  1.24 \\
GCN &
  49.19 &
  4.32 &
  60 &
  6.26 &
  55.69 &
  3.64 &
  60.83 &
  0.53 &
  57.23 &
  1.91 &
  61.78 &
  0.35 &
  42.78 &
  1.02 &
  57.62 &
  0.11 &
  28.07 &
  0.79 \\
GAT &
  56.76 &
  7.05 &
  61.62 &
  3.15 &
  56.08 &
  4.22 &
  59.59 &
  0.46 &
  57.09 &
  0.54 &
  62.15 &
  0.25 &
  41.73 &
  0.73 &
  57.64 &
  0.5 &
  35.52 &
  0.79 \\
GIN &
  52.43 &
  6.3 &
  61.08 &
  4.71 &
  56.08 &
  6.27 &
  60.1 &
  0.87 &
  59.82 &
  1.45 &
  63.18 &
  0.24 &
  47.34 &
  0.26 &
  58.74 &
  0.18 &
  30.55 &
  0.49 \\
GSAINT &
  67.03 &
  3.15 &
  79.46 &
  6.53 &
  79.61 &
  6.86 &
  64.15 &
  0.69 &
  69.57 &
  0.71 &
  75.11 &
  0.33 &
  77.51 &
  0.47 &
  68.1 &
  0.15 &
  39.14 &
  1.45 \\
LINKX &
  \textbf{76.76} &
  {4.32} &
  81.62 &
  2.02 &
  83.53 &
  4.74 &
  66.63 &
  1.37 &
  79.64 &
  0.35 &
  \textbf{85.92} &
  {0.32} &
  59.14 &
  0.45 &
  80.1 &
  0.27 &
  35.91 &
  1.09 \\
ACM-GCN &
  74.59 &
  1.32 &
  84.32 &
  7.13 &
  \textbf{84.31} &
  {3.28} &
  66.11 &
  1.25 &
  78.12 &
  0.3 &
  85.38 &
  0.53 &
  71.42 &
  0.39 &
  78.43 &
  0.5 &
  72.06 &
  2.21 \\
LINK &
  54.59 &
  4.32 &
  64.86 &
  7.83 &
  56.47 &
  6.61 &
  60.62 &
  2.84 &
  70.11 &
  1.31 &
  79.25 &
  0.42 &
  8.59 &
  0.5 &
  71.49 &
  0.44 &
  \textbf{75.41} &
  {1.22} \\
MLP &
  70.27 &
  4.83 &
  67.57 &
  5.92 &
  80.39 &
  4.47 &
  41.14 &
  0.9 &
  51.81 &
  2.75 &
  71.22 &
  0.57 &
  65.31 &
  0.32 &
  61.12 &
  1.45 &
  31.22 &
  0.88 \\
CS &
  69.73 &
  5.77 &
  69.73 &
  5.51 &
  79.22 &
  6.86 &
  44.66 &
  3.23 &
  49.71 &
  3.39 &
  70.86 &
  0.76 &
  65.42 &
  0.3 &
  60.12 &
  0.96 &
  31.01 &
  0.81 \\
SGC &
  46.49 &
  5.24 &
  52.43 &
  5.57 &
  53.73 &
  2.66 &
  58.86 &
  2.38 &
  68.86 &
  2.48 &
  68.29 &
  0.18 &
  42.32 &
  0.63 &
  69.14 &
  0.66 &
  24.32 &
  0.46 \\
GPRGNN &
  50.81 &
  6.26 &
  55.68 &
  8.82 &
  60.78 &
  4.64 &
  50.67 &
  3.95 &
  58.08 &
  1.92 &
  75.43 &
  0.99 &
  69.68 &
  0.3 &
  63.75 &
  0.96 &
  26.94 &
  1.13 \\
APPNP &
  45.41 &
  7.53 &
  52.43 &
  7.17 &
  50.2 &
  6.02 &
  53.37 &
  1.99 &
  62.55 &
  2.22 &
  73.54 &
  0.41 &
  56.99 &
  0.26 &
  66.78 &
  0.79 &
  24.17 &
  1.64 \\
MIXHOP &
  62.7 &
  9.27 &
  54.05 &
  9.21 &
  72.94 &
  6.49 &
  57.41 &
  2.61 &
  68.99 &
  2.07 &
  78.08 &
  1.18 &
  78.49 &
  0.24 &
  70.09 &
  0.98 &
  32.8 &
  1.33 \\
GCNJK &
  45.95 &
  7.25 &
  47.03 &
  6.3 &
  50.2 &
  2.66 &
  58.34 &
  3.01 &
  71.41 &
  2.74 &
  78.33 &
  1.32 &
  58.32 &
  0.46 &
  68.85 &
  0.77 &
  26.13 &
  0.64 \\
GATJK &
  54.59 &
  3.15 &
  52.97 &
  7.57 &
  56.86 &
  4.11 &
  61.35 &
  4.23 &
  70.29 &
  1.11 &
  79.47 &
  0.32 &
  71.54 &
  0.67 &
  69.7 &
  0.31 &
  31.2 &
  0.72 \\
LINKConcat &
  76.22 &
  6.26 &
  74.05 &
  2.76 &
  81.57 &
  4.04 &
  59.79 &
  0.96 &
  76.06 &
  1.47 &
  83.81 &
  0.28 &
  12.06 &
  1.61 &
  76.87 &
  0.33 &
  67.36 &
  1.57 \\
GCNII &
  60.36 &
  8.92 &
  54.95 &
  1.27 &
  71.9 &
  7.22 &
  57.51 &
  2.77 &
  70.32 &
  1.48 &
  80.21 &
  0.4 &
  74.82 &
  0.26 &
  71.97 &
  0.28 &
  32.69 &
  0.84 \\
\agsns{} &
  74.59 &
  2.16 &
  \textbf{84.86} &
  {6.19} &
  81.96 &
  4.71 &
  \textbf{66.74} &
  {1.37} &
  \textbf{79.19} &
  {0.47} &
  76.06 &
  0.41 &
  \textbf{80.49} &
  {0.48} &
  \textbf{82.84} &
  {0.01} &
  68.24 &
  0.97 \\
  \midrule
 &
  \multicolumn{2}{c}{{johnshopkins55}} &
  \multicolumn{2}{c|}{{Actor}} &
  \multicolumn{2}{c|}{{Minesweeper}} &
  \multicolumn{2}{c|}{{Questions}} &
  \multicolumn{2}{c|}{{Chameleon}} &
  \multicolumn{2}{c|}{{Tolokers}} &
  \multicolumn{2}{c|}{{Flickr}} &
  \multicolumn{2}{c|}{{Amazon-ratings}} &
   &
   \\   
\textbf{} &
  \textbf{$\mu$} &
  \textbf{$\sigma$} &
  \textbf{$\mu$} &
  \textbf{$\sigma$} &
  \textbf{$\mu$} &
  \textbf{$\sigma$} &
  \textbf{$\mu$} &
  \textbf{$\sigma$} &
  \textbf{$\mu$} &
  \textbf{$\sigma$} &
  \textbf{$\mu$} &
  \textbf{$\sigma$} &
  \textbf{$\mu$} &
  \textbf{$\sigma$} &
  \textbf{$\mu$} &
  \textbf{$\sigma$} &
  \textbf{} &
  \textbf{} \\
  \midrule
GSAGE &
  67.37 &
  0.54 &
  34.82 &
  0.55 &
  \textbf{85.74} &
  {0.25} &
  97.13 &
  0.01 &
  51.18 &
  2.7 &
  79.15 &
  0.32 &
  50.86 &
  0.32 &
  48.08 &
  0.38 &
   &
   \\
GCN &
  61.81 &
  0.95 &
  29.18 &
  0.78 &
  80.12 &
  0.06 &
  97.02 &
  0.00 &
  41.89 &
  2.15 &
  78.17 &
  0.01 &
  49.05 &
  0.06 &
  43.63 &
  0.43 &
   &
   \\
GAT &
  60.33 &
  0.72 &
  30.64 &
  0.8 &
  80.00 &
  0.00 &
  97.02 &
  0.00 &
  48.11 &
  1.82 &
  78.16 &
  0.00 &
  48.29 &
  0.16 &
  39.21 &
  0.37 &
   &
   \\
GIN &
  63.4 &
  0.31 &
  28.09 &
  0.25 &
  80.74 &
  0.24 &
  97.1 &
  0.01 &
  40.13 &
  2.21 &
  78.29 &
  0.09 &
  47.2 &
  0.88 &
  45.85 &
  0.4 &
   &
   \\
GSAINT &
  67.43 &
  0.2 &
  35.24 &
  0.81 &
  85.46 &
  0.49 &
  97.18 &
  0.04 &
  52.32 &
  2.47 &
  78.89 &
  0.37 &
  50.28 &
  0.11 &
  52.21 &
  0.27 &
   &
   \\
LINKX &
  \textbf{79.63} &
  {0.16} &
  33.93 &
  0.82 &
  80.02 &
  0.03 &
  97.06 &
  0.03 &
  50.18 &
  2.01 &
  80.07 &
  0.53 &
  \textbf{53.81} &
  {0.31} &
  52.68 &
  0.26 &
   &
   \\
ACM-GCN &
  77.37 &
  0.61 &
  34.42 &
  1.08 &
  80.33 &
  0.23 &
  97.02 &
  0.00 &
  75.81 &
  1.67 &
  80.45 &
  0.54 &
  52.19 &
  0.24 &
  52.94 &
  0.23 &
   &
   \\
LINK &
  70.95 &
  1.37 &
  22.54 &
  1.54 &
  70.74 &
  0.72 &
  96.41 &
  0.1 &
  \textbf{79.08} &
  {1.29} &
  78.13 &
  0.56 &
  51.18 &
  0.02 &
  51.43 &
  0.2 &
   &
   \\
MLP &
  52.26 &
  2.31 &
  33.91 &
  0.7 &
  80.00 &
  0.00 &
  97.14 &
  0.06 &
  40.26 &
  2.96 &
  78.52 &
  0.14 &
  46.6 &
  0.07 &
  41.35 &
  0.46 &
   &
   \\
CS &
  52.61 &
  2.68 &
  32.41 &
  0.68 &
  80.00 &
  0.00 &
  97.14 &
  0.04 &
  37.87 &
  1.19 &
  78.34 &
  0.04 &
  46.6 &
  0.14 &
  40.61 &
  0.45 &
   &
   \\
SGC &
  69.75 &
  0.64 &
  27.91 &
  0.71 &
  82.22 &
  0.18 &
  97.07 &
  0.03 &
  33.29 &
  1.41 &
  77.48 &
  0.98 &
  44.46 &
  0.11 &
  41.27 &
  0.1 &
   &
   \\
GPRGNN &
  58.65 &
  1.47 &
  29.49 &
  0.97 &
  85.67 &
  0.31 &
  97.08 &
  0.05 &
  34.17 &
  2.51 &
  79.41 &
  0.51 &
  51.69 &
  0.19 &
  46.28 &
  0.49 &
   &
   \\
APPNP &
  63.2 &
  1.48 &
  25.75 &
  0.4 &
  79.95 &
  0.11 &
  97.06 &
  0.04 &
  31.97 &
  1.7 &
  78.03 &
  0.26 &
  49.65 &
  0.19 &
  46.69 &
  0.43 &
   &
   \\
MIXHOP &
  68.86 &
  2.26 &
  32.21 &
  0.86 &
  85.54 &
  0.38 &
  97.15 &
  0.03 &
  40.88 &
  2.53 &
  \textbf{81.82} &
  {0.4} &
  52.23 &
  0.31 &
  46.61 &
  0.99 &
   &
   \\
GCNJK &
  69.21 &
  0.62 &
  26.14 &
  0.68 &
  84.86 &
  0.43 &
  97.1 &
  0.02 &
  34.43 &
  0.84 &
  81.51 &
  0.57 &
  53.17 &
  0.18 &
  45.18 &
  0.26 &
   &
   \\
GATJK &
  68.32 &
  1.09 &
  26.42 &
  0.91 &
  84.66 &
  0.37 &
  96.96 &
  0.17 &
  44.25 &
  1.89 &
  81.29 &
  0.38 &
  51.35 &
  0.32 &
  49.41 &
  0.61 &
   &
   \\
LINKConcat &
  74.48 &
  1.22 &
  28.09 &
  1.65 &
  62.72 &
  2.5 &
  96.79 &
  0.17 &
  69.04 &
  1.97 &
  77.96 &
  0.91 &
  46.71 &
  1.73 &
  51 &
  0.37 &
   &
   \\
GCNII &
  70.53 &
  0.61 &
  26.29 &
  0.27 &
  80.92 &
  0.38 &
  97.07 &
  0.03 &
  32.68 &
  1.17 &
  81.8 &
  0.61 &
  52.74 &
  0.24 &
  47.74 &
  0.24 &
   &
   \\
\agsns{} &
  78.13 &
  0.42 &
  \textbf{36.55} &
  {0.93} &
  85.56 &
  0.28 &
  \textbf{97.27} &
  {0.04} &
  73.46 &
  2.29 &
  80.52 &
  0.41 &
  51.52 &
  0.13 &
  \textbf{53.21} &
  {0.46} &
   &
  \\\bottomrule
\end{tabular}
}
\caption{F1-Score of other homophilic and heterophilic GNN performance on small heterophilic graphs (node $<100k$).}
\label{tab:smallothergnns}
\end{table*}



\section{Appendix: Ablation Study Results}
\label{subsec:ablation}

\subsection{AGS with existing GNNs}

Table~\ref{tab:agsgnnvariants} compares the performance cof AGS with existing GNNs (GSAGE, ChebNet, GSAINT, GIN, GAT, and GCN) evaluated on the heterophilic graphs (\texttt{Reed98, Roman-empire, Actor, Minesweeper, Tolokers}). 
These results show that AGS with GSAGE, ChebNet, and GSAINT performed the best.
Fig.~\ref{fig:agsgnnvariants} in the ablation study (\S\ref{subsec:ablation_study}) is based on the result from this table.

\begin{table*}[!htbp]
\centering
\resizebox{0.7\linewidth}{!}
{
\def\arraystretch{1.0}
\begin{tabular}{l|cc|cc|cc|cc|cc|cc}
\toprule
Methods &
  \multicolumn{2}{c|}{AGS-GSAGE} &
  \multicolumn{2}{c|}{AGS-GSAINT (rw)} &
  \multicolumn{2}{c|}{AGS-Chebnet} &
  \multicolumn{2}{c|}{AGS-GCN} &
  \multicolumn{2}{c|}{AGS-GAT} &
  \multicolumn{2}{c|}{AGS-GIN} \\\midrule
Dataset      & $\mu$            & $\sigma$ & $\mu$   & $\sigma$ & $\mu$            & $\sigma$        & $\mu$   & $\sigma$ & $\mu$   & $\sigma$ & $\mu$   & $\sigma$ \\\midrule
reed98       & 63.32          & 0.76          & 55.44          & 3.04          & \textbf{64.56} & {0.84} & 56.58 & 1.20 & 55.85 & 1.33 & 59.48 & 2.00 \\
Roman-empire & 77.35          & 0.43          & 74.39          & 0.41          & \textbf{77.44} & {0.24} & 44.36 & 1.76 & 46.57 & 0.75 & 45.91 & 0.26 \\
Actor        & \textbf{35.36} & {0.49} & 29.97          & 0.52          & 33.76          & 0.73          & 28.64 & 0.75 & 26.58 & 1.01 & 26.74 & 0.30 \\
Minesweeper & \textbf{84.45} & {0.41} & 82.31 & 0.68 & \textbf{84.45} & {0.62} & 80.21 & 0.08 & 80.08 & 0.16 & 80.88 & 0.13 \\
Tolokers     & 78.33          & 0.14          & \textbf{78.69} & {0.39} & 78.31          & 0.35          & 78.16 & 0.00 & 78.16 & 0.00 & 78.38 & 0.11\\\bottomrule
\end{tabular}
}
\caption{$F_1$ score of heterophilic graphs using AGS sampler with different underlying GNNs. The best-performing models are AGS-GSAGE and AGS-CHEB.}
\label{tab:agsgnnvariants}
\end{table*}

\subsection{Different similarity functions for Nearest Neighbor and Submodular optimization}
\label{subsec:submdoularablation}

Table~\ref{tab:submodular_functions} shows a performance comparison of different submodular and nearest neighbor samplers in heterophilic graphs. Here, we generate three synthetic graphs of average degree $200$ with strong ($0.05$), moderate ($0.25$), and weak ($0.50$) heterophily. 
The underlying GNN is ChebNet, and only the samplers are changed. Neighborhood samplers of two hops, with  sizes  set as $k = [25, 25]$.

\begin{table*}[htbp]
\centering
\resizebox{0.7 \linewidth}{!}
{
\def\arraystretch{1.2}
\begin{tabular}{l|cc|cc|cc}
\toprule
           & \multicolumn{2}{c|}{Cora1000.050.110True} & \multicolumn{2}{c|}{Cora1000.250.110True} & \multicolumn{2}{c}{Cora1000.500.110True} \\
Samplers                 & $\mu$   & $\sigma$ & $\mu$   & $\sigma$ & $\mu$            & $\sigma$        \\\midrule
\tt{Wholegraph}               & 61.20 & 0.16     & 66.49 & 0.25     & 95.45          & 0.41          \\
\tt{Random}                   & 56.61 & 0.33     & 61.16 & 1.08     & 88.61          & 0.66          \\
\tt{knncosine}                & 53.02 & 0.88     & 66.67 & 0.30     & \textbf{90.72} & {0.52} \\
\tt{knneuclidean}             & 52.95 & 0.85     & 63.53 & 0.70     & 88.82          & 0.67          \\
\tt{submodularcosine}         & 56.65 & 0.75     & 58.77 & 0.45     & 84.23          & 0.79          \\
\tt{submodulareuclidean}      & 56.16 & 1.42     & 59.05 & 0.36     & 87.72          & 0.53          \\
\tt{fastlink}                 & 45.96 & 0.58     & 64.76 & 0.68     & 84.87          & 1.01          \\
\tt{link-nn}                  & 44.20 & 0.33     & 59.44 & 0.62     & 75.73          & 0.76          \\
\textbf{\tt{link-sub}} & \textbf{65.26}      & {0.35}      & \textbf{67.69}      & {0.57}      & 84.09               & 0.65               \\
\tt{apricotfacilityeuclidean} & 56.12 & 0.73     & 59.05 & 0.28     & 87.44          & 0.50          \\
\tt{apricotfacilitycosine}    & 56.40 & 0.54     & 59.79 & 0.49     & 86.56          & 0.26          \\
\tt{apricotcoverage}          & 55.24 & 0.71     & 59.58 & 0.81     & 85.19          & 1.02          \\
\tt{apricotfeature}           & 54.99 & 0.47     & 57.07 & 0.43     & 75.49          & 0.77          \\
\tt{apricotgraph}             & 56.40 & 0.54     & 59.79 & 0.49     & 86.56          & 0.26 \\\bottomrule
\end{tabular}
}
\caption{Performance of different sub-modular functions and nearest neighbors on Synthetic Cora graphs.}
\label{tab:submodular_functions}
\end{table*}

We consider the following sampler variants.

\begin{itemize}
    \item {\tt Wholegraph}: All neighbors of vertices are taken.     
    \item {\tt Random}: $25$ Random neighbors are taken from each node's first and second hop.
    \item {\tt knncosine}: Similarity sampler with $k$-Nearest Neighbor using Cosine similarity function. 
    \item {\tt knneuclidean}: $k$-Nearest Neighbor sampler with Euclidean distance function. Note that $distance.max()-distance$ operation is performed to convert it into similarity.
    \item {\tt submodularcosine}: submodular optimization with facility location using cosine similarity as $\phi$ is used. The general form of the facility location function is:
   
   $f(X, Y) = \sum\limits_{y \in Y} \max_{x \in X} \phi(x, y)$

   \item {\tt submodulareuclidean}: submodular optimization with facility location using Euclidean distance measures is used. The distance is also converted into similarity.

   \item {\tt 
   fastlink}: The learned weight from the Siamese model using regression task directly used for sampling.

   \item {\tt 
   link-nn}: The learned weight is used with the nearest neighbor ranking and weight assignment. 

   \item {\tt 
   link-sub}: The learned function generates a pairwise similarity kernel for the submodular optimization facility location function.

   \item {\tt apricotfacilityeuclidean}: This is the same as {\tt submodulareuclidean}, but the implementation and optimization operations are performed using the Apricot library.

   \item {\tt apricotfacilitycosine}: This is the same as {\tt submodularcosine}, but the implementation and optimization operations are performed using the Apricot library.

   \item {\tt apricotcoverage}: Maximum Coverage submodular function is used. The general form of coverage function is:

   $f(X) = \sum\limits_{d=1}^{D} \min \left( \sum\limits_{n=1}^{N} X_{i, d}, 1 \right)$

   \item {\tt apricotfeature}: Feature-based submodular function is used. The general form of feature-based function is,

   $f(X) = \sum\limits_{d=1}^{D}\phi\left(\sum\limits_{i=1}^{N} X_{i, d}\right) $

   This experiment uses the $\phi = \mathrm{sqrt}$ concave function.

   \item {\tt apricotgraph}: Graph-based submodular function is used. The general form of Graph cut selection is:

   	$f(X, V) = \lambda\sum_{v \in V} \sum_{x \in X} \phi(x, v) - \sum_{x, y \in X} \phi(x, y)$
    
\end{itemize}

\end{document}